\documentclass{article}
\usepackage{iclr2027_conference,times}
\usepackage{amsmath, amssymb, amsthm, mathtools}
\usepackage{bm}
\usepackage{microtype}
\usepackage[colorlinks=true,linkcolor=blue!50!black,urlcolor=blue!60!black,citecolor=blue!60!black]{hyperref}
\usepackage{booktabs}
\usepackage{graphicx}
\usepackage{xcolor}
\usepackage{enumitem}
\setlist{itemsep=2pt, topsep=2pt}

\theoremstyle{plain}
\newtheorem{theorem}{Theorem}
\newtheorem{proposition}[theorem]{Proposition}
\newtheorem{lemma}[theorem]{Lemma}
\newtheorem{corollary}[theorem]{Corollary}
\theoremstyle{definition}

\newtheorem{assumption}[theorem]{Assumption}

\theoremstyle{remark}
\newtheorem*{remark}{Remark}

\newcommand{\R}{\mathbb{R}}
\newcommand{\E}{\mathbb{E}}
\newcommand{\Var}{\operatorname{Var}}
\newcommand{\Cov}{\operatorname{Cov}}

\newcommand{\shat}{\hat{s}}
\newcommand{\skappa}{s_\kappa}
\newcommand{\M}{\mathcal{M}}
\newcommand{\PT}{P_T}
\newcommand{\PN}{P_N}
\newcommand{\II}{\mathrm{I\!I}}
\newcommand{\shape}{S}

\newcommand{\tpt}{\text{T-PT}}

\title{Is your uncertainty map wrong, or is its \\ 
target? Exact diagnostics for the Tweedie \\
diagonal, and a gradient-free alternative}

\author{%
  Vicent Ribas$^{1}$ \qquad Anna Oliveras Tous$^{2,1}$ \\[3pt]
  {\normalfont\normalsize $^{1}$Eurecat \qquad
   $^{2}$University of Barcelona}
}

\date{}
\iclrfinalcopy   

\begin{document}
\maketitle
\lhead{Preprint.  Under review.}

\begin{abstract}
Diffusion models can predict a follow-up medical scan from a baseline, and a clinician reading such a prediction needs a per-voxel map of where it should be trusted.  Many of these maps approximate the diagonal of the Tweedie posterior covariance, and they are evaluated by comparing one approximation against another, so it remains unclear whether the estimator or the target is the limiting factor.  We therefore compute the exact diagonal on six checkpoints across fourteen model--corpus conditions.  A Hutchinson estimator at $M{=}200$ tracks it at rank agreement of at least $0.92$ in every condition.  However, in four of the fourteen the exact diagonal is anti-correlated with the denoising error, reaching $-0.13$, so a faithful estimator reproduces that reversal.  All four are real-image conditions, and in these checkpoints the reversal does not appear on the models' own samples, so evaluation on model-generated samples gives a favourable reading of this family.  From this one may conclude that the limiting factor is the target rather than the estimator.

We then introduce Tweedie Probe-Tangent (\tpt{}), a gradient-free residual probe that corrupts one model-supported prediction repeatedly and measures the voxel-wise variance of the denoiser's response.  \tpt{} reads a different functional of the same Jacobian, and its exact second-order form ranks identically with the diagonal wherever the diagonal reverses.  At the working budget of thirty probes, however, \tpt{} returns a map too unstable to reproduce that ranking, whilst a Hutchinson estimator at $M{=}5$ is already stable enough to reproduce it; \tpt{} at that budget is therefore too unstable to be read as evidence against the reversal.  Where the diagonal is not reversed, \tpt{} tracks the diagonal closely in most conditions, at a small fraction of the cost of the exact computation.

We therefore offer \tpt{} as an instrument rather than as a better approximation, and we evaluate it on clinical follow-up prediction from a single reverse chain, on brain magnetic resonance imaging (MRI) and on lung computed tomography (CT).  On brain MRI at full resolution, where every Jacobian-based estimator we test exhausts memory, \tpt{} leads the twenty-chain Monte-Carlo ensemble on five of the eight endpoints measured inside tissue and trails it on none, at $16\times$ fewer network evaluations.  Over the whole volume the ensemble leads instead, on a comparison carried by the background that fills three quarters of it, and a fifty-chain ensemble closes the tissue gap on a thirty-volume subset.  On lung CT the ensemble is ahead on every endpoint, whilst a Hutchinson estimate of that diagonal ranks the error well below \tpt{}.  Both estimators lose most of their discrimination in the regions of longitudinal change, which remains an open problem.
\end{abstract}

\section{Introduction}
\label{sec:intro}

Diffusion models can predict a follow-up medical image from a baseline
scan \citep{oliveras2026mumins, puglisi2024brlp, litrico2026tadm,
liu2025tadpm, pinaya2022brain}, but a plausible prediction does not tell a clinician
where it should be trusted.  The practical object is therefore a
spatial uncertainty map: a per-voxel ranking that directs attention
toward locations where the prediction is most likely to be wrong.  We
study such rankings rather than calibrated predictive standard
deviations (Sec.~\ref{sec:limitations}).

Existing approaches obtain these maps in three ways.  Monte-Carlo
samplers measure variation across reverse chains and multiply
inference cost by the number of chains \citep{wolleb2022seg,
kou2024bayesdiff}: the $20$-chain reference used here costs $1{,}220$
network evaluations per volume.  Analytical methods learn
uncertainty in the noise prediction and propagate it through
approximate reverse updates along a single chain
\citep{capellera2025u2diff, capellera2026hetero, oliveras2026mumins}.
A third family approximates the diagonal of the Tweedie posterior
covariance from the denoiser Jacobian \citep{boys2024tweedie,
peng2024optimal, rout2024beyond, rissanen2025freehunch, ou2025optimal,
schioppa2026covariance}, and its effort has gone into making that
diagonal cheaper to estimate.  Because the exact diagonal requires one
Jacobian--vector product per input coordinate, evaluations normally
compare one approximation with another.  This leaves a basic ambiguity
unresolved: when an uncertainty map fails, is the estimator
inaccurate, or is it faithfully estimating an unsuitable target?

We resolve that ambiguity in two stages.  First, we treat the exact
diagonal as a diagnostic instrument.  On six checkpoints small
enough for exact computation, one of them a public medical model,
fourteen model--corpus conditions in all, we compare it with common
approximations and ask whether the diagonal itself ranks denoising
error.  The approximation is not the bottleneck: Hutchinson estimates
\citep{hutchinson1990stochastic} track the exact diagonal at rank
agreement of at least $0.92$ in every condition, while in four conditions
the exact target is anti-correlated with the error, reaching
$-0.13$, so that a faithful estimator reproduces the reversal.  In
these checkpoints, evaluation on model-generated samples does not
reveal the reversal observed on real images.  Posterior variance need not equal reconstruction
error, and the works that construct the Tweedie diagonal mostly serve
a sampler rather than claim it ranks error
(Sec.~\ref{sec:related}).  But an error-ranking overlay is what a
clinician reading a predicted scan needs, and maps of every family
are deployed that way \citep{kou2024bayesdiff, wolleb2022seg}.  We
therefore ask not whether the field's object was built for that use
but whether it would serve it, which is what decides if a better
approximation of it is worth building.

Second, we respond to the diagnosis with Tweedie Probe-Tangent
(\tpt{}).  \tpt{} repeatedly corrupts one model-supported clean
prediction and measures the voxel-wise variance of the denoiser's
residual response.  \textbf{\emph{Its estimand is a different functional whose ranking
follows the diagonal's where the diagonal reverses; at the working
budget the probe resolves that ranking only where the diagonal is
informative}}, most negative correlation $-0.016$ against the
diagonal's $-0.13$, while Hutchinson at a smaller budget already
reproduces the reversal.  On
longitudinal brain MRI at the full $128^3$ resolution, where the
Jacobian-based estimators do not fit in memory, a single reverse
chain and $K$ perturbations rank the error inside brain tissue as
well as or better than a twenty-chain Monte-Carlo ensemble at
$16\times$ fewer network evaluations.  The natural-image experiments
measure the probe against the exact diagonal (Sec.~\ref{sec:tpt});
the clinical ones test it at the scale the method is for
(Sec.~\ref{sec:exp-setup}).

\section{Related work and evaluation gap}
\label{sec:related}

\paragraph{Sampling-based uncertainty.}
Monte-Carlo diffusion methods estimate spatial uncertainty from the
spread of independent reverse chains \citep{wolleb2022seg,
kou2024bayesdiff, puglisi2024brlp, shaw2025bayes}: a natural empirical
reference, and ours, but linear in the number of chains.  The remaining
families target something other than the spatial ranking of one error
map, whether uncertainty over networks \citep{kou2024bayesdiff,
berry2024shedding, jazbec2025generative, gal2016dropout} or coverage
\citep{angelopoulos2022conformal, vanamersfoort2020duq}, and mostly
need extra training or several models.

\paragraph{Analytical variance propagation.}
Analytical methods propagate uncertainty through the reverse
diffusion process.  Whereas iDDPM learns reverse-transition variances
\citep{nichol2021improved}, U2Diff learns uncertainty in noise
predictions and accumulates it along a DDIM \citep{song2020ddim}
trajectory, neglecting state--prediction covariance and delaying
accumulation to limit variance inflation \citep{capellera2025u2diff};
U2Diffine adds denoiser sensitivity through a first-order expansion
with a diagonal Jacobian approximation \citep{capellera2026hetero}.
MUMINS \citep{oliveras2026mumins} extends U2Diff to 3-D medical
follow-up synthesis along a stochastic reverse chain;
App.~\ref{app:recursion} establishes that the recursion is expansive,
which is why the delayed start is necessary.

\paragraph{Tweedie covariance estimators.}
Recent methods estimate or approximate posterior covariance from the
denoiser Jacobian: a row-sum surrogate from one Jacobian--vector
product \citep{boys2024tweedie}, a wavelet diagonal
\citep{peng2024optimal}, an isotropic approximation
\citep{rout2024beyond}, low-rank running estimates
\citep{rissanen2025freehunch}, trained heads \citep{ou2025optimal},
and DCT block diagonals \citep{schioppa2026covariance}.  Analytic-DPM
derives an optimal diagonal under an imperfect mean
\citep{bao2022estimating}, and CA-DPS uses finite differences for
inverse problems \citep{hamidi2024cadps}.  These works improve access
to an intractable covariance object, and each serves a sampler, where
off-diagonal structure matters because one draws from the matrix.  We
claim no part of the estimand and no improvement to its approximation;
what we add is the exact object it is approximating, and what that
object turns out to say.

Algebraically \tpt{} is a perturbation-based sensitivity estimate, in
the family of test-time augmentation for aleatoric uncertainty
\citep{ayhan2018tta, wang2019aleatoric}: repeated corruptions of one
input, and the spread of the response.  Closest of all,
\citet{devita2025pixelwise} denoise to $\hat x_0$, re-noise it $M$ times
and take the per-pixel variance of the \emph{scores} $\epsilon_\theta$,
to guide a sampler and filter samples by FID; we take the variance of
the residual $\xi - \epsilon_\theta$ instead, which reads a different
functional of the same Jacobian (App.~\ref{app:openq}) and a different
map on both clinical tasks (App.~\ref{app:controlled}).  The
corruption is the forward process itself,
at the timestep $t^\ast$ the denoiser was trained on, which gives the
residual the geometric reading of Prop.~\ref{prop:pt}; keeping the
residual rather than the reconstruction is not a further difference,
the two rank identically (App.~\ref{app:controlled}).  \citet{stanczuk2024intrinsic} also
perturb a point repeatedly, but apply an SVD to score vectors to
estimate one intrinsic-dimension scalar; \tpt{} retains the per-voxel
residual variance and forms no matrix.

\section{Diagnosis: is the exact diagonal a useful target?}
\label{sec:exact}

\paragraph{What is computed.}
For a denoiser $\epsilon_\theta$ and a noised state $x_t$, Tweedie's
identity gives the posterior mean $\hat{x}_0(x_t)$ and, differentiating
it, the posterior covariance up to schedule factors; the per-voxel map
in use is its diagonal,
$\sigma_i^2 = \frac{1-\bar\alpha}{\sqrt{\bar\alpha}}\,
\partial \hat{x}_{0i}/\partial x_{ti}$ (App.~\ref{app:budgets}).
The exact diagonal costs
$d$ Jacobian--vector products, one per one-hot tangent; we compute
all $d$: $12{,}288$ for a $64^2$ RGB image, $3{,}072$ at $32^2$,
$12{,}288$, $4{,}096$ and $4{,}096$ in the three latent spaces.  Forward-mode
automatic differentiation is used throughout, under \texttt{no\_grad}
(App.~\ref{app:checkpointing} explains why none of the three public
codebases can be traversed with \texttt{torch.func} as released).

\paragraph{Design.}
Six checkpoints: \texttt{inet} and \texttt{cifar}, the unconditional
ImageNet-64 \citep{deng2009imagenet} and CIFAR-10
\citep{krizhevsky2009cifar} models of \citet{nichol2021improved};
\texttt{cond}, the class-conditional ImageNet-64 model of
\citet{dhariwal2021diffbeatgans}; \texttt{ffhq} and
\texttt{church}, latent-diffusion models \citep{rombach2022ldm} for
FFHQ \citep{karras2019stylegan} and LSUN churches
\citep{yu2015lsun}; and \texttt{brats}, a public MONAI
latent-diffusion bundle \citep{pinaya2023monaigen} for axial MRI brain slices from BraTS dataset
\citep{menze2015brats} (App.~\ref{app:conditions}).  Each is evaluated under up to three
corpora: \texttt{coco} (held-out photographs
\citep{lin2014coco}), \texttt{self} (the model's own samples) and
\texttt{real} (held-out images from the training domain, Imagenette
\citep{howard2019imagenette} for the ImageNet checkpoints, and
unavailable for the latent models).  A condition such as \texttt{cond/self}
therefore names a checkpoint--corpus pair, and fourteen such pairs
exist.  These checkpoints denoise single images; the
baseline--follow-up structure of Sec.~\ref{sec:exp-setup} plays no
part in this track.  The three corpora are the experiment, not a robustness check
(Sec.~\ref{sec:res-sign}).  One hundred images per condition, a seeded random subset; the probe
timestep matches the clinical runs' forward marginal scale
($\sigma_{t^\ast}= 0.32$, from $t^\ast=60$ under their $T{=}300$ cosine
schedule; Sec.~\ref{sec:exp-setup}), so the two halves of the paper measure at
comparable noise.  Every estimator's cost is counted in the same unit,
forward-equivalent network evaluations, with the conversion fixed in
App.~\ref{app:budgets} before any score was read.  The denoising
error $e$ is $|\hat x_0(x_{t^\ast}) - x_0|$ per pixel, averaged over
twenty noise draws (App.~\ref{app:defs}).
\subsection{Validity of the implied covariance}
The bracket $1-\sqrt{1-\bar\alpha}\,\operatorname{diag}(J)$ can be
negative.  In the one-hot runs it never is on the two pixel-space
checkpoints, and it is on up to $5.7\%$ of coordinates on the
class-conditional one (Table~\ref{tab:conditions}).  A negative entry
there is not an estimation error, because the diagonal is exact; it
says that the covariance implied by the learned denoiser is not a
covariance.  We clamp at zero and report the clamped fraction
(App.~\ref{app:defs}).  The probe introduced in Sec.~\ref{sec:tpt} cannot produce this, since it reports a sample standard deviation directly; its clamped fraction
is exactly zero.

\subsection{Fidelity is not the binding constraint}
\label{sec:res-fidelity}

The Hutchinson estimator converges to the exact diagonal; the convergence, however,
buys nothing.  On the class-conditional ImageNet checkpoint scored on
its own samples (\texttt{cond/self}, Table~\ref{tab:conditions}),
Hutchinson's Spearman correlation with the denoising error rises with
$M = 5, 15, 50, 200$ probes as $0.298$, $0.359$, $0.410$, $0.435$,
approaching the $0.448$ that the exact diagonal itself achieves
(Table~\ref{tab:mgrid}).  The
pattern holds in all fourteen conditions: at $M{=}200$ Hutchinson
matches the exact diagonal at rank correlation $0.92$--$0.996$, and
its correlation with the error climbs to a limit that is the
diagonal's own and no better.  Fidelity to the diagonal does not
guarantee usefulness of the diagonal: the approximation is loosest on
the class-conditional checkpoint, where the diagonal ranks the error
best (relative $L_2$ of $8\times10^{-2}$), and tightest on CIFAR,
where the diagonal is anti-correlated with the error on both
real-image corpora ($3\times10^{-4}$).  On \texttt{cond/self},
\tpt{} reaches $0.451$ at $30$ forwards, against roughly $49{,}000$
measured forward-equivalents for the exact object
(App.~\ref{app:budgets}).  The row-sum surrogate fails differently:
it does not converge to the diagonal at all, and its correlation with
the error is negative or null in ten of the fourteen conditions,
reaching $-0.232$ (Table~\ref{tab:conditions}).  A better approximation of this diagonal is
therefore not a route to a better map.

\subsection{The failure mode, and what it depends on}
\label{sec:res-sign}

Table~\ref{tab:conditions} reports the Spearman correlation between
the exact diagonal and the denoising error in each of the fourteen
model--corpus conditions; every $95\%$ interval excludes zero except
\texttt{ffhq/self}, which straddles it.  The four negative rows are the failure mode of
Sec.~\ref{sec:intro} (a map that ranks error-prone voxels below
error-free ones), and the table narrows where it lives.  Being out of distribution does not explain the
sign: Imagenette, the unconditional checkpoint's own training domain,
gives $-0.131$ next to held-out COCO's $-0.132$, so on the one
checkpoint where both corpora exist, the domain leaves the sign
untouched.  What separates the
conditions is real images against the model's own samples: the two
real-image corpora agree closely, while on their own output five of
the six checkpoints turn clearly positive and the sixth stays
non-negative.  For the two pixel-space
unconditional checkpoints the flip is visible within a single model
across all three corpora (\texttt{inet}: $-0.132$, $-0.131$, $+0.058$;
\texttt{cifar}: $-0.091$, $-0.077$, $+0.029$).

The reversal belongs to one estimand: it is a property of the
diagonal, inherited by any estimator of it but not by Monte-Carlo
sampling or by the variance recursion, which estimate different
objects.  On the two clinical checkpoints, where the exact diagonal is
out of reach, every computable estimator correlates positively with
the error (Table~\ref{tab:clinical}).

\begin{figure}[t]
\centering
\includegraphics[width=\linewidth]{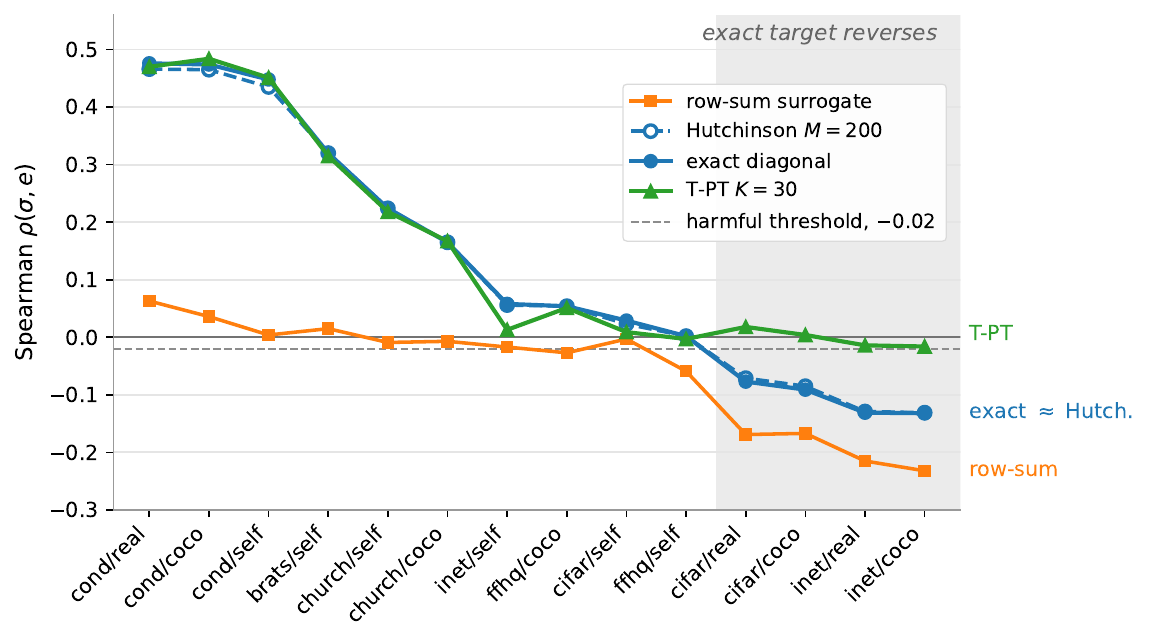}
\caption{Spearman correlation with the denoising error, per condition,
mean over $100$ images, ordered by the exact diagonal
(Table~\ref{tab:conditions} gives the numbers and the intervals,
in its own ordering).  In the shaded block the exact diagonal is
\emph{anti}-correlated with the error, and Hutchinson at $M{=}200$
follows it there as everywhere else, so what fails is the target and
not its approximation.  A condition is \emph{harmful} when its
interval lies entirely below $-0.02$ (dashed), a threshold fixed on
the synthetic manifolds of App.~\ref{app:toy} before any
natural-image score was read: four of fourteen for the exact diagonal,
four for Hutchinson, six for the row-sum, and none for \tpt{} at $K =
30$, whose map runs flat along zero there rather than tracking the
diagonal.  Sec.~\ref{sec:res-robust} and App.~\ref{app:openq} show
that this flatness is a property of that budget: with $4$--$16\times$
the probes \tpt{} is harmful on the two \texttt{inet} conditions too.}
\label{fig:conditions}
\end{figure}

\section{Response: Tweedie Probe-Tangent, a one-chain residual probe}
\label{sec:tpt}

If the target is what fails, a better approximation of it is not the
remedy, and the estimator we introduce is not one either: it is a
different functional whose second-order expansion ranks with the
diagonal where the diagonal reverses.  This section defines it, gives it a
geometric reading, and returns to the fourteen conditions to show
what it delivers at its budget.

\subsection{Estimator and cost}
\label{sec:tpt-estimator}

Given a model-supported clean state $x_0 \in \M$, fix a probe timestep
$t^\ast$, draw $K$ independent $\xi_k \sim \mathcal{N}(0, I)$ and form
\begin{equation}
    X_{t^\ast}^{(k)} := \sqrt{\bar\alpha_{t^\ast}}\, x_0
        + \sqrt{1 - \bar\alpha_{t^\ast}}\, \xi_k.
    \label{eq:xt}
\end{equation}
The \emph{Tweedie residual} is
$r_k := \xi_k - \epsilon_\theta(X_{t^\ast}^{(k)}, t^\ast, c)$, with
$c$ the conditioning where the model has one (Sec.~\ref{sec:tpt-indist});
the estimator is the per-voxel sample variance across draws,
\begin{equation}
    \widehat{\Var}(r_i) := \frac{1}{K-1}
    \sum_{k=1}^K \bigl(r_k[i] - \bar r[i]\bigr)^2,
    \qquad
    \bar r[i] := \frac{1}{K}\sum_{k=1}^{K} r_k[i],
    \label{eq:tpt}
\end{equation}
and the reported map is
$\sigma_{\text{TW}}[i] := \sqrt{\widehat{\Var}(r_i)}$.
Throughout, $K$ is the number of independent draws
$\xi_1, \dots, \xi_K \sim \mathcal{N}(0, I_d)$ at the single
timestep $t^\ast$: each costs one evaluation of $\epsilon_\theta$
and yields one residual $r_k \in \mathbb{R}^d$, and the map is the
unbiased sample variance of $\{r_k[i]\}_{k=1}^K$ at every voxel, so
$K$ is at once that variance's sample size and the number of network
evaluations the probe adds to its chain.  One set of $K$ draws serves
every voxel of a map; runs at different $K$ use independent draws.
One timestep suffices: the probe reads local sensitivity at one
noise scale, where Prop.~\ref{prop:pt} holds, and the admissible
band is wide (App.~\ref{app:tstar}).  The draws batch into one
call; nothing is differentiated, and the estimator runs under
\texttt{no\_grad}.  The probe costs one
reverse chain of $T/\zeta + 1$ evaluations plus $K$ forwards.  On the
two clinical tasks of Sec.~\ref{sec:exp-setup} that is $91$ forwards
on lung Computed Tomography (CT) at $K{=}30$ and $76$ on brain Magnetic Resonance Imaging (MRI) at $K{=}15$, against
$J(T/\zeta+1) = 1{,}220$ for MC-DDIM at $J{=}20$;
App.~\ref{app:kablation} sweeps $K$.  We count network evaluations
throughout; batching is an implementation option, not a measured
speed-up.

\subsection{Geometric interpretation}
\label{sec:tpt-geom}

\noindent
Two measures are in play, and we keep them apart.  Under the
\emph{joint} law $X_0 \sim p_{\text{data}}$, $\xi \sim
\mathcal{N}(0,I)$, conditional expectation is the orthogonal
projection in $L^2$, so $r = \xi - \E[\xi \mid X_t]$ is the component
of $\xi$ orthogonal to the square-integrable
$\sigma(X_t)$-measurable functions, and $\Cov(r)$ is the covariance
of the injected noise that $X_t$ does not identify.  The estimator,
however, \emph{fixes} $x_0$ and varies only $\xi$, so we write
$C_t(x_0) := \Cov_\xi\bigl[\xi - \epsilon^\ast(X_t,t)\bigr]$ for that
conditional object, with $\epsilon^\ast$ still the population-optimal
denoiser under the joint law; the probe estimates $C_t(x_0)$.

\begin{proposition}[The residual is the unidentifiable part of the noise]
    \label{prop:pt}
    On a flat manifold $\M = x_0 + T$, $\operatorname{range}(C_t)
    \subseteq T$ for any tangent law.  If the tangent coordinate is
    Gaussian, $u \sim \mathcal{N}(0,\tau^2\PT)$, then
    \begin{equation}
        C_t(x_0) = \gamma_t^2\,\PT,
        \qquad
        \gamma_t = \frac{\bar\alpha_t\tau^2}{\bar\alpha_t\tau^2+\sigma_t^2}.
        \label{eq:flat-exact}
    \end{equation}
    The sample covariance of $K$ residuals converges almost surely to
    $C_t(x_0)$.  Proofs and the hypotheses they need are in
    App.~\ref{app:proof}.
\end{proposition}
\noindent
The translation-invariant case, $C_t(x_0) = \PT$ for every $t$, is the
$\tau \to \infty$ limit of Eq.~\eqref{eq:flat-exact}, an improper
idealisation rather than a law: the geometry fixes the leading
eigenspace, while tangent density, noise level, denoiser error and
finite $K$ all move the measured eigenvalues
(App.~\ref{app:toy}).
\paragraph{Two diagonals, and why they are not the same.}
The object of Prop.~\ref{prop:pt} is the diagonal of $C_t(x_0)$, the
covariance of the unidentified component of the injected noise, which on
a flat manifold is $\gamma_t^2\PT$, a projection, up to a scalar, whose
per-voxel value is the tangent leverage $\PT[i,i]$.  The family surveyed
in Sec.~\ref{sec:intro} estimates a different diagonal: that of
$\partial \hat{x}_0/\partial x_t$, the coordinate-wise sensitivity of
the posterior mean that Tweedie's identity \citep{efron2011tweedie}
ties to the full posterior covariance.  When the posterior-mean Jacobian is an orthogonal projector (the flat, affine case), the two diagonals carry the same ranking in any tangent frame; beyond that regime they separate, and to second order the probe additionally carries the row energy of the Jacobian (Sec.~\ref{sec:res-robust}).

This distinction is between estimands and is drawn before any
measurement: a faithful estimator of
$\operatorname{diag}(\partial\hat{x}_0/\partial x_t)$ must reproduce
its sign, while which sign $\operatorname{diag}(C_t)$ takes is an
empirical question that Sec.~\ref{sec:res-robust} answers.

Off the manifold, the leading covariance is governed by the
differential of the nearest-point projection, $\mathrm{D}\pi_\M(x^\ast)
= (I + hW_\nu)^{-1}\PT$ with $W_\nu$ the Weingarten map, so the
tangent eigenvalues of $C_t$ become $(1 + h\kappa_i)^{-2}$ in the
principal curvatures $\kappa_i$: curvature at distance $h$ shrinks the
covariance along the directions it bends
(Thm.~\ref{thm:curved}, App.~\ref{app:curved}, tested against an exact
denoiser in App.~\ref{app:curvtest}).
\paragraph{Why a geometric quantity should order an error map.}
$C_t$ describes what the noisy observation fails to identify, not the
error of any one reconstruction, and nothing above makes the two
equal.  A large $\sigma_{\text{TW}}$ marks a voxel where the implied
clean prediction is unstable to how the observation was corrupted, and
a sampler drawing a clean state can be wrong by more there: that is
why we expect the map to order the error, and it fixes no scale, so
every endpoint here is rank-based.  The probe concerns the model's
posterior rather than the anatomy, so a confidently wrong
reconstruction is probed as confidently as a correct one
(Sec.~\ref{sec:limitations}).
\subsection{Conditional prediction and the probe point}
\label{sec:tpt-indist}

\tpt's probe (Sec.~\ref{sec:tpt}) requires a full clean state
$x_0 \in \M$.  In the conditional setting, $x_0 = (u_0, v_0)$ splits into a known
baseline $u_0$ and the follow-up $v_0$ the model is predicting, so
half of $x_0$ is not directly available, and this section constructs a
usable probe point.  The models here learn the joint state, and at
inference the baseline channel is re-injected onto the noised
observation at every reverse step \citep{oliveras2026mumins}.

Under the local-product hypothesis of App.~\ref{app:fibred} we model
this state as a fibration over plausible baselines: the tangent space
splits into horizontal and vertical parts, and for an ideal denoiser
conditioned on the clean baseline, measurement injection suppresses
the horizontal residual, leaving a progression-conditional component, not ``how uncertain is this image'' but ``given this patient, how
plausible is this progression''.  Our network receives the noised
baseline, so for the implemented residual we claim only predominant
fibre-alignment, with horizontal contamination from conditioning and
model error (Prop.~\ref{thm:fibred}, App.~\ref{app:fibred}).

The probe point is load-bearing.  The training objective constrains
the denoiser on and near the training marginal and does not identify
its off-support behaviour, so off-support residuals carry no licensed
information about $\M$.  Setting the unknown channel to $v_0 = 0$
probes a point no training pair ever visited, and the estimator fails
\emph{quietly}: the map is smooth,
non-negative and spatially varying, yet it correlates with the true
error at $\bar\rho \approx 0.004$ on lung CT.  Using the reconstruction of
one additional DDIM chain as the probe, changing nothing else, moves
$\bar\rho$ to $0.65$ (Table~\ref{tab:probeinput},
App.~\ref{app:controlled}).  Probe adequacy is therefore a condition
we validate against a reference rather than derive from the theory.

\subsection{Two estimands that share a reversal}
\label{sec:res-robust}

\tpt{} estimates the diagonal of $C_t$; the family of
Sec.~\ref{sec:exact} estimates the diagonal of the posterior-mean
Jacobian.  They are different functionals, separated by
Eq.~\eqref{eq:expansion} below, but on the four reversed conditions
their \emph{rankings} coincide: the second-order map of
Eq.~\eqref{eq:expansion} ranks identically to the exact diagonal and
reverses with it (Table~\ref{tab:rowenergy}).  A shared ranking is
not a shared estimand, but there it is enough --- a probe that
resolved that ranking would reverse too.  At $K = 30$ it does not, with no
interval in the harmful region, against four of fourteen for the exact
diagonal and for Hutchinson at $M{=}200$ (Fig.~\ref{fig:conditions}).
The two therefore target the same ranking here and differ in how hard
that ranking is to estimate, which Table~\ref{tab:kbudget} measures.

\paragraph{The two are not the same functional of $J$.}
Let us write
$a=\sqrt{1-\bar\alpha_{t^*}}=\sigma_{t^\ast}$ for the probe's perturbation scale, take
$X^{(k)}=\sqrt{\bar\alpha}\,\hat x_0+a\,\xi_k$ and
$r_k=\xi_k-\epsilon_\theta(X^{(k)}, t^\ast, c)$, and linearise $\epsilon_\theta$
about the operating point with Jacobian
$J := \partial \epsilon_\theta/\partial x$:
\begin{equation}
r_k \;\approx\; (I-aJ)\,\xi_k-c,
\qquad
\operatorname{Var}_k(r_i)\;=\;1-2a\,J_{ii}+a^{2}\,\|J_{i:}\|^{2}+O(a^{3}).
\label{eq:expansion}
\end{equation}
At the probe point $a = 0.32$, the quadratic term is not a
correction to be dropped, and the $O(a^3)$ remainder is smaller than
it by one factor of $a$ only: Eq.~\eqref{eq:expansion} identifies
which functionals the estimators read, not the size of their
difference.  In the ideal
flat case the tangent projector is $I - aJ$, not $J$: then $aJ = \PN$,
$\|J_{i:}\|^2 = J_{ii}/a$, $\operatorname{Var}_k(r_i) = 1 - a\,J_{ii}$,
and the two diagonals rank alike.  The separation measured below is
therefore a measure of the denoiser's departure from projector
structure, consistent with the negative diagonal entries of
Sec.~\ref{sec:exact}, which a projector cannot produce.  Away from
that case the exact diagonal depends only on $J_{ii}$; the row-sum
returns $\sum_j J_{ij}$, in which off-diagonal mass of opposite sign
cancels; the probe adds the row energy $\|J_{i:}\|^{2}$, which on the
reversed conditions ranks with $J_{ii}$ and changes nothing there.  The
row energy alone is a fourth functional, the one a score-variance probe
reads \citep{devita2025pixelwise}, and it mirrors the diagonal rather
than escaping it (App.~\ref{app:openq}).

\paragraph{Where the map resolves, the probe carries more than the diagonal.}
On twenty $64^3$ lung-CT volumes, against a converged Hutchinson
reference at $M{=}500$, the probe retains $+0.409$ conditioned on the
reference and the reference $+0.056$ conditioned on the probe
(Table~\ref{tab:functional}); the same reference at $M{=}50$ adds only
$+0.056$.  The clinical map is the one that resolves, so its contrast
with the natural-image partials of
Table~\ref{tab:functional-natural} is a probe budget and not two
verdicts on the functional (App.~\ref{app:openq}).

\paragraph{The probe does not escape the reversal, and at its working
budget does not resolve its own estimand.}
Eq.~\eqref{eq:expansion} is a constant modulated by $aJ_{ii}$, and at
$K = 30$ the sample variance carries a $26\%$ relative error.  Where
the Jacobian is strong, that resolves the modulation: on
\texttt{cond/self} two disjoint groups of $30$ probes agree at
$0.86$.  Where it is weak, which is where the diagonal reverses, it
does not: the same agreement is $-0.013$ to $0.017$
(Table~\ref{tab:kbudget}).  As the budget grows, reproducibility
rises on the two \texttt{inet} conditions and the anti-correlation
emerges with it, $-0.016$, $-0.041$, $-0.079$ at $K = 30$, $120$,
$480$; on \texttt{cifar} neither happens at any budget tried, and an
approximate attenuation model accounts for all four within $0.02$
(App.~\ref{app:openq}).  Hutchinson at $M = 5$, about $20$
forward-equivalents, is already harmful on all four, so at matched
budget one estimator reproduces the reversal, and the other returns a
map too unstable to reproduce the ranking it estimates
(Table~\ref{tab:kbudget}).

\section{Clinical evaluation: brain MRI and lung CT}
\label{sec:experiments}\label{sec:exp-setup}\label{sec:results}

\subsection{Design}
We evaluate two 3-D medical diffusion models from the MUMINS
follow-up-prediction pipeline \citep{oliveras2026mumins}, distinct
from the six single-image checkpoints of Sec.~\ref{sec:exact}.
Both were retrained on public data using their shared code.  Each is a 3-D U-Net \citep{cicek2016unet} trained
with an $\ell_1$ noise-prediction loss plus a Gaussian noise-prediction
NLL, adapted from U2Diff \citep{capellera2025u2diff}, on a
two-channel joint state (Sec.~\ref{sec:tpt-indist}).  One predicts
CT lung-nodule progression at $64^3$ on the PNG pulmonary-nodule growth
dataset, released with \citet{tang2026ngpnet}, with $130$ test pairs
from $47$ participants;
the other predicts MRI brain progression at $128^3$ on OASIS-3
\citep{lamontagne2019oasis3}, with $211$ scored pairs from $63$
participants.  Both test splits are patient-disjoint from the training
cohorts, so no participant contributes volumes to both.  The two
models jointly predict the noise $\epsilon$ and a per-voxel variance,
following the same U2Diff-derived NLL head
\citep{capellera2025u2diff}; that variance head feeds the analytical
baseline $\sigma_\theta$ below and is not used by \tpt{}, and the
baseline channel $u_0$ is never predicted, only re-injected at
inference (Sec.~\ref{sec:tpt-indist}).  Within each task, every
estimator uses the same network.  The two tasks use different U-Nets,
because the architecture's depth follows the input size.

MC-DDIM is the per-voxel standard deviation of $J{=}20$ reverse
chains, generated for this work in the deployable conditioning arm
with twenty distinct seeds.  The analytical baseline $\sigma_\theta$
propagates the variance head through the reverse update, with the
standard delayed-start mask $\shat = 15$, rechecked on our
retrained checkpoints' own held-out validation likelihood following
the MUMINS selection protocol.  The delay is forced rather than
tuned, because within the affine additive class the forcing term
cancels from the comparison, so no variance head, injected-noise
schedule or additive correction makes the recursion contractive, and a
delayed start can only bound the accumulated gain
(Prop.~\ref{prop:expansive}, App.~\ref{app:recursion}).  \tpt{} uses
$t^\ast = 60$, $K = 30$ on lung CT and $15$ on brain MRI, and one
additional DDIM chain as its probe; both were fixed before any
participant-level score was read, the sweeps over $t^\ast$ and $K$ were
run on lung CT, and probe-seed sensitivity is in
App.~\ref{app:controlled} (Apps.~\ref{app:tstar},~\ref{app:kablation}).

Every map is scored against one common error map, the error of the
$J{=}20$ ensemble-mean prediction, so that the three estimators are
ranked against the same target.  That map is built from the same
twenty chains whose spread is $\sigma_{\text{MC}}$, which favours the
reference; App.~\ref{app:controlled} therefore also scores each
estimator against the error of the prediction it would accompany at
deployment (one chain for \tpt{}, the $J$-chain mean for MC-DDIM),
where that advantage is absent.  The endpoint is per-volume Spearman
$\rho(\sigma,e)$: rank, because $\sigma$ and $e$ live on different
scales; per-volume, because pooling lets global intensity differences
pass as localisation.  Tables write this correlation as $\rho$ and a
paired difference of it between two estimators as $\Delta\rho$.  A
bracket is the $95\%$ interval of the number to its left, never of the
row.  OASIS additionally uses a prespecified
\emph{tissue} endpoint (App.~\ref{app:defs}), with the background being three
quarters of a skull-stripped volume.  Paired intervals resample \emph{participants} on both tasks, up to
$21$ pairs from one participant on OASIS, and a volume-level
bootstrap covers a known simulated effect $57\%$ of the time against
a nominal $95\%$.

\begin{table}[t]
\centering\scriptsize
\setlength{\tabcolsep}{2.2pt}
\begin{tabular}{l rrr|l rr|l}
\toprule
 & \multicolumn{4}{c}{Lung CT (PNG, $N{=}130$)} & \multicolumn{3}{c}{Brain MRI (OASIS, $N{=}211$)} \\
\cmidrule(lr){2-5}\cmidrule(lr){6-8}
 & \tpt{} & MC & Hutch. & $\Delta$ [95\% CI] & \tpt{} & MC & $\Delta$ [95\% CI] \\
\midrule
$\rho$, full volume       & 0.669 & \textbf{0.693} & 0.404 & $-0.022\,[-0.037,-0.007]$ & 0.518 & \textbf{0.559} & $-0.038\,[-0.041,-0.035]$ \\
\midrule
$\rho$, tissue            & 0.670 & \textbf{0.694} & 0.405 & $-0.022\,[-0.037,-0.007]$ & \textbf{0.349} & 0.337 & $+0.016\,[+0.006,+0.028]$ \\
$\rho$, interior             & 0.612 & \textbf{0.633} & 0.326 & $-0.023\,[-0.043,-0.005]$ & \textbf{0.334} & 0.318 & $+0.022\,[+0.009,+0.037]$ \\
$\rho$, edges                & 0.352 & \textbf{0.528} & 0.177 & $-0.147\,[-0.186,-0.110]$ & 0.248 & 0.245 & $+0.003\,[-0.005,+0.010]$ \\
$\rho$, change region        & 0.100 & \textbf{0.265} & 0.033 & $-0.134\,[-0.183,-0.085]$ & 0.128 & 0.133 & $-0.003\,[-0.013,+0.008]$ \\
$\rho$, partial (grad.\ rm.) & 0.339 & \textbf{0.409} & 0.143 & $-0.060\,[-0.090,-0.031]$ & \textbf{0.290} & 0.274 & $+0.021\,[+0.009,+0.035]$ \\
AUROC, top-5\% error       & 0.884 & \textbf{0.922} & 0.753 & $-0.027\,[-0.039,-0.017]$ & 0.782 & 0.782 & $-0.001\,[-0.007,+0.004]$ \\
AUSE $\downarrow$          & 0.152 & \textbf{0.133} & 0.383 & $+0.014\,[+0.004,+0.024]$ & \textbf{0.306} & 0.315 & $-0.011\,[-0.019,-0.004]$ \\
AURC $\downarrow$          & 0.030 & \textbf{0.029} & 0.047 & $+0.001\,[+0.000,+0.002]$ & \textbf{0.053} & 0.054 & $-0.001\,[-0.002,-0.000]$ \\
\midrule
forwards per volume        & \textbf{91} & $1{,}220$ & $61{+}200$\,JVP & & \textbf{76} & $1{,}220$ & \\
\bottomrule
\end{tabular}
\caption{Clinical results against the common error map.  Estimator
columns are per-volume means; $\Delta$ is \tpt{} minus MC-DDIM, so on
the two rows marked $\downarrow$ a positive $\Delta$ favours MC-DDIM,
with its paired $95\%$ confidence interval resampling participants
($63$ on OASIS, $47$ on lung CT, where the full volume is $99.5\%$
tissue).  Hutch.\ is Hutchinson at $M{=}200$ on the probe's own chain,
which has no brain column because at $128^3$ it exhausts memory
(Sec.~\ref{sec:res-cost}); on lung CT,
the probe's chain, the reference and the error map share one
conditioning (App.~\ref{app:defs}).  Bold marks the better estimator
where the interval excludes zero.  Volume-weighted contrasts:
Table~\ref{tab:paired-oasis}.}
\label{tab:clinical}\label{tab:masked}
\end{table}

\subsection{Results}
\label{sec:res-clinical}
\label{sec:exp-masked}\label{sec:exp-aggregate}

On brain MRI ($211$ pairs, $63$ participants) each row of
Table~\ref{tab:clinical} is a paired per-volume difference with a
cluster bootstrap over participants, and the two halves of the table
point in opposite directions.  Over the whole volume the $20$-chain
reference is ahead by $0.038$ and wins on $210$ of the $211$
pairs.  Inside tissue the ordering reverses: \tpt{} leads on tissue
Spearman, the interior, the gradient-controlled partial, and AUSE (and
AURC, Table~\ref{tab:paired-oasis}), at $76$ forwards against
$1{,}220$; on edges, in the change region, and on AUROC, the interval
includes zero.  Three quarters of a skull-stripped volume is
background, where both maps and the error are near zero and agree for
free, so which reading matters depends on whether the map is to be
used inside the anatomy.  The analytical baseline $\sigma_\theta$
(App.~\ref{app:controlled}) is not distinguishable from the ensemble
over the full brain volume, last by a wide margin inside tissue, and
below both on lung CT.

On lung CT, \tpt{} loses to the ensemble on every endpoint.  Against
the common error map, the paired tissue difference is
$-0.022\;[-0.037, -0.007]$ in MC-DDIM's favour, and
Table~\ref{tab:clinical} places the reference's advantage where the
structure is: on edges ($0.528$ against $0.352$) and inside the change
region ($0.265$ against $0.100$).  A second lung-CT checkpoint,
retrained without the variance head (App.~\ref{app:nohead}), reaches
$0.633$ against the reference's $0.636$, $-0.020\;[-0.040, -0.001]$
(Table~\ref{tab:paired}), so the probe needs no variance head to run.
At $64^3$ the Jacobian family is computable: Hutchinson at
$M{=}200$ on the probe's own chain reads $0.405$ inside lung tissue
against $0.670$ ($+0.189\;[+0.158, +0.222]$) and trails on every
row; on twenty volumes the probe retains $+0.409$
conditioned on the converged diagonal, against $+0.056$ the other
way (App.~\ref{app:functional}).  Scored against its own single-chain
prediction, the deployment case, \tpt{} loses $0.009$ inside brain
tissue and $0.007$ on lung CT, and over the whole brain volume
MC-DDIM's own-error figure is $0.552$ against \tpt{}'s $0.520$
(Table~\ref{tab:aggregate}).  One interpretation, which these
experiments do not test and uncropped scans would: a lung-CT volume is
$99.5\%$ foreground, so the tissue restriction that reverses the OASIS
ordering has almost no background to exclude, and over the full volume
OASIS orders the same way ($-0.038$).

\subsection{Cost is a feasibility boundary, not a ratio}
\label{sec:res-cost}\label{sec:exp-cost}\label{sec:exp-convergence}

At $128^3$ on a $40$~GB accelerator, run one at a time in clean
processes, the row-sum surrogate and Hutchinson at $M{=}50$ and
$M{=}500$ each exhaust memory before completing one volume, while
\tpt{} peaks at $33.8$~GiB and completes in $66.7$~s
(Table~\ref{tab:memprobe}).  The three that fail are the
Jacobian-based ones: forward-mode differentiation doubles the stored
activations, and the probe differentiates nothing.  The boundary is
that of the released implementations, with activation checkpointing
disabled for the tangent pass (App.~\ref{app:checkpointing});
batching tangents would not move it, since each Jacobian--vector
product already carries one tangent and the peak is that of a single
forward-mode pass.

\section{Discussion and limitations}
\label{sec:limitations}

\paragraph{The estimate loses its information where the change is, and so
does the reference.}
The clinical task is to predict how a structure evolves.  Restricted
to the tissue that actually changed between visits, the top decile of
$|x_{\text{target}} - u_0|$ computed without any model output, the
OASIS tissue Spearman of $0.349$ falls to $0.128$, while random tissue
subsets of the same size score $0.378$, so the drop is specific to
where the change is.  The
$20$-chain reference falls to $0.133$ (Table~\ref{tab:clinical}): the
failure belongs to the family, not to its cheapest member.  Two
untested mechanisms are consistent with it (App.~\ref{app:openq}), one
binding every estimator here: a confident prediction receives low
variance whether it is right.

\paragraph{Scope.}
The OASIS volumes were resampled to $1.5$~mm isotropic before
training, so every clinical number is scoped to that preprocessing
(App.~\ref{app:openq} shows the structure it would have erased is
still present where $\sigma$ collapses).  Why $\operatorname{diag}(J)$ should be \emph{anti}-correlated with the
error rather than merely uninformative remains open (four accounts
refuted in App.~\ref{app:openq}).  Six checkpoints are not a survey;
the clinical evidence is two tasks from the two MUMINS checkpoints,
with the functional separation resting on one of them and on the four
reversed conditions, and the reference is still gaining tissue
discrimination at $J{=}20$, its $J{=}50$ advantage narrowing to
$+0.010$ (Table~\ref{tab:j50}).  The ranking depends on the probe
point lying on the data manifold (Sec.~\ref{sec:tpt-indist}) and on
one timestep.  Where the diagonal reverses the $K = 30$ map is barely
reproducible, so its absence from the harmful column is a fact about
the estimator at that budget, not a guarantee about the map
(Table~\ref{tab:kbudget}); and whether the reversal occurs in clinical
checkpoints these tasks cannot answer, the one medical checkpoint on
which the diagonal is computable being two-dimensional and
unconditional.
\section{Conclusion}
\label{sec:conclusion}

In four of fourteen conditions the exact diagonal is anti-correlated
with the error it is meant to rank, and only exact evaluation shows
that the failure belongs to the target.  We offer the probe as an
instrument rather than as a better approximation, since its
second-order ranking follows the diagonal there, and at thirty probes
its map is too unstable to reproduce that ranking.  What it buys is
feasibility at $128^3$, where every Jacobian-based estimator we test
runs out of memory, and, where its own map resolves, more about the
error than a Hutchinson estimate of that diagonal; the changed region
remains open.

\section*{Reproducibility statement}
Every claim is backed by material we release.
Apps.~\ref{app:proof}--\ref{app:curved} prove the affine and
curved-support results under the assumptions stated there;
App.~\ref{app:probe} derives the schedule-only budget rule, so
Fig.~\ref{fig:blowup} is reproducible from the schedule in a few
lines.  The synthetic manifolds of App.~\ref{app:toy} ship with the
code, including training scripts, seeds, and per-number protocols.  For
the medical experiments, Sec.~\ref{sec:exp-setup} and the
configuration index of App.~\ref{app:defs}
(Table~\ref{tab:configindex}) record checkpoints, samplers, schedules,
probe settings and test splits;
App.~\ref{app:defs} defines every mask and metric operationally;
App.~\ref{app:controlled} states the resampling unit behind every
interval and validates it on simulated data.  Source code for the estimator, the exact-denoiser and
synthetic-manifold checks, the scoring pipeline and the statistical
analysis is included in the supplementary material, together with the
training configurations and the preprocessing.  Weights are not
redistributed: the OASIS-3 data-use agreement does not cover
derivatives of that cohort, so the brain checkpoints have to be
retrained from the released configuration against an approved copy of
the data, and we release the recipe rather than the model.  No image
data are redistributed either; the release documents how to request
access and reproduce the preprocessing.

\section*{Ethics statement}
This work uses two de-identified, publicly released medical imaging
datasets under their respective data-use agreements, and introduces no
new data collection and no human-subject experimentation.  The
estimator produces an uncertainty \emph{ranking}, not a calibrated
predictive standard deviation, and we say so in the body: it indicates
where a model's prediction is least constrained, and it is not
validated for, and should not be used for, clinical decision-making.
We would flag one risk specifically.  An uncertainty map is persuasive
because it looks quantitative, and Sec.~\ref{sec:tpt-indist} shows an
off-manifold probe returning a smooth, plausible map whose correlation
with the true error is $0.004$.  Any deployment therefore needs its
own validation against a reference, not the numbers reported here.

\section*{Use of AI tools}
Large language models were used throughout as a writing and coding
assistant: for drafting and \mbox{revising} prose, for implementing,
refactoring and debugging analysis code, for drafting proof
arguments subsequently checked by the authors, and as an adversarial
reader of both the theoretical claims and the experimental protocol.
The exact-denoiser validation of App.~\ref{app:curvtest} was
implemented with LLM assistance under author direction. 

\bibliographystyle{iclr2027_conference}
\bibliography{references}

\appendix

\section*{Appendix roadmap}
The appendices fall into four groups.
Appendices~\ref{app:recursion}--\ref{app:propagation-ranking} contain
the theory: the DDIM recursion and its delayed start, the proofs of
Prop.~\ref{prop:expansive} and Thm.~\ref{thm:curved}, the conditional
fibred interpretation, the non-optimal denoiser, the delayed-start
rule in closed form, and what a uniform injection term can and cannot
do to a spatial ranking.
Appendices~\ref{app:curvtest} and~\ref{app:toy} test those claims
where the truth is computable, against an exact denoiser and on
synthetic manifolds with known tangent spaces.
Appendices~\ref{app:budgets}--\ref{app:defs} give the experimental
design: the cost accounting, the forward-mode implementation obstacle,
the protocol for the head-free checkpoint, and the operational
definition of every mask and metric.
Appendices~\ref{app:conditions}--\ref{app:kablation} report the
per-condition tables behind Secs.~\ref{sec:exact}
and~\ref{sec:results}, the deployment comparison and its paired
tests, and the sensitivity of the estimates to $K$, to $t^\ast$ and to
the ensemble size $J$.
Each appendix is referenced at the single point in the main text
where its conclusion is used, and none needs to be read in full.

\medskip\noindent\emph{Theory and proofs.}\medskip

\section{The DDIM recursion, the delayed start, and the toy overshoot}
\label{app:recursion}

The body states the consequence; the result itself is this.

\begin{proposition}[Expansivity of additive affine covariance accumulation]
    \label{prop:expansive}
    Let $\mathcal{T}_s(V) = a_s^2 V + Q_s$ be any \emph{affine} recursion
    on the accumulated covariance, with $Q_s \succeq 0$ independent of
    $V$ collecting the variance head, the injected sampler noise and any
    additive correction.  Then
    $\mathcal{T}_s(V) - \mathcal{T}_s(W) = a_s^2 (V-W)$ for all $V, W$,
    so the linear part has spectral radius exactly $a_s^2$,
    \emph{independently of $Q_s$}: wherever $|a_s| > 1$, no variance
    head, no injected-noise schedule and no additive correction makes
    the recursion contractive, and the cumulative gain over a window is
    $\prod_s a_s^2$.
\end{proposition}
\noindent
It is stated for the affine additive class, which is the one
\citet{capellera2025u2diff} and \citet{oliveras2026mumins} use, and we
claim nothing about state-dependent or Jacobian-linearised
recursions.  The two equations the proof refers to are
\begin{equation}
    V_{s-\zeta} \;=\; \mathcal{T}_s(V_s) \;=\; a_s^2 V_s + Q_s,
    \qquad Q_s \succeq 0 \ \text{independent of } V_s,
    \label{eq:affine}
\end{equation}
where $Q_s = b_s^2 \hat\Sigma_s + \tilde\sigma_s^2 I + C_s$ collects the
variance head, the injected sampler noise (with its $\eta^2$ factor
inside $\tilde\sigma_s^2$) and any additive correction $C_s$, with
$b_s = c_t - a_s\sigma_t$; and
\begin{equation}
    \mathcal{T}_s(V) - \mathcal{T}_s(W) = a_s^2\,(V - W)
    \quad\text{for all } V, W.
    \label{eq:diffmap}
\end{equation}

We use the DDIM formulation of \citet{song2020ddim} with $T$ steps,
cumulative products $\bar\alpha_t$ and the cosine schedule of
\citet{nichol2021improved}: the forward process sends $x_0$ to
$x_t = \sqrt{\bar\alpha_t}x_0 + \sqrt{1-\bar\alpha_t}\,\xi$ with
$\xi \sim \mathcal{N}(0,I)$, and $\epsilon_\theta(x_t,t,c)$ is trained
to recover $\xi$.  For skip interval $\zeta$, the reverse update is
\begin{equation}
    x_{t-\zeta} = \sqrt{\bar\alpha_{t-\zeta}}\,\hat x_0(x_t)
        + c_t \epsilon_\theta(x_t, t) + \tilde\sigma_t\, \xi_t,
    \qquad
    \hat x_0 = \frac{x_t - \sqrt{1-\bar\alpha_t}\,\epsilon_\theta}{\sqrt{\bar\alpha_t}},
    \label{eq:ddim}
\end{equation}
with, following Eq.~16 of \citet{song2020ddim} and exactly as
implemented,
\[
    \tilde\sigma_t^2
    = \eta^2\,\frac{1-\bar\alpha_{t-\zeta}}{1-\bar\alpha_t}
      \Bigl(1 - \frac{\bar\alpha_t}{\bar\alpha_{t-\zeta}}\Bigr),
    \qquad
    c_t = \sqrt{1-\bar\alpha_{t-\zeta}-\tilde\sigma_t^2},
\]
where $\eta \in [0,1]$ is the DDIM/DDPM interpolation.  Two noise
scales must not be conflated: throughout this paper
$\sigma_t := \sqrt{1-\bar\alpha_t}$ is the \emph{forward marginal}
scale, which is what the geometric theory of
Sec.~\ref{sec:tpt} uses, while $\tilde\sigma_t$ is the
\emph{stochastic-injection} coefficient of the reverse transition.
The transition coefficient is
$a_s := \sqrt{\bar\alpha_{t-\zeta}/\bar\alpha_t}$, where $s$ indexes
the transition $t \to t-\zeta$.
Since $\bar\alpha$ increases as $t$ decreases, $|a_s| > 1$ over the
chain.  Throughout, $\bar\alpha_0 = 1$ and
$\bar\alpha_t \in (0,1)$ for $t > 0$; $\epsilon_\theta$ is
Borel-measurable and square-integrable, and the accumulated covariance
is initialised at $V_T = 0$ and required to stay positive
semi-definite (Assumption~\ref{ass:positivity}).

\noindent
Restating Prop.~\ref{prop:expansive} in full: $Q_s$ cancels from
Eq.~\eqref{eq:diffmap}, and the rest follows by iterating; the size of
the effect it licenses is measured below.  The result is stated for
the affine additive class, which is the one those implementations
use; state-dependent corrections and
Jacobian-linearised recursions change the multiplicative operator and
fall outside it.  App.~\ref{app:proof} gives the proof, the constraint
positivity places on $C_s$, and the scope conditions.
\begin{corollary}[Schedule-calibrated delayed start]
    \label{cor:mask}
    For the cosine schedule
    $a_s^2 = \bar\alpha_{t-\zeta}/\bar\alpha_t > 1$ at \emph{every}
    transition, so no non-expansive accumulation window exists.
    Delaying the start bounds the total inflation, which telescopes to
    \begin{equation}
        \textstyle\prod_{t \le \shat} a_s^2 = 1/\bar\alpha_{\shat},
        \qquad
        \skappa := \max\{\, t \;:\; 1/\bar\alpha_t \le \kappa \,\},
        \label{eq:shat}
    \end{equation}
    so a user-specified amplification budget $\kappa > 1$ determines
    the latest admissible start from the schedule alone, in
    $\mathcal{O}(T/\zeta)$ scalar operations with no denoiser call and
    no data.  This replaces an opaque timestep by an interpretable
    tolerance; it does not make $\kappa$ data-optimal, and delayed
    start is not the only conceivable algorithmic response
    (Rem.~\ref{rem:escapes}).  Derivation in App.~\ref{app:probe}.
\end{corollary}
\noindent
At $T = 300$, $\zeta = 5$, an illustrative one-per-cent budget gives
$\skappa = 16$, close to the $15$ selected on validation likelihood;
we do not claim Eq.~\eqref{eq:shat} \emph{predicts} that value, since
the budget was not fixed independently of it.  What the mask buys, and
what it cannot buy, is arithmetic on the schedule alone: unmasked
accumulation ends at $V \approx 1.5\times10^{3}$ against
$2.5\times10^{-3}$ at $\skappa$, while a mask placed late, at $t=64$,
still inflates by a further $40\times$ (App.~\ref{app:proof},
Fig.~\ref{fig:blowup}).  Equation~\eqref{eq:shat} configures the chain
in Sec.~\ref{sec:experiments} once the budget is chosen.
\paragraph{The distinction is material.}
On a diffusion model trained on a one-dimensional analytic manifold in
$\R^{1024}$, where the true per-pixel variance is measurable by
Monte-Carlo, propagating the variance analytically along the same
chains under the same independence assumption gives
\begin{equation}
    \frac{\Var_{\text{MC}}}{\Var_{\text{analytic}}} = 0.073,
    \label{eq:overshoot}
\end{equation}
an overshoot of $13.7\times$, with the delayed-start mask
\emph{already applied} at 60\% of that model's schedule.
This is not the cost of ignoring the accepted remedy; it is what
survives it, consistent with the schedule amplification quantified by
Cor.~\ref{cor:mask}.  Training protocol,
seeds, ground-truth construction and caveats in App.~\ref{app:toy}.

\section{Proof of Prop.~\ref{prop:expansive}}
\label{app:proof}

\begin{figure}[h]
\centering
\includegraphics[width=0.62\linewidth]{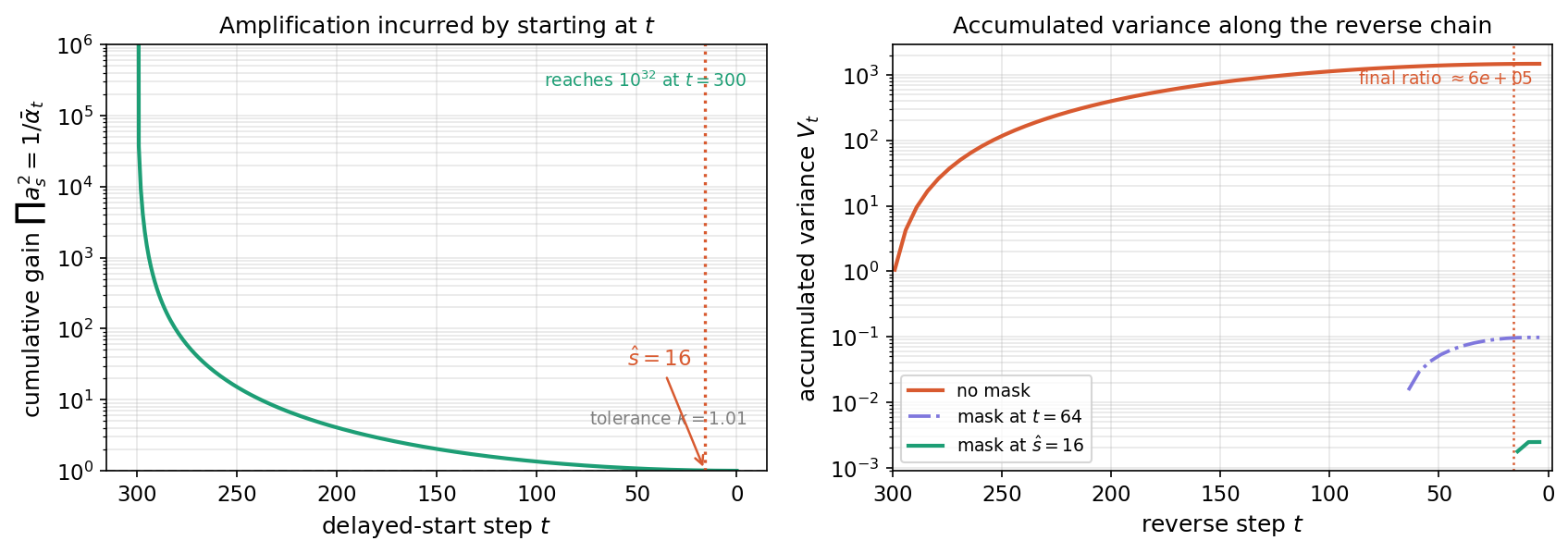}
\caption{Schedule-induced amplification in the additive variance
recursion.  \textbf{Left:} cumulative gain
$\prod a_s^2 = 1/\bar\alpha_t$ from starting accumulation at step $t$;
a one-per-cent budget gives $\skappa = 16$.  \textbf{Right:}
accumulated variance with no mask, a late mask ($t{=}64$), and the
schedule-calibrated mask.  The mask bounds inflation but cannot make
the recursion contractive.  Every coefficient comes from the schedule;
no model or data is involved.}
\label{fig:blowup}
\end{figure}

\begin{assumption}[Positivity]
    \label{ass:positivity}
    We set $\bar\alpha_0 = 1$ and $\bar\alpha_t \in (0, 1)$ for
    $t > 0$; $\epsilon_\theta$ is Borel-measurable and
    square-integrable.  The accumulated
    covariance is initialised at $V_T = 0$ and is required to remain
    positive semi-definite at every step, as any object claiming to be
    a variance must.
\end{assumption}

\begin{proof}
Write the recursion as $\mathcal{T}_s(V) = a_s^2 V + Q_s$ with
$Q_s \succeq 0$ independent of $V$.  For any $V, W$,
\[
    \mathcal{T}_s(V) - \mathcal{T}_s(W)
    = \bigl(a_s^2 V + Q_s\bigr) - \bigl(a_s^2 W + Q_s\bigr)
    = a_s^2 (V - W),
\]
since the forcing term cancels.  Hence, in any norm on symmetric
matrices,
$\|\mathcal{T}_s(V) - \mathcal{T}_s(W)\| = a_s^2 \|V - W\|$, so
$\mathcal{T}_s$ is a similarity whose Lipschitz constant equals
$a_s^2$, this being the operator norm and spectral radius of its linear
part $D\mathcal{T}_s = a_s^2 I$.  This is exact, not a bound,
and it does not involve $Q_s$: the choice of variance head, of
injected-noise schedule, or of any additive correction $C_s$ is
irrelevant to contractivity.

Composing over a window $[t_1, t_2]$ on which $|a_s| > 1$,
\[
    \mathcal{T}_{t_2} \circ \cdots \circ \mathcal{T}_{t_1}(V)
    - \mathcal{T}_{t_2} \circ \cdots \circ \mathcal{T}_{t_1}(W)
    = \Bigl(\prod_{s \in [t_1, t_2]} a_s^2\Bigr)(V - W),
\]
and the gain $\prod_s a_s^2 > 1$ telescopes to
$\bar\alpha_{t_2}/\bar\alpha_{t_1}$ by
Eq.~\eqref{eq:spectral}.  Applied to the trajectory itself with
$V_T = 0$, the same factor multiplies the accumulated covariance, so
$V$ inflates monotonically over the window.

One remark on the hypotheses.  Independence of $Q_s$ from $V_s$ is
what makes the result exact, and it is not a technicality: a
state-dependent correction can restore contractivity while preserving
positivity.  Taking $J(V_s) = -(a_s^2 - c)V_s$ with $0 < c < 1$ gives
$V_{s-\zeta} = c\,V_s + Q_s$, which is a contraction with factor $c$
and maps the positive-semidefinite cone into itself.  Such a
correction simply lies outside the class of Eq.~\eqref{eq:affine}, and
it requires a propagation rule that knows $a_s$ and rescales by it
rather than a closed-form accumulation.  Within the unmodified
additive affine rule, delaying the start (Cor.~\ref{cor:mask}) is the
remedy in use; Rem.~\ref{rem:escapes} lists what leaving that rule
would buy.
\end{proof}

\begin{remark}[What is and is not claimed]
\label{rem:scope}
The remark is deliberately narrower than ``no variance recursion can
work''.  It is exact for the affine family, which is the family those
implementations actually use, and it is exact \emph{because} $Q_s$ drops out of
Eq.~\eqref{eq:diffmap}: no amount of engineering in the forcing term
is relevant to contractivity.  Two things fall outside it.  A state-dependent correction such as $J(V_s) = -(a_s^2 - c)V_s$ is
not an additive forcing independent of $V_s$: it changes the linear
part of the recursion from $a_s^2 I$ to $cI$, and it does render the
map contractive.  It does so, however, by subtracting the schedule's
own gain, which presupposes knowing $a_s$ and is a rescaling rather
than a Jacobian.
\end{remark}

\subsection*{Proposition~\ref{prop:pt}, stated in full}

Fix $t$ and $x_0$, and let $r = \xi - \E[\xi \mid X_t]$.  The body
states the following three claims together; they are separated here
and proved in turn.

\begin{enumerate}[label=(\roman*), leftmargin=2em]
\item \emph{(Sampling.)}  For fixed $t$, the sample covariance of $K$
      residuals converges almost surely as $K \to \infty$ to the
      population covariance $C_t(x_0)$.
\item \emph{(Flat manifold, translation-invariant prior.)}  If
      $\M = x_0 + T$ is an affine subspace and the law of the tangent
      coordinate is translation-invariant along $T$, then the tangent
      component of $\xi$ is unidentifiable and the normal component is
      determined, so $C_t(x_0) = \PT$ for every $t$
      (the $\tau\to\infty$ limit of Eq.~\ref{eq:flat-exact}).
\item \emph{(Flat manifold, proper Gaussian prior.)}  If instead
      $u \sim \mathcal{N}(0, \tau^2 \PT)$, the observation carries
      partial information about $u$ and the projection shrinks:
      \begin{equation}
          C_t(x_0) = \Bigl(\tfrac{\bar\alpha_t \tau^2}
                                 {\bar\alpha_t \tau^2 + \sigma_t^2}\Bigr)^{\!2} \PT .
          \label{eq:flat-gauss}
      \end{equation}
\end{enumerate}

\noindent
The estimator used in this paper is the diagonal of the sample
covariance: a per-voxel scalar, not the full $d \times d$ matrix.
Part (ii) is an idealisation; no translation-invariant probability
measure exists on a non-compact subspace, and should be read as the
improper-prior limit of (iii) rather than as a statement about data.
To be explicit about what is and is not proved: $K \to \infty$ buys
only (i), convergence to a population covariance, which is not yet a
projector; (ii) and (iii) are exact but flat; departures from flatness
are the signature of curvature; and Thm.~\ref{thm:curved} gives the
curved extension to second order.

\subsection*{Proof of Proposition~\ref{prop:pt}(ii) and (iii)}

\begin{proof}
Split $\R^d = T \oplus N$ orthogonally and write $\xi = \xi_T + \xi_N$,
$x = x_0 + u$ with $u \in T$.  The forward marginal is
$X_t = \sqrt{\bar\alpha_t}(x_0 + u) + \sigma_t(\xi_T + \xi_N)$, whose
normal component
\[
    X_t^N = \sqrt{\bar\alpha_t}\, x_0^N + \sigma_t \xi_N
\]
does not involve $u$.  So $\xi_N$ is a measurable function of $X_t$ and
$\E[\xi_N \mid X_t] = \xi_N$: the normal part of the noise is fully
identified, and contributes nothing to the residual.  Everything turns
on the tangent part, where the observation is
\begin{equation}
    X_t^T = \sqrt{\bar\alpha_t}\, u + \sigma_t \xi_T ,
    \label{eq:tangent-obs}
\end{equation}
a single sum of the unknown position $u$ and the unknown noise
$\xi_T \sim \mathcal{N}(0, I_T)$.  How much of $\xi_T$ that sum reveals
depends entirely on the prior for $u$.

\paragraph{(ii) Translation-invariant prior.}
If the law of $u$ is invariant under translations along $T$, observing
Eq.~\eqref{eq:tangent-obs} says nothing about $\xi_T$: any value of
$\xi_T$ is consistent with a suitably shifted $u$, with the same prior
weight.  The posterior of $\xi_T$ is therefore its prior,
$\E[\xi_T \mid X_t] = 0$, and
\[
    r = \xi - \epsilon^\ast = \xi_T,
    \qquad \Cov(r) = \Cov(\xi_T) = \PT,
\]
independently of $t$.  The invariance is what does the work, and it is
also what makes this an idealisation: Lebesgue measure on $T$ is
translation-invariant but infinite, so no proper probability
distribution has this property on a non-compact subspace.  Read (ii)
as the limit of (iii), not as a claim about data.

\paragraph{(iii) Gaussian prior.}
Take $u \sim \mathcal{N}(0, \tau^2 \PT)$ independent of $\xi$.  Then
$(u, \xi_T)$ is jointly Gaussian with $X_t^T$, and the standard
conditioning formula applied to Eq.~\eqref{eq:tangent-obs} gives
\[
    \E[\xi_T \mid X_t^T]
    = \frac{\sigma_t}{\bar\alpha_t \tau^2 + \sigma_t^2}\, X_t^T .
\]
Evaluating the probe at $x_0$, that is at $u = 0$, gives
$X_t^T = \sigma_t \xi_T$ and hence
\[
    r_T = \xi_T - \E[\xi_T \mid X_t^T]
        = \frac{\bar\alpha_t \tau^2}{\bar\alpha_t \tau^2 + \sigma_t^2}\, \xi_T ,
\]
so $\Cov(r) = \bigl(\bar\alpha_t \tau^2 /
(\bar\alpha_t \tau^2 + \sigma_t^2)\bigr)^2 \PT$, which is
Eq.~\eqref{eq:flat-gauss}.  The eigenspace is $T$ for every $\tau$ and
every $t$; the common eigenvalue is one only in the limit
$\tau \to \infty$ or $\sigma_t \to 0$, recovering (ii).
\end{proof}


\section{Proof of Theorem~\ref{thm:curved}}

\begin{theorem}[Curved case]
\label{thm:curved}
Assume (H1)--(H4) of App.~\ref{app:curved}.  Let $x^\ast \in K$,
$x^\ast = m^\ast + n^\ast$ with $m^\ast = \pi_\M(x^\ast)$, and write
$n^\ast = h\,\nu$ with $\nu$ a unit normal.  Let $W_\nu :=
\mathrm{D}\nu|_{T_{m^\ast}\M}$ be the Weingarten map, with eigenvalues
the principal curvatures $\kappa_1,\dots,\kappa_d$ in the direction
$\nu$.  Then the nearest-point projection satisfies, \emph{exactly in
$h$},
\begin{equation}
    \mathrm{D}\pi_\M(x^\ast) \;=\; (I + h\,W_\nu)^{-1}\,\PT ,
    \label{eq:dpi}
\end{equation}
and in the regime $\sigma_t \to 0$ with $h = O(\sigma_t)$,
\begin{equation}
    C_t(x^\ast)
    \;=\; \mathrm{D}\pi_\M(x^\ast)\,\mathrm{D}\pi_\M(x^\ast)^{\!\top}
        \;+\; \mathcal{O}(\sigma_t^2)
    \;=\; (I + h\,W_\nu)^{-2}\,\PT \;+\; \mathcal{O}(\sigma_t^2),
    \label{eq:curved}
\end{equation}
so the eigenvalues of the leading term of $C_t$ restricted to
$T_{m^\ast}\M$ are
\begin{equation}
    \frac{1}{(1 + h\,\kappa_i)^2}, \qquad i = 1,\dots,d.
    \label{eq:kappa}
\end{equation}
Proof in App.~\ref{app:curved}.
\end{theorem}

\label{app:curved}

\paragraph{Hypotheses.}
We list them because the theorem is exact only under all four, and
because the fourth is the one that fails in practice.

\begin{enumerate}[label=(H\arabic*), leftmargin=2.6em]
\item \emph{Manifold.} $\M \subset \R^N$ is a compact embedded $C^4$
      submanifold of dimension $d$, without boundary, with positive reach
      $\tau_{\M} > 0$.  (Without boundary is what the Laplace expansion
      of Step~2 needs: the Gaussian integrals over the tangent chart
      are complete, the odd moments vanish by parity, and the remainder
      is $O(\sigma^3)$ rather than the $O(\sigma)$ a truncated domain
      would leave.)  ($C^4$, rather than $C^3$, because the proof
      Taylor-expands the projection to third order; $C^{3,1}$ would
      also do.)  We \emph{assume} this rather than establish it.
      Intensity normalisation and percentile clipping put the data in a
      bounded box, which makes boundedness plausible, but boundedness
      is not compactness without boundary, and the manifold a network
      represents is not ours to characterise.  Positive reach is the load-bearing
      part: it is what makes the projection unique and the expansion of
      Step 2 legitimate.
\item \emph{Density.} The data measure has a $C^3$ density $\varphi$ with
      respect to the volume measure of $\M$, with locally bounded
      derivatives, bounded away from zero on a neighbourhood of
      $m^\ast$.  The local statement is the one we need and the one
      that is available: were $\M$ non-compact, a probability density
      could not be globally bounded below.
\item \emph{Tubular neighbourhood.} $x^\ast$ lies in a compact set $K$
      strictly inside the tube of radius $\varrho < \tau_\M$, so the
      projection $m^\ast = \pi_\M(x^\ast)$ is unique,
      $n^\ast = x^\ast - m^\ast \in N_{m^\ast}\M$, and
      $\delta := \mathrm{dist}(K, \partial\,\mathrm{Tub}_\varrho) > 0$.
\item \emph{Optimal denoiser.} $\epsilon_\theta(\cdot, t) =
      \E[\xi \mid X_t = \cdot]$.  App.~\ref{app:modelerror} treats
      the failure of this one, which is the only one we cannot arrange.
\end{enumerate}

\paragraph{Step 1: the residual is a posterior-mean displacement.}
Tweedie's identity applied to $Y = \sqrt{\bar\alpha_t}\,x^\ast +
\sigma_t \xi$, rescaled so that the noised variable reads
$Y/\sqrt{\bar\alpha_t} = x^\ast + \sigma\xi$ with
$\sigma = \sigma_t/\sqrt{\bar\alpha_t}$, gives
$\epsilon^\ast(Y) = (Y/\sqrt{\bar\alpha_t} - \E[X_0 \mid Y])/\sigma$
and hence, exactly,
\begin{equation}
    r \;=\; \xi - \epsilon^\ast(Y)
      \;=\; \frac{\E[X_0 \mid Y] - x^\ast}{\sigma}.
    \label{eq:rpost}
\end{equation}
No expansion has been used: the residual \emph{is} the displacement of
the posterior mean from the probe, in units of $\sigma$.  Writing the
proof this way removes the need for an asymptotic expansion of the
score, which is where an order was lost in an earlier draft of this
appendix.

\paragraph{Step 2: a uniform second-order expansion of the posterior mean.}
\begin{lemma}[Uniform posterior-mean expansion]
\label{lem:postproj}
Let $K^{+} := \{y : \mathrm{dist}(y, K) \le \delta/2\} \subset
\mathrm{Tub}_\varrho(\M)$, so that perturbations
$y = x^\ast + \sigma\xi$ of any $x^\ast \in K$ with
$\|\sigma\xi\| \le \delta/2$ remain inside the domain of
uniformity, the enlargement is what licenses applying the lemma at
the perturbed point, not only at $x^\ast$.  Under (H1)--(H4), with the
regularity of $\varphi$ holding uniformly on a neighbourhood of
$\pi_\M(K^{+})$, there is a locally $C^1$ vector field
$b : K^{+} \to \R^N$ such that, uniformly for $y \in K^{+}$,
\[
    m_\sigma(y) := \E[X_0 \mid Y = y]
    \;=\; \pi_\M(y) \;+\; \sigma^2 b(y) \;+\; R_\sigma(y),
    \qquad
    \sup_{y \in K^{+}}\|R_\sigma(y)\| \;\le\; C_{K^{+}}\,\sigma^3 .
\]
\end{lemma}
\begin{proof}
By Tweedie's identity, $m_\sigma(y) = y + \sigma^2\nabla\log q_\sigma(y)$
with $q_\sigma = \rho\,\mathrm{vol}_\M \ast \mathcal{N}(0,\sigma^2 I)$,
so it is enough to prove
$\nabla\log q_\sigma(y) = -(y - \pi_\M(y))/\sigma^2 + b(y) + O(\sigma)$
uniformly on $K^{+}$.  Fix $y \in K^{+}$, set $p = \pi_\M(y)$ and
$n = y - p = h\nu$, and parametrise $\M$ near $p$ by
$m(u) = p + u + \tfrac12\II_p(u,u) + O(\|u\|^3)$, $u \in T_p\M$.  With
$u = \sigma v$,
\[
    \|y - m(\sigma v)\|^2
    \;=\; \|n\|^2 \;+\; \sigma^2\,\langle v,\, G_y\, v\rangle
    \;+\; O(\sigma^3\|v\|^3),
    \qquad
    G_y := I + h\,W_\nu,
\]
the tangent Hessian of the squared-distance function.  The point that
must not be glossed: $h$ is \emph{fixed} on $K^{+}$, not $O(\sigma)$,
so the $h$-dependent quadratic term belongs inside the principal
Gaussian, not in the remainder.  Positive reach and the strict
inclusion $K^{+} \Subset \mathrm{Tub}_\varrho(\M)$ give uniform bounds
$0 < c_- I \preceq G_y \preceq c_+ I$ on $K^{+}$ (since
$|h\kappa_i| < 1$ there), so Laplace's method applies with the
anisotropic Gaussian $e^{-\langle v, G_y v\rangle/2}$ and yields,
uniformly in $C^1(K^{+})$-norm,
\[
    q_\sigma(y)
    \;=\; C_\sigma\, e^{-\|n\|^2/2\sigma^2}\, a_0(y)\,
          \bigl(1 + O_{C^1(K^{+})}(\sigma)\bigr),
    \qquad
    a_0(y) \;=\; \frac{\varphi(p)}{\sqrt{\det G_y}},
\]
with $C_\sigma$ independent of $y$; the $C^1$-uniformity comes from
$\M \in C^4$, $\varphi \in C^3$, and Gaussian domination by the lower
bound on $G_y$.  Using
$\nabla \tfrac12\,\mathrm{dist}(y,\M)^2 = y - \pi_\M(y)$ and
differentiating the logarithm,
\[
    \nabla\log q_\sigma(y)
    = -\frac{y - \pi_\M(y)}{\sigma^2} + \nabla\log a_0(y) + O(\sigma),
\]
so the claim holds with $b(y) := \nabla\log a_0(y)$, which is $C^1$ on
$K^{+}$; multiplying by $\sigma^2$ gives the statement.  The far part of the
manifold integral is harmless: for $m$ with $d_\M(p,m) \ge r_0$,
positive reach and the strict inclusion
$K^{+} \Subset \mathrm{Tub}_\varrho(\M)$ give
$\|y-m\|^2 \ge \|n\|^2 + c\,r_0^2$ with $c > 0$ uniform on $K^{+}$, so
its contribution to $q_\sigma$ and $\nabla q_\sigma$ is
$O(e^{-c r_0^2/2\sigma^2})$, beyond every polynomial order.

The formula for $b$ is checkable, and we checked it.  On the unit
circle with uniform density, $G_y = 1 + h$ and
$b(y) = -\tfrac{1}{2(1+h)}\,\nu$; at $h = 0.15$ this predicts
$-0.43478$.  Quadrature at the same point, computed \emph{before}
this form of the lemma was written down, gives
$(m_\sigma(y) - \pi_\M(y))/\sigma^2 \to -0.4348\,\nu$ with vanishing
tangential component, agreeing to four decimals, and the residual
after subtracting $\sigma^2 b$ scales as $\sigma^{4.5}$, inside the
$O(\sigma^3)$ asserted.
\end{proof}

\noindent
The strengthening matters, and it is worth recording why a bare
$O(\sigma^2)$ would not do: dividing by $\sigma$ in
Eq.~\eqref{eq:rpost} turns an unstructured $O(\sigma^2)$ remainder
into an unstructured \emph{random} $O(\sigma)$ term in $r$, whose
cross-covariance with the leading term need not vanish and could
contaminate the covariance at $O(\sigma)$.  With the remainder pinned
at $O(\sigma^3)$ and the deterministic $\sigma^2 b$ split off, every
surviving cross-term is $O(\sigma^2)$.

\paragraph{Interlude: applying the lemma at the random point.}
Lemma~\ref{lem:postproj} is a deterministic statement on $K^{+}$; the
proof needs it at the random point $y = x^\ast + \sigma\xi$.  Define
$G_\sigma := \{\|\sigma\xi\| \le \delta/2\}$.  On $G_\sigma$,
$y \in K^{+}$ and the lemma applies verbatim.  On $G_\sigma^c$,
$\Pr(G_\sigma^c) \le C e^{-c\delta^2/\sigma^2}$ and the posterior mean
is bounded by compactness of the tube, so the complement contributes
beyond every polynomial order to the $L^2$ remainder of
Eq.~\eqref{eq:rexp} below.

\paragraph{Step 3: expand the projection, not the score.}
For $x^\ast \in K$, $\pi_\M$ is $C^3$ and
\[
    \pi_\M(x^\ast + \sigma\xi)
    = m^\ast + \sigma\,A\,\xi
      + \frac{\sigma^2}{2}\,B[\xi,\xi]
      + \mathcal{O}(\sigma^3\|\xi\|^3),
    \qquad
    A := \mathrm{D}\pi_\M(x^\ast),\; B := \mathrm{D}^2\pi_\M(x^\ast),
\]
with $A = (I + hW_\nu)^{-1}\PT$ from Eq.~\eqref{eq:dpi}.  Substituting
into Eq.~\eqref{eq:rpost} with Lem.~\ref{lem:postproj}, and using
$b(x^\ast + \sigma\xi) = b(x^\ast) + O(\sigma\|\xi\|)$ from the local
$C^1$ regularity of $b$,
\begin{equation}
    r \;=\; -\frac{n^\ast}{\sigma}
        \;+\; A\,\xi
        \;+\; \sigma\Bigl(\tfrac12 B[\xi,\xi] + b(x^\ast)\Bigr)
        \;+\; \mathcal{O}_{L^2}(\sigma^2).
    \label{eq:rexp}
\end{equation}
That the term linear in $\xi$ is the perturbation of the projection is
now a consequence of the lemma and the Taylor expansion, not an
assertion.

\paragraph{Step 4: take the covariance.}
In Eq.~\eqref{eq:rexp}, $-n^\ast/\sigma$ and $\sigma b(x^\ast)$ are
deterministic and drop out of the covariance.  $A\xi$ is odd in $\xi$
and $B[\xi,\xi]$ is even, so their cross-covariance is an odd moment
of a centred Gaussian and vanishes; the quadratic term contributes
$\mathcal{O}(\sigma^2)$ in its own right, and the $L^2$ remainder only
produces cross-terms of $\mathcal{O}(\sigma^2)$ by Cauchy--Schwarz.
Hence, in operator norm,
\[
    \Cov(r) = A\,\Cov(\xi)\,A^{\!\top} + \mathcal{O}(\sigma^2)
            = (I + hW_\nu)^{-2}\,\PT + \mathcal{O}(\sigma^2),
\]
whose tangent eigenvalues are $(1+h\kappa_i)^{-2}$:
Eq.~\eqref{eq:kappa}.  Finally, the proof works in the rescaled noise
$\sigma = \sigma_t/\sqrt{\bar\alpha_t}$; since $\bar\alpha_t \to 1$
in the small-noise regime, $\sigma^2 = \sigma_t^2/\bar\alpha_t =
O(\sigma_t^2)$ and the remainder has the order stated in the
theorem. \qed

\paragraph{A check on the bias, which the covariance does not see.}
Taking expectations in Eq.~\eqref{eq:rexp} gives
\[
    \E[r] = -\frac{n^\ast}{\sigma}
    + \sigma\Bigl(\tfrac12\,\mathrm{tr}\,B + b(x^\ast)\Bigr)
    + \mathcal{O}(\sigma^2),
\]
so the order-$\sigma$ coefficient is \emph{not} curvature alone: it
also carries $b(x^\ast)$, which contains the density gradient.  On the
unit circle with uniform density the density term vanishes and the
coefficient reduces to its curvature part; quadrature there gives
$\|\E[r]\| = 1.00\,\sigma$ over $\sigma \in [0.05, 0.30]$ with
fitted exponent $1.03$, against $|H| = 1$.  On a manifold with
non-uniform density, this identification would be wrong, and the bias
would mix curvature with density drift.  We record the check because
an earlier version of this proof carried the curvature term at the
wrong power of $\sigma$, and the bias is the quantity that detects
such an error; the covariance does not, since the bias drops out of
it.

\subsection*{Remarks on the leading term and the remainder}

The body states four caveats in one paragraph; here they are in full.

\textbf{The projection differential is the object, not its expansion.}
Writing the first-order truncation $(I - hW_\nu)\PT$, from
$(I+hW_\nu)^{-1} = I - hW_\nu + O(h^2)$, instead of
Eq.~\eqref{eq:dpi} loses an $O(h^2)$ term: on the unit
circle at $h = 0.3$ they give $0.592$ and $0.490$.  Since the
statement is available exactly, we use the exact one; it is what
App.~\ref{app:curvtest} measures.

\textbf{Eigenvalues, not eigenspace, at leading order only.}
The leading term of Eq.~\eqref{eq:curved} is a congruence of $\PT$,
so to that order the range is unchanged and only the eigenvalues move.
The $\mathcal{O}(\sigma_t^2)$ remainder is not constrained to respect
that, and may both leak into the normal space and rotate the
eigenvectors.  It is tempting to conclude that
a rank statistic cannot see the correction, and that would be wrong:
the per-voxel value is $\sigma_i^2 = \sum_j \lambda_j U_{ij}^2$, and
reweighting the $\lambda_j$ \emph{anisotropically} can reorder voxels
even with $U$ fixed.  Spearman is invariant only under a common
positive scalar $\Lambda \mapsto c\Lambda$.  What we can say is
conditional: if the correction acts approximately as a common scalar
across the dominant tangent directions, which holds when the
$\kappa_i$ are comparable, or when one direction dominates
$U_{ij}^2$, the per-voxel ordering is preserved.  The observed
flatness of the $t^\ast$ sweep is consistent with that condition; it
does not establish it.

\textbf{The tangent density lives inside the remainder, and dominates
it here.}  Eq.~\eqref{eq:curved} carries no factor $\gamma_t$, which
looks at first like a conflict with Eq.~\eqref{eq:flat-gauss}.  It is
not: for a proper tangent prior of scale $\tau$,
\begin{equation}
    1 - \gamma_t^2
    \;=\; 1 - \Bigl(\tfrac{\bar\alpha_t\tau^2}{\bar\alpha_t\tau^2+\sigma_t^2}\Bigr)^{\!2}
    \;=\; \frac{2\sigma_t^2}{\bar\alpha_t \tau^2} + \mathcal{O}(\sigma_t^4),
    \label{eq:densityremainder}
\end{equation}
so the density shrinkage \emph{is} an $\mathcal{O}(\sigma_t^2)$ term
and is already inside the remainder of Eq.~\eqref{eq:curved}, with
constant $2/(\bar\alpha_t\tau^2)$.  The two statements agree; the
paper would be wrong only if it attributed the whole remainder to
curvature.  It does not, and the ordering matters: at
$\sigma_{t^\ast} = 0.32$ with a tangent extent $\tau \sim 1$,
Eq.~\eqref{eq:densityremainder} is $0.19$, larger than the curvature
contribution we are able to name.  A measured eigenvalue below one is
therefore evidence of a finite tangent prior before it is evidence of
curvature.

\textbf{The remainder is what shrinks.}  Under the cosine schedule at
$T = 300$ our probe sits at $\sigma_{t^\ast} = 0.32$, which is not an
asymptotic regime.  App.~\ref{app:curvtest} tests
Eq.~\eqref{eq:kappa} against an exactly computed denoiser on a
manifold whose projection is closed form: the predicted eigenvalue is
tracked to one per cent over a factor of $2.3$, and the disagreement
scales as $\sigma^2$ as claimed.  With constants of order one, the first-order
statement $C_t = \PT + \mathcal{O}(\sigma_t)$ leaves about a third of
the leading term unaccounted for, and Eq.~\eqref{eq:curved} leaves
about a tenth.  We report the second-order form because it is the one
that is quantitatively honest where we actually operate, not because
$\sigma_t$ is small.

\paragraph{What the theorem does not say.}
Two readings are available, and only one is correct.  It does
\emph{not} say that a curved manifold gives non-unit eigenvalues:
$\shape_{n^\ast} := h\,W_\nu$ is linear in $n^\ast$, so an on-manifold probe gives
$\PT$ to this order however curved $\M$ is.  It says that the
eigenvalues report curvature \emph{as seen from the offset}.  This is
why the on-manifold experiments of App.~\ref{app:toy} do not test
Eq.~\eqref{eq:kappa}, and we do not present them as though they did.


\section{The fibred residual, and what measurement injection does}
\label{app:fibred}

\paragraph{Fibration.}
Add to (H1)--(H4) of App.~\ref{app:curved}:
\begin{enumerate}[label=(H\arabic*), leftmargin=2.6em, start=5]
\item \emph{Fibration.} $\pi_0 : \M \to \M_0$, $\pi_0(u,v) = u$, is a
      $C^{3}$ submersion onto the manifold $\M_0 \subset \R^{N'}$ of
      plausible baselines, $\dim \M_0 = d_0$.  For each $u$ the fibre
      $\mathcal{F}_u = \pi_0^{-1}(u)$ is a smooth submanifold of
      dimension $d_F = d - d_0$.
\end{enumerate}
The tangent space then splits orthogonally,
$T\M = T^{\mathrm{h}} \oplus T^{\mathrm{v}}$ with
$T^{\mathrm{v}} = \ker \mathrm{d}\pi_0 = \{0\} \times T\mathcal{F}$,
and correspondingly $\PT = \PT^{\mathrm{h}} + \PT^{\mathrm{v}}$ with
$\PT^{\mathrm{h}}\PT^{\mathrm{v}} = 0$.  The normal splitting does \emph{not} follow from (H5) alone.  Under
the additional compatible local-product assumption of
Prop.~\ref{thm:fibred}, the ambient normal bundle admits the
corresponding splitting; under that same assumption we write
$n^\ast = n^\ast_{\mathrm{h}} + n^\ast_{\mathrm{v}}$ with
$\|n^\ast_{\mathrm{h}}\| = \mathrm{dist}(u^\ast, \M_0)$ and
$\|n^\ast_{\mathrm{v}}\| = \mathrm{dist}(v^\ast, \mathcal{F}_{u^\ast})
+ O(\|n^\ast_{\mathrm{h}}\|)$.

\begin{proposition}[Fibred structure of the leading covariance]
\label{thm:fibred}
Assume (H1)--(H5) and that, near $m^\ast$, the manifold, the ambient
normal bundle and the data density admit a compatible local product
decomposition, with horizontal and vertical projections $P_{\mathrm{h}},
P_{\mathrm{v}}$.  Then the leading covariance of
Thm.~\ref{thm:curved},
$\mathrm{D}\pi_\M(x^\ast)\,\mathrm{D}\pi_\M(x^\ast)^{\!\top}$, is block
diagonal with respect to the horizontal--vertical splitting, and the
cross-covariance of the two residual components is
$\mathrm{Cov}(r_{\mathrm{h}}, r_{\mathrm{v}}) = O(\sigma_t^2)$.
Exact vanishing would require an exact product factorisation, and is
not claimed for the implemented estimator, whose covariance carries
the contamination term $E_{\mathrm{fib}}$ below.
\end{proposition}

\begin{proof}
Under the product assumption $P_{\mathrm{h}}$ and $P_{\mathrm{v}}$
commute with the Weingarten map, so
$\mathrm{D}\pi_\M = (I+hW_\nu)^{-1}\PT$ splits blockwise and the
congruence $\mathrm{D}\pi_\M\,\mathrm{D}\pi_\M^{\!\top}$ is block
diagonal.  For the cross term, apply the expansion
Eq.~\eqref{eq:rexp} projected by $P_{\mathrm{h}}$ and
$P_{\mathrm{v}}$: after centring, the linear parts are uncorrelated by
the block structure, the linear--quadratic cross terms vanish by
Gaussian parity, and what survives is $O(\sigma_t^2)$.  We state the
result for the \emph{covariance}, not the raw cross-moment
$\E[r_{\mathrm{h}} r_{\mathrm{v}}^\top]$: both components carry
deterministic means ($-n^\ast_\bullet/\sigma$ and $\sigma b_\bullet$),
so the uncentred moment need not be small even when the covariance
is. \qed
\end{proof}

\begin{remark}[Why the extra assumption is not cosmetic]
A submersion alone gives the \emph{tangent} splitting
$T\M = T^{\mathrm{h}} \oplus T^{\mathrm{v}}$ by taking
$T^{\mathrm{h}} = (T^{\mathrm{v}})^\perp \cap T\M$.  It does not by
itself give a compatible splitting of the \emph{ambient normal} space,
nor the distance identities $\|n^\ast_{\mathrm{h}}\| =
\mathrm{dist}(u^\ast, \M_0)$ and $\|n^\ast_{\mathrm{v}}\| =
\mathrm{dist}(v^\ast, \mathcal{F}_{u^\ast}) +
O(\|n^\ast_{\mathrm{h}}\|)$: those need the fibres to sit in a
compatible way as the base point moves, which is what a local product
structure, or a metric connection with vanishing O'Neill tensor,
supplies.  We assume it and do not verify it on our data.
\end{remark}

\begin{remark}[What measurement injection does, and does not, give]
Re-projecting the baseline channel onto the noised observation at
every reverse step removes the horizontal component of the residual
\emph{provided the denoiser is conditioned on the clean baseline
$u^\ast$}.  If it observes only $u_t = \sqrt{\bar\alpha_t}u^\ast +
\sigma_t \xi_u$, it cannot separate $u^\ast$ from $\xi_u$ exactly, and
what survives is a horizontal residual of the same order as the
model error.  Our sampler injects the measurement but the network is
not given $u^\ast$ as an explicit conditioning input, so the
cancellation is approximate.  We therefore read
\[
    r_{\mathrm{obs}} \approx r_{\mathrm{v}} + \Delta_\theta^{u}
\]
as an interpretation supported by the construction, not as an
identity.
\end{remark}

\paragraph{Three consequences, and one caveat.}
\emph{Under the exact local-product, clean-conditioning
idealisation, the vertical component is fibre uncertainty}: not ``how
uncertain is this image'' but ``given this patient, how plausible is
this progression''.  The implemented estimator only approximates that
component and may carry horizontal contamination.

\emph{The rank counts progression freedom, for the ideal model.}
Under the product structure and clean-baseline conditioning, the
covariance of the observable residual has rank
$d_F = \dim \mathcal{F}_{u^\ast}$: the degrees of freedom of
progression available to this patient.  For the implemented estimator
$\Cov(r_{\mathrm{obs}}) = \Cov(r_{\mathrm{v}}) + E_{\mathrm{fib}}$
with $E_{\mathrm{fib}}$ collecting the horizontal-contamination
terms $\Delta_\theta^{u}$ of the previous paragraph, and we make no
exact-rank claim.

\emph{Zero-placeholder probing, in this language.}  Setting
$v_0 = 0$ creates primarily an off-\emph{fibre} displacement,
$n^\ast_{\mathrm{v}} \ne 0$; without an exact product structure it can
also excite the horizontal component.  The collapse from
$\bar\rho = 0.65$ to $0.004$ in Sec.~\ref{sec:tpt-indist} is
consistent with this reading, though it does not by itself separate
the two contributions.

\emph{Caveat.}  The measurement-injection remark above presumes
$u^\ast \in \M_0$,
that is, a baseline typical of the training population.  For a patient
outside it the injection does not cancel the horizontal drift and the
vertical component is contaminated.  We have not measured this, and it
is the one place where the fibred reading could mislead: an atypical
baseline would produce a fibre-uncertainty map that is not one.


\section{When the denoiser is not optimal}
\label{app:modelerror}

Hypothesis (H4) of App.~\ref{app:curved} asks for the population-optimal
denoiser.  No trained network satisfies it, and this is the failure that
matters in practice, so it is worth writing down exactly rather than
absorbing into a remainder.  Define the \emph{model error}
\begin{equation}
    \Delta_\theta(X_t, t) \;:=\; \epsilon_\theta(X_t, t) - \E[\xi \mid X_t].
    \label{eq:delta}
\end{equation}
Since the residual is linear in the predictor,
$r_{\text{real}} = r_{\text{ideal}} - \Delta_\theta$ pointwise, and
therefore
\begin{equation}
    \Cov(r_{\text{real}})
    \;=\; \Cov(r_{\text{ideal}})
        + \Cov(\Delta_\theta)
        - 2\,\mathrm{sym}\,\Cov(r_{\text{ideal}}, \Delta_\theta).
    \label{eq:covreal}
\end{equation}

Three remarks, in decreasing order of comfort.

\textbf{The global bias is estimable; the local one is not identified.}
Averaged over a validation distribution matching the joint data law,
the mean residual estimates the global mean model-error bias
$-\E[\Delta_\theta]$, so
\begin{equation}
    \widehat{\Delta_\theta}
    \;=\; -\,\E_{(x_0,\xi) \sim \mathcal{D}_{\text{val}}}\bigl[r\bigr]
    \label{eq:deltahat}
\end{equation}
can be subtracted from the observed bias to remove a global offset;
it does not identify the input-dependent local error
$\E[\Delta_\theta \mid x_0]$.  Nothing here requires a posterior over
weights, which is what Laplace- and ensemble-based corrections need.

\textbf{The cross term is not.}  Eq.~\eqref{eq:covreal} has a term
coupling the ideal residual to the model error, and estimating it needs
the ideal residual, which is the unknown.  We do not estimate it and we
do not assume it is zero.  It is the honest content of the statement
that $\sigma_{\text{TW}}$ is an ordering rather than a calibrated
covariance.

\textbf{This is what the synthetic manifolds measure.}  On the toy
manifolds of App.~\ref{app:toy} the geometry is closed-form, so
departures of the measured spectrum from the leading projector combine
finite-$K$ sampling, the finite-noise geometric and density remainder,
and denoiser error $\Delta_\theta$; the known geometry lets these be
bounded or measured separately, but they do not vanish identically.
That is the
right way to read the second tangent eigenvalue reported there: at
$K = 2000$ and $d = 2$ in codimension $1022$, finite sampling alone
places the two non-zero eigenvalues at $1.03 \pm 0.03$ and
$0.97 \pm 0.03$, so a measured $0.839$ is roughly five standard
deviations low and is a real effect, while a measured $1.052$ is not
distinguishable from sampling and needs no explanation at all.


\section{The delayed-start mask in closed form}
\label{app:probe}

Corollary~\ref{cor:mask} states that, \emph{given} a user-selected
amplification budget $\kappa$, the schedule uniquely determines
$\skappa$.  This appendix derives that map, and records the two facts
the body quotes: that no stable window exists, and that the cumulative
inflation telescopes.  It replaces an opaque timestep by an
interpretable tolerance; it does not eliminate the user's choice.

\paragraph{No stable window exists.}
Define the linear map on the accumulated covariance
$V \mapsto \mathcal{T}_t[V] := a_s^2 V + b_s^2 \hat\Sigma_s +
\tilde\sigma_s^2 I$ implied by any recursion of the form of
Prop.~\ref{prop:expansive}.  By Eq.~\eqref{eq:diffmap} its
derivative is $a_s^2 I$, so
\begin{equation}
    \rho\bigl(D\mathcal{T}_t\bigr) = a_s^2
    = \frac{\bar\alpha_{t-\zeta}}{\bar\alpha_t}.
    \label{eq:spectral}
\end{equation}
It is tempting to place the mask where $\rho \le 1$.  That set is
empty: $\bar\alpha$ increases monotonically as $t$ decreases, so
Eq.~\eqref{eq:spectral} exceeds $1$ at every transition of the reverse
chain.  For the cosine schedule at $T = 300$, $\zeta = 5$ the
homogeneous part is expansive at $100\%$ of steps.  This is
Prop.~\ref{prop:expansive} restated concretely: a mask cannot
make the recursion contractive.  What it can do is bound how much the
recursion inflates.

\paragraph{The cumulative gain telescopes.}
Starting accumulation at step $\shat$, the total homogeneous
amplification down to $t = 0$ is
\begin{equation}
    \prod_{t \le \shat} a_s^2
    = \prod_{t \le \shat} \frac{\bar\alpha_{t-\zeta}}{\bar\alpha_t}
    = \frac{\bar\alpha_0}{\bar\alpha_{\shat}}
    = \frac{1}{\bar\alpha_{\shat}},
    \label{eq:gain}
\end{equation}
because every intermediate factor cancels.  The delayed-start step is
therefore a statement about how much inflation one is willing to
tolerate, which is Eq.~\eqref{eq:shat}; the tolerance itself remains
the user's choice.

\paragraph{Sensitivity to the tolerance.}
At $T = 300$, $\zeta = 5$: $\kappa = 1.005$ gives $\skappa = 11$,
$\kappa = 1.01$ gives $16$, and $\kappa = 1.02$ gives $24$.  The value
reported in variance-propagating implementations, $\shat = 15$,
is close to the one-per-cent inflation point of this schedule.  We do
not claim Eq.~\eqref{eq:shat} predicted it independently: the
one-per-cent budget was not fixed before seeing the published value.
What the map offers is $\mathcal{O}(T/\zeta)$ scalar operations, no
autograd, no denoiser call, and no data.

\begin{remark}[Escapes]
\label{rem:escapes}
Prop.~\ref{prop:expansive} rules out one mechanism, not
uncertainty estimation in diffusion models.  Four escapes remain.
Change the schedule so $|a_s| \le 1$ throughout, which means leaving
the standard DDIM/DDPM family and retraining.  Leave the affine class
for a state-dependent or fully linearised propagation, at the cost of
the Jacobian products the closed-form recursion existed to avoid.
Track higher moments, which reintroduces the sampling cost.  Or stop
recursing: with no accumulated covariance there is no operator to
misbehave.  We take the last (Sec.~\ref{sec:tpt}).
\end{remark}


\medskip\noindent\emph{Exact and synthetic validation.}\medskip

\section{Magnitude and ranking in analytical variance propagation}
\label{app:propagation-ranking}

Consider the zero-Jacobian analytical baseline in its stochastic form
\citep{oliveras2026mumins}: along a DDIM chain with $\eta = 1$, its
output separates into a spatially varying contribution from the
uncertainty head and a uniform contribution from the injected sampler
noise.  The deterministic variant of \citet{capellera2025u2diff}
($\eta = 0$) is the special case $\nu_k = 0$ below, in which the second
contribution is absent.  This decomposition clarifies which parts of
the map can affect the spatial ranking.

Index the actual active propagation transitions by $k=0,\ldots,n-1$,
in sampling order, and write
\[
V_{k+1}=a_k^2V_k+b_k^2D_k+\nu_k^2I,\qquad V_0=0.
\]
Here $a_k,b_k,\nu_k$ are spatially uniform schedule coefficients,
$D_k$ is the nonnegative diagonal predicted noise variance along the fixed
trajectory, and $\nu_k$ is the sampler's injected-noise scale ($\nu_k = 0$
for a deterministic DDIM chain). The delayed
start determines the active transitions. A final step that carries variance
unchanged is omitted from this index, matching the MUMINS convention.

\begin{proposition}[Uniform injection preserves spatial ordering]
\label{prop:propagation-ranking}
For the fixed trajectory and active transitions above, define
\[
w_k=\prod_{j=k+1}^{n-1}a_j^2,\qquad
L=\sum_{k=0}^{n-1}w_kb_k^2D_k,\qquad
F=\sum_{k=0}^{n-1}w_k\nu_k^2.
\]
Then $V_n=L+FI\succeq FI$. The propagated standard-deviation map
$\sqrt{V_n[i,i]}=\sqrt{L_{ii}+F}$ has the same voxel ordering, including
ties, as $L_{ii}$. Consequently, adding the uniform injection contribution
cannot change Spearman correlation with a fixed error map or selective
prediction based only on the uncertainty ordering, wherever those metrics
are defined and ties are handled consistently.
\end{proposition}
\begin{proof}
Repeated substitution in the recursion gives
$V_n=\sum_k w_k(b_k^2D_k+\nu_k^2I)=L+FI$.
All summands of $L$ are positive semidefinite. Since $F\ge0$ is common
to all coordinates, $x\mapsto\sqrt{x+F}$ is strictly increasing on
$x\ge0$ and preserves the ranks.
\end{proof}

This result concerns rank-based evaluation, not calibrated intervals:
$F$ can change coverage and interval width. A common positive conversion
back to image units likewise preserves ranks. The statement does not
compare different samplers or start times: changing $\eta$, the active
window, or the trajectory can change $D_k$, $b_k$, and the temporal weights,
and therefore change the ordering. It provides no guarantee that the
spatially varying term $L$ ranks reconstruction error well. A spatially
constant injection floor by itself cannot explain a loss of ranking
performance; that explanation must involve other parts of the estimator
or evaluation.

\section{Testing the curvature term against an exact denoiser}
\label{app:curvtest}

Theorem~\ref{thm:curved} is the one claim in this paper that our
manifold experiments do not touch, because those probe \emph{on} the
manifold where $\shape_{n^\ast} = 0$ and the second-order statement
collapses to the first-order one.  This appendix tests it directly,
and does so without a trained network: the optimal denoiser is
computed by quadrature, so hypothesis (H4) holds by construction,
$\Delta_\theta = 0$, and any disagreement is the theorem's own
remainder rather than model error.

\paragraph{Setup.}
$\M$ is the unit circle in $\R^2$ with uniform density, chosen because
its projection is closed form: $\pi_\M(y) = R\,y/\|y\|$, whose
differential at $x^\ast = (R+h)\hat n$ is $[R/(R+h)]\PT$.  With the
outward unit normal, $W_\nu = R^{-1} I_T$ and $\kappa = 1/R$, so
Eq.~\eqref{eq:dpi} reads $(1 + h/R)^{-1}\PT$ and reproduces that
exactly.  The prediction of Eq.~\eqref{eq:kappa} therefore carries no
free parameter and no sign convention to argue about:
\begin{equation}
    \lambda_{\text{tangent}}
    \;=\; \frac{1}{(1 + h\kappa)^2}
    \;=\; \Bigl(\frac{R}{R+h}\Bigr)^{\!2}.
    \label{eq:circlepred}
\end{equation}
This is also the test that distinguishes the exact covariance
eigenvalue $(1+h\kappa)^{-2}$ from the value obtained by squaring the
first-order projection $(I-hW_\nu)\PT$, namely $(1-h\kappa)^2$: at
$h = 0.3$ they predict $0.592$ and $0.490$, a $17\%$ gap that the
measurement resolves cleanly in favour of the exact form.

The residual is formed exactly as in the toy scripts: fix $x^\ast$,
draw $\xi$, set $X_s = \sqrt{\bar\alpha}\,x^\ast +
\sqrt{1-\bar\alpha}\,\xi$, take $r = \xi - \epsilon(X_s)$, with
$\epsilon$ the exact conditional mean, and both expectations are
quadratures rather than Monte Carlo, so no sampling error enters.

\begin{table}[h]
\centering
\small
\begin{tabular}{rrrrr}
\toprule
$h$ & $h/\sigma$ & predicted & measured & rel.\ error \\
\midrule
$-0.20$ & $-2.0$ & $1.5625$ & $1.5425$ & $-1.28\%$ \\
$-0.10$ & $-1.0$ & $1.2346$ & $1.2206$ & $-1.13\%$ \\
$\phantom{-}0.00$ & $0.0$ & $1.0000$ & $0.9898$ & $-1.02\%$ \\
$+0.10$ & $+1.0$ & $0.8264$ & $0.8188$ & $-0.92\%$ \\
$+0.20$ & $+2.0$ & $0.6944$ & $0.6886$ & $-0.84\%$ \\
\bottomrule
\end{tabular}
\caption{Tangent eigenvalue of $\Cov(r)$ against normal offset, at
$\sigma = 0.10$; ``predicted'' is Eq.~\eqref{eq:circlepred}.  The
prediction moves over a factor of $2.3$ and the measurement follows it
to about one per cent throughout, with the sign the theorem requires: probing inside the circle inflates the eigenvalue
above one, outside deflates it, and $h = 0$ recovers the flat result.
The normal component stays below $1.3 \times 10^{-2}$, confirming that
the normal direction is identified as the decomposition claims.}
\label{tab:curvtest}
\end{table}

\paragraph{The remainder is $O(\sigma^2)$, tested not asserted.}
Holding $h/\sigma = 1$ so the regime $\|n^\ast\| = O(\sigma)$ is
respected, the relative disagreement divided by $\sigma^2$ is
$-0.87$, $-0.90$, $-0.92$, $-0.94$, $-0.96$ at
$\sigma = 0.20, 0.14, 0.10, 0.07, 0.05$.  A sixteen-fold change in
$\sigma^2$ moves the ratio by nine per cent, so the error really does
scale as the theorem's remainder, with a constant near $-0.9$.

\paragraph{Control.}
The residual could have been quadrature error rather than physics, and
at small $\sigma$ the kernel narrows so discretisation would
\emph{grow} as $\sigma$ falls, not shrink.  It shrinks.  Refining both
rules eightfold, $n_\theta$ from $1024$ to $8192$, Gauss--Hermite
order from $48$ to $128$, leaves the measured eigenvalue unchanged at
$0.818837$ to six figures.  The quadrature is converged and the one
per cent is the second-order term.


\section{Synthetic validation on manifolds with known tangent spaces}
\label{app:toy}

Everything in Sec.~\ref{sec:tpt} rests on Proposition~\ref{prop:pt},
and on medical data that proposition cannot be checked directly:
$\M$ is unknown, so $\PT$ has no ground truth to be compared against.
The results of Sec.~\ref{sec:results} are
therefore consistent with the geometric reading but do not establish
it.  A flat $t^\ast$ curve, in particular, is equally consistent with
normalisation effects, denoiser saturation, or spatial correlation
induced by the U-Net.

This appendix removes the ambiguity by running the same estimator on
manifolds where $\PT$ is available in closed form: an analytic
manifold of dimension one, and an analytic manifold of dimension two
with position-dependent curvature.  Both use the same
$32 \times 32$ embedding ($d = 1024$), the same 2-D U-Net \citep{ho2020ddpm} with adaptive group normalisation (AdaGN) \citep{dhariwal2021diffbeatgans}, $T = 100$ timesteps, and the same hybrid objective
$\mathcal{L}_{\text{simple}} + \lambda \mathcal{L}_{\text{vlb}}$ at
$\lambda = 10^{-3}$, so what changes between them is the geometry and
not the recipe.  The objective is of the same family as the medical
checkpoints', though trained here from scratch at a much smaller
scale; we do not claim the two share hyper-parameters.

\begin{table}[h]
\centering
\small
\begin{tabular}{llcc}
\toprule
Stage & Manifold & $\dim \M$ & codim \\
\midrule
S1 & $S^1$, one Gaussian blob on a circular orbit & $1$ & $1023$ \\
S2 & $T^2$, two independent blobs                 & $2$ & $1022$ \\
\bottomrule
\end{tabular}
\caption{The two synthetic regimes.  Codimension is severe by design:
a rank-$1$ covariance inside $\R^{1024}$ is not something a
normalisation artefact produces by accident.}
\label{tab:toy-setup}
\end{table}

\subsection{$\Cov(r)$ recovers the tangent eigenspace and effective rank}
\label{app:toy-projector}

The prediction of Proposition~\ref{prop:pt}, read in its broad-prior
limit, is that the residual covariance is tangent-supported with the
leading eigenspace of the projector onto $T_{x_0}\M$: its spectrum should consist of $\dim \M$ eigenvalues
near one and $d - \dim \M$ eigenvalues near zero, and its leading
eigenspace should coincide with the analytic tangent space.  Both are
directly testable in S1 and S2.

\begin{table}[h]
\centering
\small
\begin{tabular}{lccccc}
\toprule
& $\lambda_1$ & $\lambda_2$ & $\lambda_3$ & subspace alignment with $T_{x_0}\M$ \\
\midrule
S1 ($\dim = 1$) & $0.9931$ & $0.0118$ & --- & $0.9998$ \\
S2 ($\dim = 2$) & $1.052$  & $0.839$  & $0.018$ & $0.9967$ \\
\midrule
Prop.~\ref{prop:pt} (broad-prior limit) & $1$ & $1$ if $\dim \ge 2$, else $0$ & $0$ & $1$ \\
\bottomrule
\end{tabular}
\caption{Spectrum and leading eigenspace of $\Cov(r)$ against the
analytic prediction.  Alignment is the mean squared canonical
correlation between the leading $\dim \M$ eigenvectors and the
analytic tangent span.  On S1 the eigengap $\lambda_1/\lambda_2$ is
$84.2$; on S2 the gap between the second and third eigenvalues is a
factor of $47$, so the intrinsic dimension is read off the spectrum
without a threshold being tuned.}
\label{tab:toy-projector}
\end{table}

The spectrum exhibits the correct effective rank in both cases
without a threshold being tuned, and the leading
eigenspace agrees with the analytic tangent space to four decimal
places on S1 and three on S2.  We take this as validation of tangent
support, effective rank and leading-eigenspace recovery.  It is not
evidence that the finite-noise, trained-denoiser covariance equals the
exact projector, Prop.~\ref{prop:pt}(iii) itself forbids that at
finite $\tau$.

\subsection{The correction terms are measurable, and not zero}
\label{app:toy-correction}

Proposition~\ref{prop:pt} is a statement about the limit
$K \to \infty$ \emph{and} $\sigma_{t^\ast} \to 0$ under an optimal
denoiser.  At finite $K$, finite noise, and with a trained network, the
estimator carries three separable errors,
\begin{equation}
    \widehat C_{K,t}(x_0) - \PT
    \;=\; E_K(t) \;+\; E_{\mathrm{geom}}(t) \;+\; E_\theta(t),
    \label{eq:toy-correction}
\end{equation}
where $E_K$ is sample-covariance error, $E_{\mathrm{geom}} =
\mathcal{O}(\sigma_t^2)$ for an on-manifold probe by
Thm.~\ref{thm:curved}, and $E_\theta$ is the gap between the trained
and the population denoiser.  We do not quote an operator-norm rate
for $E_K$, in this codimension that would require a dependence on
effective rank we have not derived, and instead measure it by
repetition (App.~\ref{app:modelerror} gives the finite-$K$ eigenvalue
spread).  On the toys, all three can be measured rather than assumed.

The geometric term is visible in the \emph{mean} residual, not the
covariance: on the manifold, $\E[r] = \sigma\,b(x_0) +
\mathcal{O}(\sigma^2)$, so the posterior-mean displacement
$\sigma\E[r]$ offsets along the mean curvature vector at order
$\sigma^2$.  Sweeping the probe timestep on S1
gives a ratio between the measured and predicted offset that is linear
in $\sigma$,
\[
    \frac{\sigma\,|\langle \bar r, \hat H\rangle|}{\sigma^2 \kappa}
    \;\approx\; 1.171 - 0.138\,\sigma,
    \qquad \kappa = 0.497,
\]
across $\sigma \in [0.04, 0.81]$.  On S2, at a single probe step, the
measured offset is $0.0508$ against a predicted $\sigma^2 |H| =
0.0358$, a ratio of $1.42$.

Two things follow.  The curvature scaling is correct; the offset
tracks $\sigma^2 |H|$ over more than a decade of $\sigma$, but the
extrapolation to $\sigma \to 0$ lands at $1.17$ rather than $1$.  A
residual of roughly $17\%$ therefore survives the small-noise limit
and is not curvature.  We do \emph{not} identify it with $E_\theta$:
the experiment cannot separate denoiser approximation from
discretisation of the schedule, higher-order geometric terms, and
finite-sample error in the covariance, and we interpret the
discrepancy as a combination of these.  What the measurement does
establish is a magnitude; the estimator approximates $\PT$ to within
tens of per cent at the noise scales used, not to within per cent,
and that magnitude is now on record.

\subsection{Naive variance propagation, against ground truth}
\label{app:toy-varprop}

Proposition~\ref{prop:expansive} argues that the affine variance
recursion inflates.  On S1, the argument can be checked against a
quantity that is otherwise never available: the true per-pixel
variance of the reverse chain, obtained by Monte-Carlo over the toy
model, where nothing is approximated, and the manifold is known.

Eq.~\eqref{eq:overshoot} is obtained here, and the protocol matters
enough to give in full.  We draw $J = 20$ unconditional DDPM samples
from the S1 model (seeds $100$--$119$, $T = 100$ ancestral steps).
Along each chain, we propagate the diagonal variance with the affine
additive recursion of App.~\ref{app:recursion} in its stochastic
form, as those implementations do along a DDIM chain with $\eta = 1$
(the ancestral DDPM step is that case), using the model's
own $\sigma_\theta$ head, under the independence assumption those
implementations make.
The empirical reference is the per-pixel variance across the $J$
samples; the analytic quantity is the per-pixel propagated variance
averaged over the $J$ runs.  Eq.~\eqref{eq:overshoot} is the ratio of
their means over the full $32 \times 32$ frame.

The single most important detail is that accumulation is
\emph{delayed-started at $t = 60$ of $100$}, so the recursion is
already running with the mask that published implementations apply.
The overshoot of $13.7\times$ is therefore not what happens if one
ignores the accepted remedy; it is what remains after applying it.
Cor.~\ref{cor:mask} says the mask bounds the inflation rather than
removing it, and this is the size of what it fails to remove.

Three caveats.  $J = 20$ is a small Monte-Carlo estimate: its standard
error is immaterial against a factor of $13.7$, but it is not zero.
The S1 model is unconditional, so the empirical variance is the spread
of the generative distribution over the orbit rather than the
predictive variance of a conditional forecast, the same object the
recursion claims to track in this setting, but not identical to the
conditional quantity of Sec.~\ref{sec:tpt-geom}.  And this is one
manifold with one seed set; we report the magnitude, not a
distribution over regimes.

\subsection{On- and off-manifold behaviour}
\label{app:toy-separation}

The probe input matters, and the toys say why.  Displacing the query
point off the manifold along the normal direction and measuring the
residual drift $\|\bar r\|$ separates cleanly in both stages: by
$4.54\times$ on S1 and $2.79\times$ on S2, with no overlap between the
on- and off-manifold histograms in either case.

This is the toy version of a failure we later met on real data.
Probing at $[x_{\text{baseline}}, 0]$ rather than at a reconstruction
places the query far off the learned manifold, and the resulting map
carries essentially no information about the error
($\bar\rho = 0.004$).  Probing at the model's own reconstruction
restores it.  The predicate in Prop.~\ref{prop:pt} that $x_0$ lie on
$\M$ is not a formality; it is the whole method, and S1--S2 show the
degradation is monotone in the distance to $\M$ rather than a cliff.

We note, as the medical sections do, that a DDIM reconstruction is
only approximately on $\M$: the sampler is approximate, the score
imperfect, and $\zeta = 5$ skips.  The claim we make is the weaker and
correct one, that it lands near enough for the projector reading to
survive.

\subsection{Reproducing}
\label{app:toy-repro}

Both stages (S1 and S2 in App.~\ref{app:toy}) are self-contained scripts of a few hundred lines, train
their own model from scratch, run on a single consumer GPU or on CPU
in under half an hour, and depend only on \texttt{torch} and
\texttt{numpy}.  Seeds are fixed at $42$ throughout.  Scripts and
logs accompany the submission.

\medskip\noindent\emph{Experimental design and implementation.}\medskip

\section{Cost accounting}
\label{app:budgets}
The $\sqrt{\bar\alpha}$ in the diagonal of Sec.~\ref{sec:exact}
follows from $\partial \hat{x}_0/\partial x_t =
\bar\alpha^{-1/2}\bigl(I + (1-\bar\alpha)\nabla^2 \log p_t\bigr)$
and $\operatorname{Cov}(x_0 \mid x_t) = \tfrac{1-\bar\alpha}{\bar\alpha}
\bigl(I + (1-\bar\alpha)\nabla^2 \log p_t\bigr)$.
Budgets are counted in forward-equivalent network evaluations and were
fixed before any score was read.  A Jacobian--vector product carries a
primal and a tangent, so it costs two forwards of compute in
principle; the conversion actually used is the measured wall-time
ratio given at the end of this section.  One DDIM
reconstruction at skip $\zeta$ costs $T/\zeta + 1$ evaluations, $61$ at
$T=300$, $\zeta=5$.  The MC-DDIM reference with $J$ chains therefore
costs $J(T/\zeta+1) = 1{,}220$ at $J=20$.  \tpt{} costs
$(T/\zeta+1) + K$: $76$ at $K=15$ and $91$ at $K=30$, or $16\times$ and
$13.4\times$ fewer.  The row-sum surrogate is one JVP on top of the
reconstruction; Hutchinson at $M$ probes is $M$ JVPs; the exact
diagonal is $d$ JVPs, one per input coordinate, $12{,}288$ at $64^2$
RGB.  JVPs are converted to forward-equivalents with a wall-time
ratio \emph{measured on each architecture} rather than the
theoretical factor of two, which is right about memory and wrong
about time, because forward-mode kernels are not the fused ones:
$6.5\times$ on the $64^3$ clinical model and about $4\times$ on the
natural-image checkpoint, which is where the roughly $49{,}000$
forward-equivalents of Sec.~\ref{sec:res-fidelity} come from.

\begin{table}[h]
\centering\small
\begin{tabular}{lrrl}
\toprule
estimator & peak memory & wall time & one $128^3$ volume \\
\midrule
\tpt{}, $K{=}30$            & $33.8$~GiB & $66.7$~s & completes \\
row-sum surrogate (one JVP) & $>40$~GB   & ---      & out of memory \\
Hutchinson, $M{=}50$        & $>40$~GB   & ---      & out of memory \\
Hutchinson, $M{=}500$       & $>40$~GB   & ---      & out of memory \\
\bottomrule
\end{tabular}
\caption{Memory feasibility at $128^3$ on a $40$~GB accelerator
($39.5$~GiB usable), one estimator per invocation in a clean process,
released implementations with activation checkpointing disabled for
the tangent pass (App.~\ref{app:checkpointing}).  The three
Jacobian-based estimators exhaust memory before completing a single
volume; the probe differentiates nothing and completes.}
\label{tab:memprobe}
\end{table}

\section{Why forward-mode differentiation cannot traverse the public codebases}
\label{app:checkpointing}
\begingroup\sloppy
Forward-mode differentiation through the released implementations fails
before it reaches the denoiser.  All three wrap their residual blocks
in PyTorch's activation checkpointing, whose
\texttt{CheckpointFunction} is an old-style autograd function: it defines \texttt{forward} and
\texttt{backward} but no \texttt{setup\_context}, and
\texttt{torch.func.jvp} requires the latter to push a tangent through a
custom function.  The call raises rather than silently returning a wrong
tangent, which is the good failure.  We disable checkpointing for the
probe pass; nothing else about the network changes, and the primal
outputs are identical to the released ones.  This is an implementation
obstacle, not a mathematical one, but it is the reason none of the three
can be run as released.\par\endgroup

\section{Protocol for the checkpoint without a variance head}
\label{app:nohead}

This checkpoint supports the head-free-compatibility claim, so its
provenance should not live in a table footnote.  Everything below is
read from the training script and configuration shipped with the code.

Architecture and data are identical to the hybrid PNG checkpoint \citep{oliveras2026mumins},
same 3D U-Net (base width $64$, multipliers $1,2,4,8$, conditioning
dimension $256$), same two-channel joint state at $64^3$, same
training split and preprocessing, same $T = 300$ cosine schedule,
with one change: \texttt{learned\_variance = false}, so the network
emits $C$ output channels instead of $2C$ and the
$\mathcal{L}_{\text{vlb}}$ term vanishes from the loss, which reduces
to the $\ell_1$ noise-prediction objective.  Training ran $150$k steps
at batch size $4$ on two GPUs with the repository's default optimiser
settings; the training log confirms the loss dictionary contains no
NLL term.  At inference, the sampler is the one used everywhere else
($\zeta = 5$, $\eta = 1$); its $J = 20$ Monte-Carlo reference and
the probe's chain both carry the observed blur difference in the
conditioning vector, so this experiment also shows the probe under a
second conditioning of the same models.


\section{Operational definitions of masks and metrics}
\label{app:defs}

Every quantity below is computed per volume, inside the stated mask,
exactly as implemented in the released code; nothing here is a
post-hoc idealisation.  The three thresholds (foreground fraction,
change percentile, edge percentile) are the defaults fixed when the
masking module was written, before any masked result was seen, and
were not tuned afterwards.

\paragraph{Conditioning arm.}
The conditioning vector carries a blur difference computed from both
scans of the pair.  The Monte-Carlo reference and the analytical
baseline were generated with that component zeroed, the value a
forecasting deployment supplies, and so was the probe's chain on both
tasks; every estimator in Table~\ref{tab:clinical} is scored against
the same common error map, built from the zero-blur ensemble, so the
three are compared under one conditioning.

\paragraph{Clamping.}
Every estimate of the Tweedie diagonal is clamped at zero before
scoring, since the implied variance can be negative; the rate on
medical data is reported in App.~\ref{app:functional}.

\paragraph{Masks.}
\emph{Tissue (foreground):} voxels whose baseline intensity exceeds
$2\%$ of the volume's dynamic range, computed from the baseline only.
On skull-stripped brain MRI, this recovers the brain ($\approx 24\%$ of
voxels); on cropped lung CT, it removes only the air surrounding
($99.5\%$ retained), which is why the tissue row there is a
consistency check.  If the threshold leaves fewer than $100$ voxels,
the mask falls back to the full volume.  The threshold is not
load-bearing: rescored at $1\%$ and $5\%$ over the same $211$
volumes, the participant-weighted primary contrast moves by
$+0.0008$ and $-0.0017$ from its value at the prespecified $2\%$,
and the gradient-controlled contrast by $+0.0009$ and $-0.0017$,
with intervals of the same width: sign and magnitude are stable
across a five-fold change of threshold.
\emph{Edges:} tissue voxels in the top $20\%$ of $|\nabla u_0|$, the
gradient magnitude by central differences, percentile per volume
within tissue.
\emph{Interior:} tissue minus edges.
\emph{Regions of true change:} tissue voxels in the top $10\%$ of
$|x_{\text{target}} - u_0|$, percentile per volume within tissue.
This uses the target and is therefore an \emph{oracle evaluation
slice}, not a mask available at deployment.

\paragraph{Partial Spearman.}
Within the mask, rank-transform $\sigma$, $e$ and $|\nabla u_0|$;
linearly regress the ranks of $\sigma$ and of $e$ on the ranks of
$|\nabla u_0|$; report the Pearson correlation of the two residual
vectors.  This controls for the monotone component explained by
gradient magnitude, no more.  A synthetic map built from
$|\nabla u_0|$ itself scores $0$ under this control, which is the
sense, and the only sense, in which ``a pure edge tracer scores
zero''.  If either variable is a deterministic function of the control,
the statistic is reported as $0$ rather than NaN.

\paragraph{Selective-prediction metrics.}
\emph{AUSE:} remove voxels in decreasing $\sigma$ in $20$ steps over
removal fractions $0$ to $0.95$; at each step record the mean error of
the retained set, normalised by the mean error at zero removal; AUSE
is the trapezoidal area between this curve and the oracle curve that
removes by true error, with the random-ordering area reported as the
scale.
\emph{AURC:} coverage from $0.05$ to $1$ in $20$ steps retaining
voxels in increasing $\sigma$; risk is the mean error of the retained
set; the oracle ordering is the attainable floor.
\emph{AUROC:} positives are the top $5\%$ of error voxels within the
mask, computed by the Mann--Whitney identity.

\paragraph{Configuration index.}
One place for every number needed to re-run the medical experiments;
each entry is the value in the released launch scripts and configs.

\begin{table}[h]
\centering
\small
\begin{tabular}{lll}
\toprule
Field & PNG (hybrid / head-free) & OASIS \\
\midrule
Input resolution & $64^3$ & $128^3$ \\
Scored split & $130$ pairs, $47$ participants & $211$ pairs, $63$ participants \\
Forward schedule & cosine, $T=300$ & cosine, $T=300$ \\
Reverse sampler & DDIM, $\zeta=5$, $\eta=1$ & DDIM, $\zeta=5$, $\eta=1$ \\
MC reference & $J{=}20$, seeds $42$--$61$ & $J{=}20$, seeds $42$--$61$ \\
Self-chain probe & chain seed $42$ & chain seed $42$ \\
Conditioning $c$ & \multicolumn{2}{l}{$(\Delta\tau,\ \text{age},\ \Delta b)$: interval, baseline age, blur difference of the pair} \\
$t^\ast$, $K$ & $60$, $30$ & $60$, $15$ \\
$\sigma_\theta$ baseline & propagated, $\shat{=}15$; none for the head-free variant & propagated, $\shat{=}15$ \\
Preprocessing & source pipeline & source pipeline \\
Head-free variant & PNG only, retrained without the head (App.~\ref{app:nohead}) & --- \\
Code, commit, versions & \multicolumn{2}{l}{pinned in the released repository} \\
\bottomrule
\end{tabular}
\caption{Configuration index for the medical experiments.
The unscored OASIS pairs are duplicated split entries or corrupted
files, excluded before any scoring.  Masks and
metrics are defined operationally above; the statistical protocol is
below and in App.~\ref{app:controlled}.}
\label{tab:configindex}
\end{table}

\paragraph{Statistical plan.}
The primary endpoint on OASIS is the participant-weighted paired
difference of tissue Spearman: correlation per volume, mean over each
participant's volumes, mean over participants, paired bootstrap
resampling participants with $B = 2\times10^4$ and percentile
intervals.  The gradient-controlled partial Spearman is the
robustness endpoint.  Edges, interior, regions of change, AUROC, AUSE
and AURC are secondary and unadjusted for multiplicity; their role is
to show the direction is consistent, and any one alone should be read
as exploratory.


\medskip\noindent\emph{Additional results and robustness checks.}\medskip

\section{The fourteen conditions}
\label{app:conditions}
\begingroup\sloppy
Five natural-image checkpoints against three corpora would be fifteen
combinations, and two are absent: \texttt{church/real} and
\texttt{ffhq/real}.  That leaves thirteen: \texttt{cond},
\texttt{inet} and \texttt{cifar} on all three corpora each, and
\texttt{church} and \texttt{ffhq} on \texttt{self} and \texttt{coco}.
The fourteenth is \texttt{brats/self}.

\paragraph{The medical checkpoint.}
\texttt{brats} is the MONAI model-zoo bundle
\texttt{brats\_mri\_\allowbreak axial\_slices\_\allowbreak generative\_diffusion}
\citep{pinaya2023monaigen}: a latent diffusion model over axial
slices of the BraTS glioma MRI collection \citep{menze2015brats},
with a KL autoencoder and a $1{\times}64{\times}64$ latent, so the exact
diagonal is $4{,}096$ one-hot Jacobian--vector products per slice.
It is read exactly as the bundle ships it: its own configuration
instantiates both networks, the released weights load with no missing
or unexpected key, the bundle carries a single weight set and no EMA
copy, and it states no latent scale factor, so latents are taken at
the autoencoder's native scale.  The corpus is $100$ latents drawn
from the model's own prior by deterministic DDIM (seed $20260813$),
scored in latent space without decoding, since a decode--encode round
trip would fold the autoencoder's reconstruction error into a
measurement that concerns the denoiser's Jacobian.  The generator is
unconditional: there is no follow-up to predict, so it enters
Table~\ref{tab:conditions} and nowhere else.  On it, the exact
diagonal correlates with the error at $+0.320\;[+0.298, +0.342]$,
with no clamped voxel; Hutchinson at $M{=}200$ reproduces it at rank
agreement $0.995$; \tpt{} reads $+0.315\;[+0.292, +0.337]$; and the
row-sum $+0.015$.  It is the one medical checkpoint in this paper
whose weights the reader can download.\par\endgroup

\begin{table}[h]
\centering\footnotesize
\begin{tabular}{llrrrrr}
\toprule
checkpoint & corpus & $\rho(\nabla,e)$ & exact & Hutch.\ $M{=}200$ & row-sum & \tpt{} $K{=}30$\\
\midrule
cond & coco & 0.389 & \textbf{0.474} & 0.465 & 0.036 & \textbf{0.484}\\
cond & real & 0.382 & \textbf{0.476} & 0.466 & 0.063 & 0.470\\
cond & self & 0.340 & \textbf{0.448} & 0.435 & 0.004 & \textbf{0.451}\\
brats & self & 0.297 & 0.320 & 0.320 & 0.015 & 0.315\\
church & self & 0.130 & 0.224 & 0.224 & $-0.009$ & 0.218\\
church & coco & 0.064 & 0.165 & 0.165 & $-0.007$ & 0.167\\
inet & self & 0.022 & 0.058 & 0.056 & $-0.017$ & 0.013\\
ffhq & coco & 0.006 & 0.054 & 0.054 & $-0.027$ & 0.051\\
cifar & self & $-0.007$ & 0.029 & 0.023 & $-0.003$ & 0.009\\
ffhq & self & $-0.024$ & 0.002$^{\dagger}$ & 0.002$^{\dagger}$ & $-0.059$ & $-0.003$\\
\midrule
cifar & real & 0.037 & $\mathbf{-0.077}$ & $\mathbf{-0.071}$ & $-0.169$ & $+0.018$\\
cifar & coco & 0.054 & $\mathbf{-0.091}$ & $\mathbf{-0.085}$ & $-0.167$ & $+0.004$\\
inet & real & 0.040 & $\mathbf{-0.131}$ & $\mathbf{-0.129}$ & $-0.215$ & $-0.014^{\ddagger}$\\
inet & coco & 0.066 & $\mathbf{-0.132}$ & $\mathbf{-0.131}$ & $\mathbf{-0.232}$ & $-0.016^{\ddagger}$\\
\bottomrule
\end{tabular}
\caption{Spearman correlation with the denoising error, per condition,
mean over $100$ images.  Rows are ordered by $\rho(\nabla,e)$, how much
of the error is spatial structure.  The block below the rule is where the
exact diagonal is \emph{anti}-correlated with the error.  The
Hutchinson column carries the same four harmful conditions as the
exact column, with worst value $-0.131$ and mean over the fourteen
$0.128$.  Every interval in both columns
excludes zero except $\dagger$ (\texttt{ffhq/self}, $[-0.006,
0.010]$ for each).  For \tpt{}, the only
intervals that exclude zero on the negative side are the two marked
$\ddagger$: \texttt{inet/real} $[-0.021, -0.007]$ and
\texttt{inet/coco} $[-0.023, -0.009]$, on the checkpoint where the
diagonal reverses hardest; neither lies entirely below the $-0.02$
threshold, whereas four exact-diagonal intervals do.
Over the ten conditions above the rule, where the diagonal is not
reversed, the exact column means $0.225$, Hutchinson $0.221$ and
\tpt{} $0.218$, and eight of the ten \tpt{}--exact differences lie
within $0.010$; the two exceptions are \texttt{inet/self} and
\texttt{cifar/self}, where the exact diagonal itself reaches only
$0.058$ and $0.029$.
The exact column is clamped at zero wherever the implied variance is
negative; the clamped fraction is zero on both pixel-space
unconditional checkpoints, below $0.2\%$ on the latent ones, and at
most $5.7\%$ on the class-conditional one.}
\label{tab:conditions}
\end{table}

\begin{table}[h]
\centering\footnotesize
\begin{tabular}{lrr}
\toprule
estimator & $\bar\rho(\sigma,e)$ & forward-equivalents \\
\midrule
Hutchinson, $M{=}5$   & $0.298$ & $\approx 20$ \\
Hutchinson, $M{=}15$  & $0.359$ & $\approx 60$ \\
Hutchinson, $M{=}50$  & $0.410$ & $\approx 200$ \\
Hutchinson, $M{=}200$ & $0.435$ & $\approx 800$ \\
exact diagonal        & $0.448$ & $\approx 49{,}000$ \\
\tpt{}, $K{=}30$      & $0.451$ & $30$ \\
\bottomrule
\end{tabular}
\caption{The $M$ sweep behind Sec.~\ref{sec:res-fidelity}, on
\texttt{cond/self} (the same $100$ images as Table~\ref{tab:conditions}):
Spearman correlation with the denoising error as the Hutchinson
estimator converges to the exact diagonal.  Forward-equivalents follow
App.~\ref{app:budgets}.  The pattern is the same in all fourteen
conditions; at $M{=}200$ Hutchinson tracks the exact column of
Table~\ref{tab:conditions} to within $0.013$ here, and at least $0.92$
in rank agreement in every one of the fourteen.}
\label{tab:mgrid}
\end{table}

\section{The diagonal reference against the probe, in full}
\label{app:functional}
\begin{table}[h]
\centering
\small
\begin{tabular}{llcc}
\toprule
Estimator & functional of $J$ & $\rho(\cdot,e)$ & partial, given $J_{ii}$ \\
\midrule
Hutchinson, $M{=}50$ & $J_{ii}$, unbiased & $0.321$ & $+0.056\;[+0.041, +0.070]$ \\
Diagonal reference ($M{=}500$) & $J_{ii}$ & $0.361$ & --- \\
Row-sum (Boys)       & $\sum_j J_{ij}$    & $0.474$ & $+0.353\;[+0.283, +0.420]$ \\
\tpt{}, $K{=}30$     & $J_{ii}$ and off-diag.\ mass & $\mathbf{0.539}$ & $\mathbf{+0.409}\;[+0.322, +0.489]$ \\
\midrule
\multicolumn{3}{l}{Diagonal reference, given \tpt{}} & $+0.056\;[+0.032, +0.085]$ \\
\bottomrule
\end{tabular}
\caption{Partial Spearman correlations with the error map, twenty
$64^3$ lung-CT volumes from twelve subjects, interior mask, clamped
voxels excluded.  The diagonal reference is Hutchinson at $M{=}500$;
the one-hot exact diagonal is not affordable at this size.  Intervals
are a bootstrap over subjects.  The last row is the reverse
conditioning.}
\label{tab:functional}
\end{table}
Table~\ref{tab:functional} reports the interior mask with clamped voxels
excluded, which is the condition least favourable to our claim.  The
ordering is unchanged in the other two.  On the whole volume the probe
retains $+0.450\;[+0.367, +0.529]$ conditioned on the diagonal
reference,
against $+0.059\;[+0.033, +0.085]$ the other way, with the row-sum at
$+0.347$ and Hutchinson at $+0.024$.  Inside the interior with clamped
voxels retained the four figures are $+0.420$, $+0.064$, $+0.367$ and
$+0.025$.  Excluding the clamped voxels moves the reference's own
correlation with the error from $0.364$ to $0.361$, so the clamping is
not what holds it back.  The clamped fraction is $7.5\%$ of interior
voxels on average and $17.4\%$ on the worst of the twenty volumes,
with estimation noise contributing ($10.3\%$ at $M{=}50$);
$\sigma_{\text{TW}}$ never clamps.  Twenty volumes from twelve subjects
throughout; intervals resample subjects.  On the full cohort of $130$
volumes, Hutchinson at $M{=}200$ on the probe's own chain (the
Table~\ref{tab:clinical} column) has a split-half agreement of
$0.79$, clamps $5.5\%$ of voxels, and trails the probe by
$+0.189\;[+0.158, +0.222]$ inside tissue and $+0.142\;[+0.123, +0.161]$
on the gradient-controlled endpoint, participant-level intervals;
its peak memory at $64^3$ is $2.1\times$ the probe's and its wall time
per volume $20\times$.

\begin{table}[h]
\centering\scriptsize
\setlength{\tabcolsep}{2.5pt}
\begin{tabular}{lcccc}
\toprule
condition & exact vs.\ $e$ & \tpt{} vs.\ $e$ & \tpt{}, given exact & exact, given \tpt{} \\
\midrule
inet/coco  & $-0.132\,[-0.164,-0.101]$ & $-0.016\,[-0.023,-0.009]$ & $+0.005\,[+0.002,+0.008]$ & $-0.131\,[-0.162,-0.100]$ \\
inet/real  & $-0.131\,[-0.165,-0.098]$ & $-0.014\,[-0.021,-0.007]$ & $+0.005\,[+0.001,+0.008]$ & $-0.130\,[-0.164,-0.097]$ \\
cifar/coco & $-0.091\,[-0.117,-0.064]$ & $+0.004\,[-0.002,+0.011]$ & $+0.007\,[+0.001,+0.013]$ & $-0.091\,[-0.117,-0.065]$ \\
cifar/real & $-0.077\,[-0.105,-0.049]$ & $+0.018\,[+0.012,+0.024]$ & $+0.020\,[+0.014,+0.025]$ & $-0.077\,[-0.106,-0.049]$ \\
\bottomrule
\end{tabular}
\caption{The same partial analysis on the four conditions where the
exact diagonal is anti-correlated with the error, $100$ images each,
full image, intervals, a bootstrap over images.  Here the conditioning
variable is the one-hot exact diagonal itself, not a Hutchinson
reference, and no voxel clamps.  The first two columns reproduce
Table~\ref{tab:conditions}.  Removing the probe from the diagonal
leaves the diagonal's reversal where it was; removing the diagonal
from the probe leaves a small positive partial that excludes zero in
every condition.  These partials describe the finite-budget maps and do not, by themselves, establish complementary information in the population estimand.  They are also not comparable in reliability with the clinical partials of Table~\ref{tab:functional}: the attenuation bound $|\rho| \le \sqrt{\mathrm{rel}}$ places the clinical map's reliability at $\ge 0.45$ and these four at $\le 0.02$, so the contrast between the two tables is a probe budget and not two verdicts on the functional.}
\label{tab:functional-natural}
\end{table}

\begin{table}[h]
\centering\scriptsize
\setlength{\tabcolsep}{2.5pt}
\begin{tabular}{lccccc}
\toprule
condition & $a\overline{\|J_{i:}\|^2}/\overline{|J_{ii}|}$ & $\rho(\|J_{i:}\|^2, e)$ & $\rho(\sigma_{\text{T2}}, e)$ & $\rho(\sigma_{\text{TW}}, \sigma_{\text{T2}})$ & \tpt{}, given $\sigma_{\text{T2}}$ \\
\midrule
inet/coco   & $0.34$ & $+0.135\,[+0.104,+0.167]$ & $-0.131\,[-0.163,-0.099]$ & $0.141$ & $+0.005\,[+0.002,+0.007]$ \\
inet/real   & $0.34$ & $+0.134\,[+0.100,+0.167]$ & $-0.130\,[-0.163,-0.096]$ & $0.143$ & $+0.005\,[+0.001,+0.008]$ \\
cifar/coco  & $0.28$ & $+0.093\,[+0.066,+0.119]$ & $-0.091\,[-0.117,-0.064]$ & $0.023$ & $+0.007\,[+0.001,+0.013]$ \\
cifar/real  & $0.28$ & $+0.079\,[+0.051,+0.106]$ & $-0.077\,[-0.105,-0.050]$ & $0.013$ & $+0.020\,[+0.014,+0.026]$ \\
\midrule
cond/self   & $0.98$ & $-0.442\,[-0.471,-0.413]$ & $+0.431\,[+0.402,+0.460]$ & $0.840$ & $+0.174\,[+0.159,+0.189]$ \\
\bottomrule
\end{tabular}
\caption{The row energy, read directly.  The same exact-Jacobian pass
accumulates $\|J_{i:}\|^2$ alongside $J_{ii}$ on the four reversed
conditions and, for contrast, on the self-conditioned ImageNet model
(\texttt{cond/self}, the best-behaved condition of
Table~\ref{tab:conditions}); $100$ images each, full image, intervals a
bootstrap over images.  $\sigma_{\text{T2}} = (1 - 2aJ_{ii} +
a^2\|J_{i:}\|^2)^{1/2}$ is Eq.~\eqref{eq:expansion} evaluated exactly,
with $a = \sigma_{t^\ast}$.  The first column is the quadratic term
against $\overline{|J_{ii}|}$; the linear term of
Eq.~\eqref{eq:expansion} carries a further factor of two, so the
quadratic term matches it at a column value of $2$.  On the reversed conditions
the row energy is the mirror image of the diagonal ($\rho(\|J_{i:}\|^2,
\sigma_{\text{exact}})$ between $-0.984$ and $-0.999$), because
$\|J_{i:}\|^2$ is dominated there by $J_{ii}^2$; the second-order map
ranks identically to the exact diagonal ($\rho = 1.000$ on all four,
$0.952$ on \texttt{cond/self}) and reverses with it; and the probe
retains conditioned on $\sigma_{\text{T2}}$ exactly what it retains
conditioned on $J_{ii}$ (Table~\ref{tab:functional-natural}).  The
$K{=}30$ map correlates with that combination at only $0.01$ to
$0.14$ on the reversed conditions, and Table~\ref{tab:kbudget} says
why: at that budget the map is barely reproducible there. The low agreement is consistent with this poor reproducibility and cannot, by itself, establish a difference in population rankings.  On
\texttt{cond/self}, where the map is reproducible, it agrees at $0.840$
and still retains $+0.174$ beyond it.}
\label{tab:rowenergy}
\end{table}

\paragraph{Is the neutrality the functional, or the probe budget?}
A near-zero correlation has two readings: the functional is
uninformative where the diagonal reverses, or the map has not been
resolved.  Table~\ref{tab:kbudget} separates them with three
budgets, $K = 30$, $120$ and $480$, on the four reversed conditions
and on \texttt{cond/self} as the informative control, and with a
direct measure of resolution.  A \emph{$g$-probe map} is the
statistic of Eq.~\eqref{eq:tpt} computed from $g$ of the $K$ draws
of one run; the \emph{reliability at $g$} is the Spearman
correlation between the $g$-probe maps built from draws $1$ to $g$
and from draws $g{+}1$ to $2g$, which share the image, the chain and
the timestep and nothing else, averaged over the $100$ images.  It
is the reproducibility of the map the body reports when $g = 30$.
Under a null of independent coordinates that correlation has
standard deviation $1/\sqrt{d}$, $0.018$ at $32^2$ and $0.009$ at
$64^2$.  Real maps are spatially dependent, so this is an order of
magnitude for comparison and not an exact threshold; the values
below are read as very low reproducibility rather than as tests
against it.

The reading is the second.  On \texttt{cond/self} the $30$-probe map
is already reproducible ($0.86$) and the correlation barely moves
with budget.  On the two \texttt{inet} conditions, reliability grows
monotonically with $g$, from $0.014$ at $15$ to $0.118$ at $240$,
and the correlation follows it towards the exact second-order map's
$-0.131$: $-0.016$, $-0.041$, $-0.079$, each step's interval
excluding zero, and the harmful threshold is crossed at $K \ge 120$ on
\texttt{inet/coco} and at $480$ on both.  On the two \texttt{cifar}
conditions, reliability stays at the noise floor through $g = 240$
and the correlation wanders within $0.02$ of zero; at $g = 15$ and
$30$ the two-group agreement is slightly negative, which is
what an unreproducible map gives and which we truncate at zero
wherever a reliability enters a formula.  As a consistency check,
classical attenuation, $\rho(\sigma_K, e) \approx
\rho(\sigma_{\text{T2}}, e)\,\sqrt{\mathrm{rel}(K)}$ with
$\mathrm{rel}(K)$ the Spearman--Brown extrapolation of the two-group
agreement at $K/2$, reproduces the twelve points of the four reversed
conditions within $0.02$ and the three control points within $0.05$.
It is an approximate model of attenuation, and we do not read an
identification of the population functional from its fit.
Hutchinson at $M = 5$, about $20$ forward-equivalents, is already
reproducible enough ($0.07$--$0.34$) to be harmful on all four.  The
$K = 30$ map is therefore not neutral about the reversal; it is too
unstable to say anything about it, and Hutchinson at a smaller
budget is not.

\begin{table}[h]
\centering\scriptsize
\setlength{\tabcolsep}{1pt}
\begin{tabular}{@{}lr|rrr|r@{}}
\toprule
 & $\sigma_{\text{T2}}$ & \multicolumn{3}{c|}{\tpt{}, $\rho(\sigma_K, e)$ [95\% CI]} & Hutch.\ $M{=}5$ \\
condition & vs.\ $e$ & $K{=}30$ & $K{=}120$ & $K{=}480$ & [95\% CI] \\
\midrule
cifar/coco & $-0.091$ & $+0.004\,[-0.002,+0.010]$ & $-0.014\,[-0.020,-0.008]$ & $-0.007\,[-0.016,+0.001]$ & $-0.043\,[-0.057,-0.029]$ \\
cifar/real & $-0.077$ & $+0.018\,[+0.012,+0.024]$ & $-0.015\,[-0.021,-0.009]$ & $-0.009\,[-0.017,-0.001]$ & $-0.042\,[-0.057,-0.028]$ \\
inet/coco  & $-0.131$ & $-0.016\,[-0.023,-0.010]$ & $-0.041\,[-0.052,-0.030]$ & $-0.079\,[-0.099,-0.059]$ & $-0.104\,[-0.131,-0.077]$ \\
inet/real  & $-0.130$ & $-0.014\,[-0.021,-0.007]$ & $-0.031\,[-0.043,-0.020]$ & $-0.067\,[-0.087,-0.047]$ & $-0.101\,[-0.129,-0.074]$ \\
\midrule
cond/self  & $+0.431$ & $+0.451\,[+0.422,+0.479]$ & $+0.465\,[+0.436,+0.493]$ & $+0.471\,[+0.443,+0.499]$ & $+0.298\,[+0.272,+0.326]$ \\
\midrule
\multicolumn{6}{l}{\emph{Reliability: rank agreement of two disjoint $g$-probe maps (reference $1/\sqrt{d}$: $0.018$ cifar, $0.009$ inet, cond)}} \\
 & & $g{=}15$ / $30$ & $g{=}60$ / $120$ & $g{=}240$ & Hutch.\ $M{=}5$ \\
cifar/coco & & $-0.026$ / $-0.013$ & $-0.001$ / $0.015$ & $0.002$ & $0.070$ \\
cifar/real & & $-0.026$ / $-0.013$ & $-0.002$ / $0.014$ & $0.000$ & $0.067$ \\
inet/coco  & & $0.014$ / $0.017$ & $0.027$ / $0.064$ & $0.118$ & $0.340$ \\
inet/real  & & $0.014$ / $0.017$ & $0.027$ / $0.064$ & $0.118$ & $0.340$ \\
cond/self  & & $0.755$ / $0.858$ & $0.918$ / $0.955$ & $0.976$ & $0.121$ \\
\bottomrule
\end{tabular}
\caption{Probe budget on the reversed conditions, with
\texttt{cond/self} as the informative control.  Upper block:
correlation with the error at three budgets from one run, each
with its own bootstrap interval over the $100$ images
($B = 2\times10^4$); a bracket belongs to the number on its left.
The $K = 30$ column reproduces the exact-Jacobian pass to the last
digit.  Probe sets at different $K$ are drawn independently.  Lower block: reliability of a
$g$-probe map, the Spearman agreement between the maps of two
disjoint groups of $g$ probes from the $K = 480$ draws, so
$g = 30$ is the reproducibility of the map the body reports.  The
last column is Hutchinson at $M = 5$, about $20$ forward-equivalents,
level above and split-half reliability below.  Two facts carry the
appendix: at matched budget Hutchinson is harmful on all four and the
probe on none, and given $4$--$16\times$ the probes the probe's
\texttt{inet} maps become more reproducible and harmful with them.}
\label{tab:kbudget}
\end{table}

\section{The open question, and two things that are not criteria}
\label{app:openq}
Candidate explanations of the two failure modes, all refuted.
Accounts (i), (ii) and the fourth below concern the natural-image sign;
account (iii) concerns the clinical change region.  (i)~That the Tweedie
conversion assigns low variance where the denoiser is most responsive,
at edges and texture: $\operatorname{diag}(J)$ is in fact \emph{lower}
where the intensity gradient is high in thirteen of the fourteen
conditions, the opposite sign, and removing the gradient from both
variables leaves the correlation essentially unchanged.  (ii)~That the
natural-image sign was a domain artefact: refuted by the real-image
corpora (Sec.~\ref{sec:res-sign}).  (iii)~That the error inside the
change region is less predictable from image structure, so no variance
estimate could track it: on lung CT this holds, as $\rho(\nabla,e)$
falls from $0.598$ to $0.071$; on brain MRI it does not ($0.230$ to
$0.197$, interval covering zero), and $\sigma$ collapses anyway.  The
drop is not a property of the $10\%$ threshold that defines the
region: on lung CT, at the top $5\%$, $10\%$ and $20\%$ of true
change, the probe reads $0.054$, $0.100$ and $0.233$ against
MC-DDIM's $0.169$, $0.265$ and $0.388$, paired gaps $-0.092$,
$-0.134$ and $-0.117$, every interval excluding zero; on brain MRI
the two read $0.101$, $0.128$, $0.106$ and $0.075$, $0.133$, $0.129$,
all far below the tissue-wide $0.349$ and $0.337$, with the ordering
between them changing sign across thresholds ($+0.028$, $-0.003$,
$-0.029$).
Two mechanisms remain consistent with the change-region drop and are
untested: conditioning on the baseline scan may suppress variance
exactly where the anatomy moves, since the model is asked for a
follow-up that stays close to what it was shown; and the model may be
confident on progression trajectories it has not seen, in which case
a confident prediction receives low variance whether or not it is
right.  The second is not specific to the probe: every sampling-based
estimator here, the reference included, reads the model's
uncertainty and not its truthfulness (App.~\ref{app:probe}).  The
volume-level coupling between the two drops is positive on both tasks
but its interval covers zero.

A fourth candidate fails as well: the reversal does not coincide with
the region where the implied diagonal becomes negative.  On the four
reversed conditions, the exact diagonal is never clamped (the clamped
fraction is $0.000$ on every image), whilst the conditions that do
clamp are the well-behaved ones, and within those the clamped fraction
associates with better ranking, not worse.

A principled reading of the row energy is available, and on natural
images it has now been tested.  The row energy
$\|J_{i:}\|^2$ measures the total sensitivity of output voxel $i$ to
perturbations of the input; where the denoiser amplifies its input, a
small mismatch between model and data becomes a large reconstruction
error, so voxels of high row energy would be voxels where errors
concentrate.  Read directly (Table~\ref{tab:rowenergy}), the row
energy does correlate positively with the error on the four reversed
conditions, $+0.079$ to $+0.135$, but only because it mirrors the
diagonal there; on the self-conditioned model it reads $-0.442$ where
the diagonal reads $+0.448$.  This matters beyond our own estimator,
because the row energy is what a probe that takes the variance of the
\emph{scores} rather than of the residual reads to second order, as
\citet{devita2025pixelwise} do: such a probe does not escape the
reversal; inverts the diagonal on every condition we measured it, and so ranks the error correctly on the conditions where the diagonal fails and incorrectly on those where it works.  Their estimate is moreover identified with
the Fisher information of the \emph{marginal} noising distribution,
an object with no $x_0$ in it, while the object a per-image map
requires is the conditional $C_t(x_0)$ of Sec.~\ref{sec:tpt-geom}.  The row energy is not what the probe reads on either:
the probe's partial given the row energy equals its partial given the
diagonal in every condition.  On lung CT, where the diagonal is only
available as a Hutchinson reference, the direct test remains open.

Nor is the ordering of Table~\ref{tab:conditions} a criterion.  That
table is ordered by $\rho(\nabla,e)$, which costs one forward pass, and
the ordering is visibly related to the column beside it: on the five
conditions where we first looked, the rank correlation between the two
was $1.000$; on all fourteen it is $0.640$, and on real images alone
$0.405$.  It does not predict the sign either; \texttt{inet/coco} has
a middling $\rho(\nabla,e)$ and the most negative correlation in the
table.  We report it as an observation about these fourteen conditions
and specifically not as a criterion, because the perfect ordering was
an artefact of the five points that produced the hypothesis.

\section{Qualitative comparison}
\label{app:qual}

\begin{figure}[t]
\centering
\includegraphics[width=0.49\linewidth]{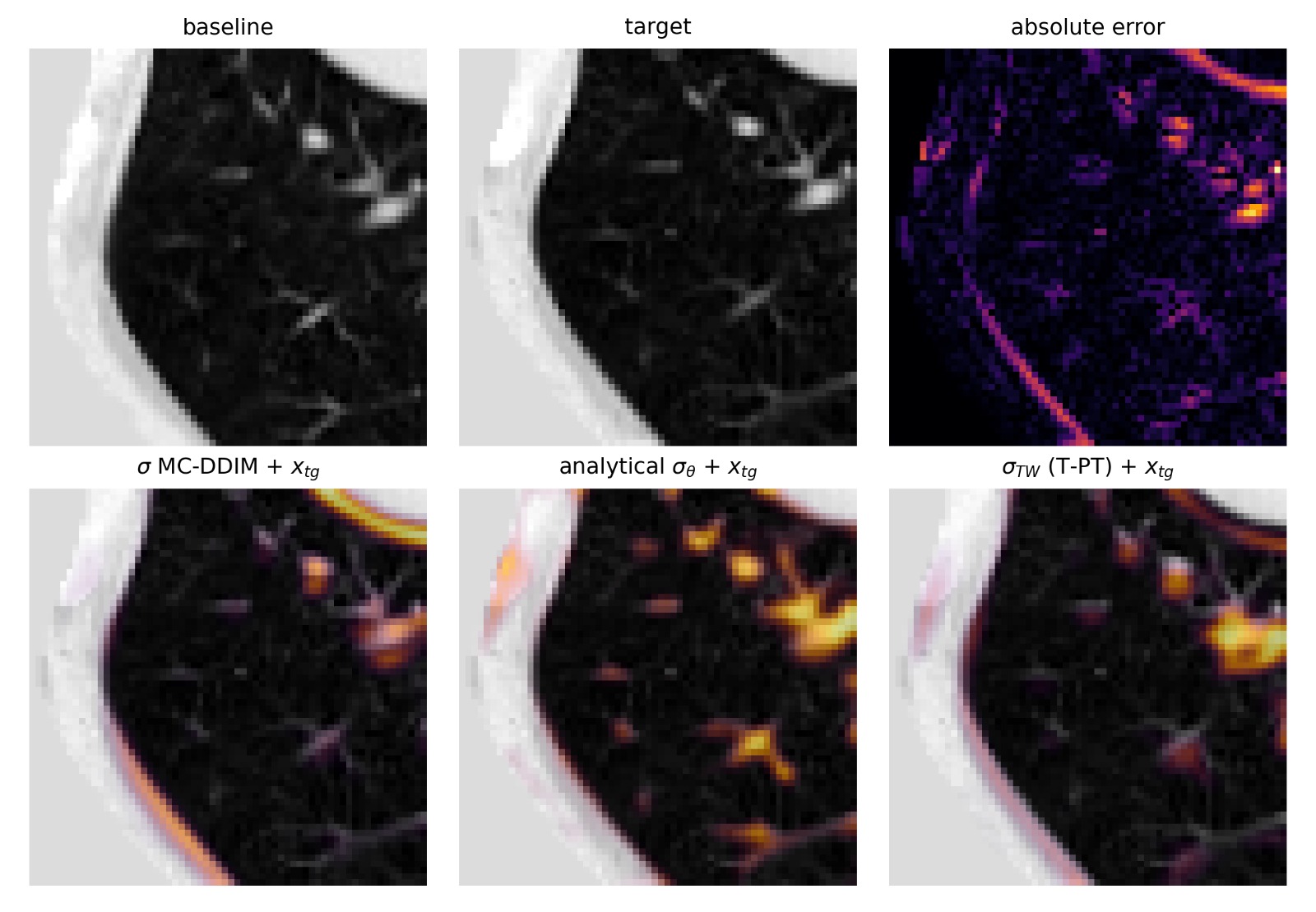}\hfill
\includegraphics[width=0.49\linewidth]{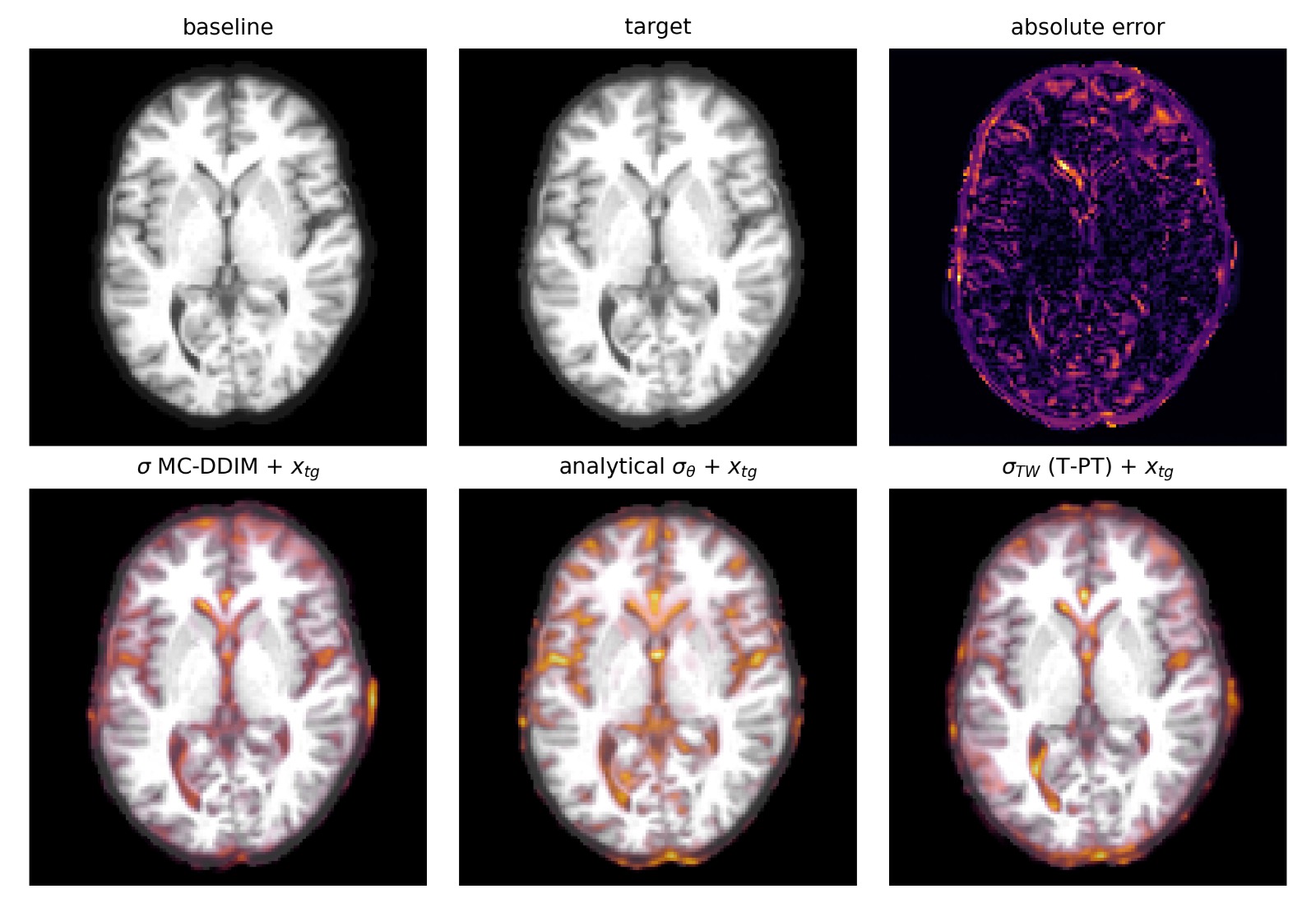}
\caption{Representative lung-CT sample (\textbf{left}) and OASIS-3
transverse slice (\textbf{right}).  Top row of each: baseline, target,
absolute error.  Bottom row, overlaid on the target: MC-DDIM $\sigma$,
the analytical $\sigma_\theta$, and \tpt{} $\sigma_{\text{TW}}$.  Each
overlay is normalised to its own quantiles, so the panels show each
map's spatial ordering and not comparable magnitudes.  On lung CT all
three concentrate on the nodule region and \tpt{} reproduces MC-DDIM's
spatial pattern (cross-agreement $\rho = 0.79$ on this volume) at
$13.4\times$ fewer network evaluations.  On brain MRI $\sigma_{\text{TW}}$
tracks the ventricular and cortical boundaries consistently with MC-DDIM
($\rho = 0.70$ here, $0.709$ over the cohort) and inside brain tissue it
leads the ensemble on five of the eight endpoints of
Table~\ref{tab:clinical}, the rest unresolved.  What
no panel shows, and Sec.~\ref{sec:limitations} reports, is that both
maps lose most of their information inside the region that actually
changed between the two visits.}
\label{fig:qual}
\end{figure}

Aggregate correlations hide what a clinician would actually look at.
Figure~\ref{fig:qual} shows one test case per
dataset: baseline, target, and absolute error on top, the three
$\sigma$ maps as alpha-blended overlays on the target anatomy below,
each normalised to its own quantiles.
The overlay is deliberate.  A $\sigma$ map is only useful when read
\emph{against} the anatomy it annotates; side-by-side heatmaps make
that harder than it needs to be.

On lung CT (Fig.~\ref{fig:qual}, left) all three estimators put their mass on
the growing nodule and the vessels next to it, and leave the static
parenchyma dark.  Without the labels, $\sigma_{\text{MC}}$ and
$\sigma_{\text{TW}}$ are hard to tell apart, which is what
cross-agreement $\rho = 0.79$ looks like on a single volume.  On brain
MRI (Fig.~\ref{fig:qual}, right) everything is more diffuse: uncertainty
follows the cortical ribbon and the ventricular boundaries, where
skull-stripping and atrophy both introduce ambiguity.  Again, the two
maps pick out the same structures.

Both samples have per-volume error correlations below their dataset
means ($\rho_{\text{TW}} = 0.601$ lung, $0.499$ brain), and per-volume
cross-agreements close to the dataset means ($0.79$ vs $0.815$;
$0.70$ vs $0.709$): representative volumes, not favourable ones.  We picked
them for anatomical clarity, not for favourable numbers.  Read the
figures for spatial behaviour; Table~\ref{tab:masked} carries the common-error-map comparison;
Table~\ref{tab:aggregate} reports the deployment variant.


\section{The deployment comparison, and paired tests}
\label{app:controlled}

Table~\ref{tab:masked} in the body scores every estimator against a
single common error map (Table~\ref{tab:fullmasked} below gives the
full per-region and selective-prediction rows), that of the $J=20$ ensemble mean, so that
differences reflect the uncertainty estimator alone.  That is the
scientific comparison.  It is not, however, the comparison a
practitioner faces: at deployment \tpt{} annotates the single chain
that produced its probe input, while MC-DDIM annotates its $J$-chain
mean, and a $J$-chain mean is the better prediction.  Table~\ref{tab:aggregate}
reports that operating comparison for completeness.  It conflates
estimator quality with prediction quality, which is why it is here and
not in the body.  \emph{Deployment} qualifies the scoring and nothing
else: each map is ranked against the error of the prediction it would
accompany rather than against a common reference no practitioner
holds; the conditioning is the deployable one on both tasks
(App.~\ref{app:defs}).

\begin{table}[h]
\centering\small
\begin{tabular}{llr}
\toprule
probe input $x_0$ (lung CT, $N{=}130$) & source & $\bar\rho(\sigma_{\text{TW}}, e)$ \\
\midrule
$[x_{\text{baseline}}, 0]$, zero placeholder in the residual channel & first implementation & $0.004$ \\
one DDIM reconstruction (deployment probe)   & Table~\ref{tab:aggregate} & $0.648$ \\
\bottomrule
\end{tabular}
\caption{The probe point is load-bearing (Sec.~\ref{sec:tpt-indist}):
the same estimator, at the same $t^\ast = 60$ and $K = 30$, scored
against the error of the chain each map accompanies --- the deployment
scoring defined above, not the common error map --- with only the
probe input changed.  Both rows are per-volume means, so the two are
read against each other; the participant-weighted figure for the
deployment probe is the $0.613\;[0.576, 0.650]$ of
Table~\ref{tab:aggregate}.  A
placeholder that no training pair ever visited yields a smooth,
non-negative map that carries no information about the error; any
model-supported reconstruction restores it.}
\label{tab:probeinput}
\end{table}

\begin{table}[h]
\centering
\small
\begin{tabular}{llrccl}
\toprule
Dataset & Estimator & Fwd. & $\bar\rho(\sigma, e)$ [95\% CI] & cross $\leftrightarrow$ MC & Requires \\
\midrule
PNG        & $\sigma_{\text{MC}}$ ($J{=}20$)
           & $1220$ & $0.642\;[0.596, 0.687]$ & $\equiv 1$    & --- \\
$N{=}130$  & $\sigma_\theta$ (analytical, $\shat{=}15$)
           & $61$   & $0.597\;[0.551, 0.641]$ & --- & hybrid loss \\
           & $\sigma_{\text{TW}}$ (\tpt)
           & $\mathbf{91}$ & $0.613\;[0.576, 0.650]$ & $0.790\;[0.765, 0.812]$ & --- \\
\midrule
OASIS      & $\sigma_{\text{MC}}$ ($J{=}20$)
           & $1220$ & $\mathbf{0.552}\;[0.538, 0.564]$ & $\equiv 1$ & --- \\
$N{=}211$  & $\sigma_\theta$ (analytical, $\shat{=}15$)
           & $61$   & $\mathbf{0.551}\;[0.539, 0.562]$ & --- & hybrid loss \\
           & $\sigma_{\text{TW}}$ (\tpt)
           & $\mathbf{76}$ & $0.520\;[0.511, 0.529]$ & $0.709\;[0.705, 0.712]$ & --- \\
\bottomrule
\end{tabular}
\caption{Predictive discrimination in the deployment setting: each
estimator is scored against the error of the prediction it would
actually accompany, so \tpt{} is measured against a single chain and
MC-DDIM against its $J$-chain mean.  ``Fwd.''\ is network evaluations
per volume; CI is confidence interval. $\sigma_\theta$ accumulates along its own single chain,
so its $61$ are that chain itself, with no evaluations beyond it.
$\sigma_{\text{MC}}$ is the reference, so its
cross-agreement with itself is $1$ by construction; entries are
participant-weighted with participant-level bootstraps on both tasks.  The two $\sigma_\theta$ rows are scored
against the common error map (their cross-agreement with MC-DDIM was
not computed); on lung CT the analytical baseline sits below both
estimators, and \tpt{} needs no variance head to run, at
$13.4\times$ fewer network evaluations than the reference.  On brain
MRI, over the whole volume, MC-DDIM is ahead
of \tpt{} by $3.2$ Spearman points in this self-scored setting (about
four against the common error map); Table~\ref{tab:masked} shows the
ordering reversing inside tissue, which is where the background stops
carrying the correlation.  The tissue rows of both analytical
baselines are in Table~\ref{tab:analytic-oasis}.}
\label{tab:aggregate}
\end{table}

\paragraph{Paired tests.}
Overlapping intervals do not establish that two estimators differ, and
non-overlapping ones do not establish by how much.  Since every
estimator is scored on the same volumes, we bootstrap the per-volume
difference directly, preserving the pairing ($2 \times 10^4$
resamples), and we resample participants rather than volumes as
specified in Sec.~\ref{sec:exp-setup} and validated below on
simulated data.

\begin{table}[h]
\centering
\footnotesize
\setlength{\tabcolsep}{2.5pt}
\begin{tabular}{lccc}
\toprule
Metric & participant-wtd.\ $\bar\Delta$ [CI] & volume-wtd.\ $\bar\Delta$ [CI] & \tpt{} wins \\
\midrule
\multicolumn{4}{l}{\emph{\textbf{Lung CT} (PNG, $N{=}130$, $G{=}47$)}} \\
$\rho$, full volume      & $-0.0215\;[-0.0373, -0.0070]$ & $-0.0242\;[-0.0376, -0.0104]$ & $36/130$ \\
$\rho$, tissue           & $-0.0215\;[-0.0368, -0.0069]$ & $-0.0242\;[-0.0375, -0.0103]$ & $36/130$ \\
$\rho$, edges               & $-0.1465\;[-0.1865, -0.1096]$ & $-0.1760\;[-0.2102, -0.1295]$ & $17/130$ \\
$\rho$, interior            & $-0.0234\;[-0.0435, -0.0048]$ & $-0.0210\;[-0.0413, -0.0062]$ & $50/130$ \\
$\rho$, regions of change   & $-0.1337\;[-0.1821, -0.0860]$ & $-0.1656\;[-0.1945, -0.1171]$ & $30/130$ \\
$\rho$, partial, $|\nabla u_0|$ ctl.\ & $-0.0595\;[-0.0900, -0.0307]$ & $-0.0696\;[-0.0930, -0.0417]$ & $33/130$ \\
AUROC, worst $5\%$        & $-0.0274\;[-0.0388, -0.0167]$ & $-0.0376\;[-0.0505, -0.0199]$ & $28/130$ \\
AUSE (lower better)       & $+0.0140\;[+0.0041, +0.0239]$ & $+0.0188\;[+0.0076, +0.0275]$ & $41/130$ \\
AURC (lower better)       & $+0.0013\;[+0.0003, +0.0023]$ & $+0.0017\;[+0.0006, +0.0025]$ & $41/130$ \\
\midrule
\multicolumn{4}{l}{\emph{\textbf{Brain MRI} (OASIS, $N{=}211$, $G{=}63$)}} \\
$\rho$, full volume      & $-0.0380\;[-0.0415, -0.0345]$ & $-0.0412\;[-0.0448, -0.0372]$ & $1/211$ \\
$\rho$, tissue           & $+0.0165\;[+0.0057, +0.0282]$ & $+0.0120\;[+0.0032, +0.0215]$ & $117/211$ \\
$\rho$, edges               & $+0.0027\;[-0.0047, +0.0102]$ & $+0.0030\;[-0.0053, +0.0115]$ & $103/211$ \\
$\rho$, interior            & $+0.0221\;[+0.0089, +0.0369]$ & $+0.0164\;[+0.0064, +0.0274]$ & $120/211$ \\
$\rho$, regions of change   & $-0.0025\;[-0.0132, +0.0082]$ & $-0.0049\;[-0.0122, +0.0029]$ & $99/211$ \\
$\rho$, partial, $|\nabla u_0|$ ctl.\ & $+0.0212\;[+0.0090, +0.0350]$ & $+0.0167\;[+0.0070, +0.0270]$ & $123/211$ \\
AUROC, worst $5\%$        & $-0.0014\;[-0.0065, +0.0041]$ & $-0.0004\;[-0.0062, +0.0058]$ & $82/211$ \\
AUSE (lower better)       & $-0.0108\;[-0.0188, -0.0036]$ & $-0.0084\;[-0.0151, -0.0024]$ & $106/211$ \\
AURC (lower better)       & $-0.0010\;[-0.0016, -0.0004]$ & $-0.0009\;[-0.0015, -0.0004]$ & $106/211$ \\
\bottomrule
\end{tabular}
\caption{Paired participant-level bootstrap of
$\sigma_{\text{TW}} - \sigma_{\text{MC}}$ against the common error
map, on both tasks: the $130$ lung-CT pairs from $47$ participants and
all $211$ OASIS volumes from the $63$ participants of the published
test split.  $\Delta$ is \tpt{} minus MC-DDIM on every row, so on
AUSE and AURC, a positive value favours MC-DDIM.  The
participant-weighted column averages the
difference within a participant before averaging across participants,
so a participant contributing $21$ of the $211$ pairs counts once
rather than $21$ times; the volume-weighted column averages over
volumes directly.  Both bootstraps resample participants.  The two
agree in sign on every row of both tasks.  The last column counts the volumes on which \tpt{} is better,
with the direction of each metric taken into account.  Tissue Spearman is the pre-designated
primary endpoint and the gradient-controlled row its robustness check
(App.~\ref{app:defs}); the remaining rows are secondary, unadjusted
for multiplicity, and any one of them alone should be read as
exploratory. CI is the confidence interval.}
\label{tab:paired-oasis}
\end{table}

\begin{table}[t]
\centering
\small
\begin{tabular}{lcc c cc}
\toprule
& \multicolumn{2}{c}{PNG ($N{=}130$)} & & \multicolumn{2}{c}{OASIS ($N{=}211$)} \\
\cmidrule(lr){2-3}\cmidrule(lr){5-6}
& $\sigma_{\text{TW}}$ & $\sigma_{\text{MC}}$ & &
  $\sigma_{\text{TW}}$ & $\sigma_{\text{MC}}$ \\
\midrule
$\rho$, full volume         & $0.669$ & $\mathbf{0.693}$ & & $0.518$ & $\mathbf{0.559}$ \\
$\rho$, tissue              & $0.670$ & $\mathbf{0.694}$ & & $\mathbf{0.349}$ & $0.337$ \\
$\rho$, edges                  & $0.352$ & $\mathbf{0.528}$ & & $0.248$ & $0.245$ \\
$\rho$, interior               & $0.612$ & $\mathbf{0.633}$ & & $\mathbf{0.334}$ & $0.318$ \\
$\rho$, regions of true change & $0.100$ & $\mathbf{0.265}$ & & $0.128$ & $0.133$ \\
\midrule
$\rho$, partial, $|\nabla u_0|$ controlled
                             & $0.339$ & $\mathbf{0.409}$ & & $\mathbf{0.290}$ & $0.274$ \\
AUROC, worst $5\%$ voxels    & $0.884$ & $\mathbf{0.922}$ & & $0.782$ & $0.782$ \\
AUSE (lower is better)       & $0.152$ & $\mathbf{0.133}$ & & $\mathbf{0.306}$ & $0.315$ \\
AURC (lower is better)       & $0.030$ & $\mathbf{0.029}$ & & $\mathbf{0.053}$ & $0.054$ \\
\bottomrule
\end{tabular}
\caption{Full per-region and selective-prediction results against
the common error map of the $J=20$ ensemble mean, expanding the
compact Table~\ref{tab:masked} of the body.  OASIS entries are
per-volume means, shown for scale, following the convention of
Table~\ref{tab:masked}; the pre-declared participant-weighted
differences and their intervals for OASIS are in
Table~\ref{tab:paired-oasis}; the PNG paired contrasts are the
$\Delta$ column of Table~\ref{tab:masked}.  Bold marks the better of $\sigma_{\text{TW}}$ and
$\sigma_{\text{MC}}$ where the $95\%$ paired interval excludes zero
(participant-level resampling on both tasks);
intervals are in Table~\ref{tab:paired-oasis}.  ``Tissue'' is the foreground mask.  On lung
CT it covers $99.5\%$ of the volume; the scan is already cropped to
the thorax, so foreground and full volume coincide, and the row is a
consistency check rather than a result.  On brain MRI it covers
$24\%$, and there the two rows differ sharply.
The $\sigma_\theta$ columns for both tasks are in
Table~\ref{tab:analytic-oasis}.}
\label{tab:fullmasked}
\end{table}

\paragraph{The analytical baseline.}
$\sigma_\theta$ is produced by a single reverse chain: the sampler
accumulates the variance head along the chain with the delayed-start
mask at $\shat = 15$, so one chain per volume (seed $42$, the first
reference seed) is the whole computation.  It is scored against the
same common error map as every other estimator
(Table~\ref{tab:analytic-oasis}); Table~\ref{tab:aggregate} gives
its full-volume figure next to the deployment-scored rows.  On lung
CT the ordering is simple: the variance head is below both estimators
on every endpoint, \tpt{} by $+0.024\;[+0.015, +0.034]$ in tissue
Spearman and the reference by $+0.046$.  On brain MRI, it is not.  The result is the paper's thesis in
one row.  Over the full volume the propagated variance is
not distinguishable from the $J{=}20$ ensemble ($0.553$ vs $0.559$;
paired $\Delta = +0.001\;[-0.004, +0.006]$) while sitting \emph{above}
\tpt{} ($0.518$).  Inside tissue it is last by a wide margin, and the
gradient control, AUROC and AUSE agree.  Its median voxel sits at $3.0\times10^{-3}$, of the same order as the
schedule's own accumulated injection at the mask, with the anatomy
modulating scarcely a factor of two above it. Those magnitudes are a
description and not an explanation of the ranking: the injected term
enters the recursion as a schedule scalar added identically at every
voxel, so it cannot reorder them: no constant additive term
across voxels can move a rank correlation
(Prop.~\ref{prop:propagation-ranking}, App.~\ref{app:propagation-ranking}).  What the tissue restriction removes is the agreement
that background supplies for free, where map and error are both near
zero; whether what remains is dominated by Monte-Carlo noise is a
question a split-half of the seeds inside tissue would answer, and we
have not run it.  At $61$ network evaluations against \tpt{}'s $76$,
the fifteen additional probe forwards buy a participant-weighted
$+0.050$ tissue-Spearman contrast.

\begin{table}[h]
\centering
\footnotesize
\setlength{\tabcolsep}{4pt}
\setlength{\tabcolsep}{3pt}
\begin{tabular}{lccc}
\toprule
 & $\sigma_\theta$ (analyt.) & $\sigma_{\text{TW}}$ & $\sigma_{\text{MC}}$ \\
\midrule
\multicolumn{4}{l}{\emph{\textbf{Lung CT} (PNG, $N{=}130$, $G{=}47$)}} \\
$\rho$, full volume  & $0.643$ & $0.669$ & $\mathbf{0.693}$ \\
$\rho$, tissue       & $0.644$ & $0.670$ & $\mathbf{0.694}$ \\
$\rho$, partial, $|\nabla u_0|$ ctl. & $0.295$ & $0.339$ & $\mathbf{0.409}$ \\
Network evaluations / volume & $61$ & $91$ & $1220$ \\
\midrule
\multicolumn{4}{l}{\emph{Paired contrasts against $\sigma_\theta$ (participant-weighted, $B = 2\times10^4$):}} \\
$\Delta\rho$, tissue & --- & $+0.0241\;[+0.0148, +0.0336]$ & $+0.0456\;[+0.0307, +0.0607]$ \\
$\Delta\rho$, partial (grad.\ rm.) & --- & $+0.0574\;[+0.0405, +0.0747]$ & $+0.1169\;[+0.0908, +0.1453]$ \\
$\Delta\rho$, full volume & --- & $+0.0238\;[+0.0142, +0.0333]$ & $+0.0453\;[+0.0306, +0.0610]$ \\
$\Delta$AUROC, worst $5\%$ & --- & $+0.0501\;[+0.0446, +0.0557]$ & $+0.0775\;[+0.0639, +0.0923]$ \\
$\Delta$AUSE (lower better) & --- & $-0.0260\;[-0.0317, -0.0203]$ & $-0.0400\;[-0.0518, -0.0283]$ \\
\midrule
\midrule
\multicolumn{4}{l}{\emph{\textbf{Brain MRI} (OASIS, $N{=}211$, $G{=}63$)}} \\
$\rho$, full volume  & $0.553$ & $0.518$ & $\mathbf{0.559}$ \\
$\rho$, tissue       & $0.299$ & $\mathbf{0.349}$ & $0.337$ \\
$\rho$, partial, $|\nabla u_0|$ ctl. & $0.253$ & $\mathbf{0.290}$ & $0.274$ \\
Network evaluations / volume & $61$ & $76$ & $1220$ \\
\midrule
\multicolumn{4}{l}{\emph{Paired contrasts against $\sigma_\theta$ (participant-weighted, $B = 2\times10^4$):}} \\
$\Delta\rho$, tissue & --- & $+0.0496\;[+0.0370, +0.0627]$ & $+0.0331\;[+0.0234, +0.0433]$ \\
$\Delta\rho$, partial (grad.\ rm.) & --- & $+0.0355\;[+0.0229, +0.0486]$ & $+0.0142\;[+0.0060, +0.0224]$ \\
$\Delta\rho$, full volume & --- & $-0.0372\;[-0.0401, -0.0341]$ & $+0.0008\;[-0.0040, +0.0058]$ \\
$\Delta$AUROC, worst $5\%$ & --- & $+0.0274\;[+0.0223, +0.0325]$ & $+0.0288\;[+0.0229, +0.0350]$ \\
$\Delta$AUSE (lower better) & --- & $-0.0391\;[-0.0480, -0.0307]$ & $-0.0283\;[-0.0355, -0.0215]$ \\
\bottomrule
\end{tabular}
\caption{The analytical variance-propagation baseline on both
tasks, from one seed-$42$ chain per volume and scored against the
common $J{=}20$ ensemble-mean error map.  Estimator rows are
per-volume means, following the convention of Table~\ref{tab:masked};
contrast rows are the participant-weighted paired difference (column
estimator minus $\sigma_\theta$), with brackets giving 95\% confidence
intervals (CI), where positive favours the column
estimator except for AUSE, where negative does.
AUROC is reported
for the worst-$5\%$-error voxels.  On lung CT, both
estimators beat the variance head on every endpoint, the reference by
more.  On brain MRI, inside tissue both estimators beat it and \tpt{}
by more; over the whole volume neither does, the \tpt{} contrast
being negative and the MC one crossing zero.  In the OASIS regions of
change (not shown) no contrast against $\sigma_\theta$ excludes zero
in either direction.}
\label{tab:analytic-oasis}
\end{table}

\paragraph{Deployment without a reference ensemble.}
The paired comparisons of Tables~\ref{tab:masked}
and~\ref{tab:paired-oasis} score $\sigma_{\text{TW}}$ against the
common ensemble-mean error map, which a deployed system does not have.  The
honest deployment quantity scores the probe against the error of its
own reconstruction.  Doing so on the same $211$ volumes moves the
whole-volume Spearman \emph{up}, from $0.518$ to $0.523$
(per-volume means; the participant-weighted figure of
Table~\ref{tab:aggregate} is $0.520$; $+0.0063\;[+0.0018, +0.0110]$), and the worst-$5\%$ AUROC up from
$0.782$ to $0.793$ ($+0.0092\;[+0.0061, +0.0125]$), because the
probe and the error map now share a reconstruction and agree on the
background.  Inside tissue it moves \emph{down}, from $0.349$ to
$0.343$ ($-0.0092\;[-0.0137, -0.0045]$), and it moves down by a
similar amount in the interior ($-0.0101$), in the regions of change
($-0.0138$) and on the gradient-controlled endpoint ($-0.0089$).  The
primary endpoint therefore costs about one Spearman point when the
reference is removed, and the whole-volume figure is the one that
should not be trusted: it improves for the same reason it was
uninformative to begin with.  The probe-seed rerun makes the same
point from the other side: changing the probe seed from $42$ to $777$
moves the self-scored whole-volume figure by $+0.08$ and the tissue
figure by $0.006$ (Table~\ref{tab:seedsens}; the number we say not
to trust is also the only one sensitive to the probe seed.

\paragraph{Probe-seed sensitivity, and the no-residual baseline.}
The deployment probe chain uses seed $42$, which is also one of the
twenty reference seeds ($42$--$61$).  To rule out any dependence
between the probe and the reference, we reran the full masked pass
with a probe seed outside the reference set, $777$ on brain MRI,
$123$ on lung CT, on both datasets
(Table~\ref{tab:seedsens}).  Every primary contrast moves by at most
$0.002$ on brain MRI and $0.006$ on lung CT, with near-identical
intervals, and cross-agreement is unchanged ($0.709$ under both seeds
on OASIS, $0.790$ vs $0.784$ on PNG), so
the shared seed neither drives the tissue advantage nor inflates the
agreement with the reference.  The same rerun carries the
no-residual control requested by the construction itself: the
per-voxel variance of the raw noise predictions over the same $K$
perturbations, identical probes, no subtraction of the injected noise,
which is the statistic of \citet{devita2025pixelwise} evaluated at our
probe point and budget rather than in their pipeline.  It correlates
with the common error map at $-0.010$ (OASIS
tissue) and $-0.019$ (PNG tissue), against \tpt{}'s $0.353$ and
$0.618$; the paired gaps are $+0.363\;[+0.345, +0.381]$ and
$+0.637\;[+0.592, +0.681]$.  (The standard-deviation map of $\hat x_0$ over the same
probes needs no experiment: $\hat x_0 = x_0 + \sigma_t/\!\sqrt{\bar\alpha_t}\,(\xi - \epsilon_\theta)$,
so it is the \tpt{} map times a schedule constant, identical in
rank.)  What carries the signal is the subtraction of the known
noise, the Tweedie residual, not the local perturbations.

\begin{table}[h]
\centering
\footnotesize
\setlength{\tabcolsep}{4pt}
\begin{tabular}{llcc}
\toprule
 & & probe seed $42$ & alternative probe seed \\
\midrule
OASIS & $\Delta\rho$, tissue   & $+0.0165\;[+0.0056, +0.0282]$ & $+0.0149\;[+0.0038, +0.0268]$ \\
      & $\Delta\rho$, partial (grad.\ rm.) & $+0.0212\;[+0.0088, +0.0348]$ & $+0.0193\;[+0.0068, +0.0332]$ \\
      & $\Delta\rho$, full volume   & $-0.0380\;[-0.0415, -0.0346]$ & $-0.0374\;[-0.0408, -0.0339]$ \\
      & $\rho(\sigma_{\text{TW}}, \sigma_{\text{MC}})$ & $0.709$ & $0.709$ \\
\midrule
PNG   & $\Delta\rho$, tissue   & $-0.0215\;[-0.0372, -0.0069]$ & $-0.0256\;[-0.0428, -0.0093]$ \\
      & $\Delta\rho$, partial (grad.\ rm.) & $-0.0595\;[-0.0896, -0.0306]$ & $-0.0650\;[-0.0960, -0.0353]$ \\
      & $\Delta\rho$, full volume   & $-0.0215\;[-0.0369, -0.0072]$ & $-0.0255\;[-0.0433, -0.0095]$ \\
      & $\rho(\sigma_{\text{TW}}, \sigma_{\text{MC}})$ & $0.790$ & $0.784$ \\
\bottomrule
\end{tabular}
\caption{Probe-seed sensitivity: the masked pass run with the
deployment probe seed ($42$) and with a probe seed outside the
reference seed set ($42$--$61$): seed $777$ on OASIS, seed $123$ on
lung CT, everything else held fixed, on the deployment conditioning
and against the same references as Table~\ref{tab:masked}.  $\Delta\rho$ is the paired difference, \tpt{} minus MC-DDIM, of the
per-volume Spearman correlation with the common error map; the last
row of each block is the Spearman correlation between the two
$\sigma$ maps themselves.  The table
is read for the seed-to-seed movement only; contrasts are
participant-weighted with participant-level intervals ($B = 10^4$) on
both tasks.  Every primary contrast moves by at most $0.006$.}
\label{tab:seedsens}
\end{table}

\noindent
The pattern is the one Sec.~\ref{sec:exp-masked} describes, now with
the resampling unit corrected.  Over the whole volume, MC-DDIM is
significantly better; inside tissue the ordering reverses, in the
interior it reverses by more, and the reversal survives partialling
out the baseline intensity gradient, which is the control that would
kill a pure anatomy tracer.  Edges, the change region and AUROC on the
worst $5\%$ of voxels are the contrasts we cannot resolve.

We report the weighting we do because the split is unbalanced, and it
is worth being explicit that the imbalance was working \emph{against}
the effect rather than creating it: the participant contributing $21$
of the $211$ pairs carries a tenth of a volume-weighted estimate, and
removing that leverage strengthens the primary and robustness
contrasts rather than
weakening it.  A reader who prefers volume weighting can read the last
column and reach the same inferential conclusion on all nine rows:
the two weightings agree in sign on every one of them, and where they
differ in magnitude, both intervals cross zero.

\begin{table}[h]
\centering
\small
\begin{tabular}{llcc}
\toprule
Checkpoint & Contrast $\Delta$ & $\bar\Delta$ [95\% CI] & $P(\Delta > 0)$ \\
\midrule
PNG hybrid  & $\sigma_{\text{TW}} - \sigma_{\text{MC}}$     & $-0.022\;[-0.037, -0.007]$ & $<10^{-3}$ \\
PNG no-head & $\sigma_{\text{TW}} - \sigma_{\text{MC}}$     & $-0.020\;[-0.040, -0.001]$ & $0.02$ \\
\bottomrule
\end{tabular}
\caption{Paired bootstrap on the two lung-CT checkpoints against the
error map of each checkpoint's own $J{=}20$ ensemble, resampling the
$47$ participants ($B = 2\times10^4$); the conditioning of each
experiment is given in App.~\ref{app:defs} and
App.~\ref{app:nohead}.  MC-DDIM is ahead of \tpt{} by about two
Spearman points on both checkpoints; on the one trained without a
variance head, the interval only just excludes zero.  The
sample-level bootstrap used in an earlier draft gave
$-0.003\;[-0.016, +0.009]$ for that checkpoint, which is the
coverage failure App.~\ref{app:controlled} documents.  The OASIS
contrasts are in Table~\ref{tab:paired-oasis}.}
\label{tab:paired}
\end{table}

\paragraph{Why the resampling unit matters.}
Resampling volumes rather than participants is not conservative in one
direction; it is simply wrong, and it is wrong in the direction that
manufactures significance.  We checked this on simulated data matched
to our cluster sizes, $20$ participants with $1$ to $11$ volumes each,
a true between-estimator difference of $0.05$ carrying a
participant-level component, and repeated it $300$ times.  A
volume-level bootstrap covered the true effect in $57\%$ of the
replicates against a nominal $95\%$; the participant-level bootstrap
covered $95\%$ ($94.7\%$ measured).  On a single replicate, the two
disagree qualitatively: the volume-level interval
$[+0.0147, +0.0358]$ excludes zero, while the participant-level
interval on the same data, $[-0.0056, +0.0546]$, contains it.  A
reader checking our contrasts should confirm that the interval
excludes zero under the participant-level resampling, which is the one
we report.

\noindent
Two readings follow, and they pull in opposite directions.  \tpt{}
does \emph{not} match MC-DDIM on lung CT: it is significantly worse,
by a margin the bootstrap resolves cleanly, and any claim of parity
there would be unsupported.  On brain MRI it is behind MC-DDIM over the whole
volume and ahead of it inside tissue (Table~\ref{tab:paired-oasis}).
The defensible summary is that \tpt{} retains most of MC-DDIM's
spatial discrimination at a fraction of the cost, and that which wins
depends on the dataset \emph{and} on where in the volume one looks.

\paragraph{Head-free compatibility, with the confound removed.}
Expressed as a fraction of each model's own Monte-Carlo reference,
\tpt{} recovers $95.6\%$ $[93.9, 97.9]$ on the hybrid checkpoint
and $99.5\%$ $[96.0, 102.4]$ on the one trained without a variance
head, resampling the $47$ participants.
The difference between the two ratios is
$+3.9$ points $[+0.4, +5.5]$ in favour of the model trained without a
variance head.
\tpt{} therefore tracks its reference \emph{at least as closely}
without the variance head as with it, which is the strongest form the
training-invariance claim could take, and we flag that we reached it
only after removing a confound: scoring \tpt{} against its own
single-chain error while scoring MC-DDIM against the ensemble-mean
error produces an apparent $86.6\%$ against $92.6\%$ and reverses the
conclusion.

The mechanism is worth stating plainly, because it is not that the
estimator improves.  Removing the variance head produces a noisier
sampler: mean absolute error rises from $0.064$ to $0.069$ for the
ensemble mean and from $0.064$ to $0.086$ for a single chain, and
$\sigma_{\text{MC}}$ is $2.7\times$ larger at the same $J$.  Both
estimators sit lower on that checkpoint: the ensemble at $0.636$
and the probe at $0.633$ against the ensemble's own reference, and the
paired gap between them is $-0.020\;[-0.040, -0.001]$.  What the experiment establishes is that no part of \tpt{} depends
on the hybrid objective; it neither needs $\sigma_\theta$ to exist
nor pays a penalty when it does not.

\section{What a larger Monte-Carlo ensemble buys}
\label{app:jconv}

Sub-sampling the complete $J = 20$ ensembles, at $J' \in \{2,5,10,20\}$
on lung CT and $J' \in \{2,3,5,8,10,15,20\}$ on OASIS,
and rescoring gives two quantities: discrimination against the error
map, and the split-half agreement obtained by splitting the $J'$
chains into two disjoint halves and correlating the two $\sigma$ maps.
The first is how informative the estimate is, the second how
reproducible.

\begin{table}[h]
\centering
\small
\begin{tabular}{r cc cccc}
\toprule
 & \multicolumn{2}{c}{PNG ($N{=}130$)} & \multicolumn{4}{c}{OASIS ($N{=}211$)} \\
\cmidrule(lr){2-3}\cmidrule(lr){4-7}
$J'$ & $\bar\rho(\sigma_{J'}, e)$ & split-half & $\bar\rho(\sigma_{J'}, e)$ & agree w/ $J{=}20$ & split-half & $\bar\rho_{\text{tissue}}(\sigma_{J'}, e)$ \\
\midrule
$2$  & $0.529$ ($77\%$) & ---     & $0.485$ ($91\%$)  & $0.668$ & ---     & $0.139$ ($43\%$) \\
$3$  & ---              & ---     & $0.521$ ($98\%$)  & $0.757$ & ---     & $0.195$ ($60\%$) \\
$5$  & $0.653$ ($95\%$) & $0.496$ & $0.529$ ($99\%$)  & $0.827$ & $0.535$ & $0.248$ ($77\%$) \\
$8$  & ---              & ---     & $0.531$ ($100\%$) & $0.887$ & $0.665$ & $0.284$ ($88\%$) \\
$10$ & $0.676$ ($98\%$) & $0.768$ & $0.532$ ($100\%$) & $0.915$ & $0.684$ & $0.296$ ($91\%$) \\
$15$ & ---              & ---     & $0.532$ ($100\%$) & $0.964$ & $0.712$ & $0.314$ ($97\%$) \\
$20$ & $0.687$ ($100\%$)& $0.863$ & $0.534$ ($100\%$) & $1.000$ & $0.740$ & $0.323$ ($100\%$) \\
\bottomrule
\end{tabular}
\caption{Percentages are relative to $J' = 20$.  Each entry averages
$20$ subsets of size $J'$ drawn at random without replacement, so the
curve does not depend on the order of the seeds.  Entries are
\emph{volume-weighted} means computed by the standalone convergence
probe directly in prediction space, so its absolute values are not
comparable with the scored pipeline of
Tables~\ref{tab:masked}--\ref{tab:fullmasked}; only the percentages,
which are internal to this table, are meaningful here.
The last column is the finding: full-volume discrimination saturates
almost immediately ($J'=3$ recovers $98\%$), but that saturation is
carried by background.  \emph{Inside tissue}, the primary endpoint,
the same reference recovers only $60\%$ at $J'=3$, $91\%$ at $J'=10$,
and is still climbing at $J'=20$.  Among the tested configurations
$J' \le 20$, no smaller ensemble reaches the $J{=}20$ tissue
discrimination, which is itself still increasing.}
\label{tab:jconv}
\end{table}

\paragraph{Fifty chains: the scores are insensitive to the target,
the reference's own spread is not.}
Table~\ref{tab:jconv} extrapolates; this rescoring measures.  We
generated thirty further reverse chains per volume, seeds $62$--$91$,
under the deployment conditioning, and rescored against the resulting
$J{=}50$ ensemble: every lung-CT volume, and the first thirty brain
volumes, which is what the budget allowed.  Two things separate, and
only the second was in doubt.  The mean absolute error of the
ensemble mean is unchanged ($0.0734 \to 0.0732$ on lung CT,
$0.0260 \to 0.0261$ on brain MRI), and the statement the tables
actually need holds directly: the probe's map is fixed by
construction, and scoring that one fixed map against $e_{20}$ and
against $e_{50}$ gives $0.670 \to 0.671$ on lung CT and $0.356 \to
0.356$ inside brain tissue (Table~\ref{tab:j50}).  Equal
correlations with two targets do not make the two targets equal, so
this is not a convergence claim about the error map; it is the
narrower one the body relies on, that the discrimination it reports
for \tpt{} is unchanged when the twenty-chain error map is replaced
by the fifty-chain one.  What is
not converged at $J{=}20$ is MC-DDIM's own $\sigma$, which improves by
$+0.008$ (lung) and $+0.027$ (brain, inside tissue) when it is
estimated from fifty chains rather than twenty.  The probe's map is
unchanged by construction, so the whole movement in
Table~\ref{tab:j50} is the reference improving.

\begin{table}[h]
\centering
\footnotesize
\setlength{\tabcolsep}{3pt}
\begin{tabular}{llcccc}
\toprule
 & & \multicolumn{2}{c}{$J{=}20$} & \multicolumn{2}{c}{$J{=}50$} \\
\cmidrule(lr){3-4}\cmidrule(lr){5-6}
 & & \tpt{} & MC & \tpt{} & MC \\
\midrule
Lung CT & $\rho$, tissue & $0.670$ & $0.694$ & $0.671$ & $0.702$ \\
$N{=}130$, $G{=}47$ & $\Delta\rho$ & \multicolumn{2}{c}{$-0.022\,[-0.037,-0.007]$} & \multicolumn{2}{c}{$-0.025\,[-0.039,-0.012]$} \\
 & $\Delta\rho$, interior & \multicolumn{2}{c}{$-0.023\,[-0.043,-0.005]$} & \multicolumn{2}{c}{$-0.026\,[-0.044,-0.009]$} \\
 & $\Delta\rho$, partial (grad.\ rm.) & \multicolumn{2}{c}{$-0.060\,[-0.090,-0.031]$} & \multicolumn{2}{c}{$-0.057\,[-0.087,-0.029]$} \\
 & $\Delta$AUROC & \multicolumn{2}{c}{$-0.027\,[-0.039,-0.017]$} & \multicolumn{2}{c}{$-0.033\,[-0.044,-0.022]$} \\
\midrule
Brain MRI & $\rho$, tissue & $0.356$ & $0.318$ & $0.356$ & $0.345$ \\
$N{=}30$, $G{=}5$ & $\Delta\rho$ & \multicolumn{2}{c}{$+0.038$} & \multicolumn{2}{c}{$+0.010$} \\
 & $\Delta\rho$, interior & \multicolumn{2}{c}{$+0.046$} & \multicolumn{2}{c}{$+0.015$} \\
 & $\Delta\rho$, partial (grad.\ rm.) & \multicolumn{2}{c}{$+0.043$} & \multicolumn{2}{c}{$+0.012$} \\
 & $\Delta$AUROC & \multicolumn{2}{c}{$+0.020$} & \multicolumn{2}{c}{$+0.000$} \\
 & volumes with \tpt{} ahead & \multicolumn{2}{c}{$24/30$} & \multicolumn{2}{c}{$12/30$} \\
\bottomrule
\end{tabular}
\caption{Rescoring against a $J{=}50$ reference.  $\Delta\rho$ and $\Delta$AUROC are \tpt{}
minus MC-DDIM, as everywhere else; the lung-CT $J{=}20$ intervals are
those of Table~\ref{tab:masked}.  The lung-CT rows are the full
cohort with participant-level intervals; the brain rows are the first
thirty volumes, which come from five participants, so they are
reported as a convergence check and carry no intervals --- resampling
five clusters would not mean anything.  The brain $J{=}20$ column is
recomputed on those same thirty volumes, so the two columns are
paired; it is larger there ($+0.038$) than over the full cohort
($+0.016$, Table~\ref{tab:masked}), which is why only the within-table
comparison is read.  At fifty chains the ensemble costs $3{,}050$
network evaluations per volume against the probe's $76$, and inside
brain tissue the two are then close on every endpoint, AUROC
included.  The last row counts volumes on which \tpt{} has the
higher tissue $\rho$.}
\label{tab:j50}
\end{table}

\section{Sensitivity to the probe timestep}
\label{app:tstar}

The probe timestep is the only continuous knob in \tpt.  The sweep
below reruns the full lung-CT masked pass of Table~\ref{tab:masked}
at four values of $t^\ast$, self-chain probe, $K = 30$, deployment
conditioning, everything else fixed; the $t^\ast = 60$ column is
Table~\ref{tab:masked} itself.

\begin{table}[h]
\centering
\small
\setlength{\tabcolsep}{4pt}
\begin{tabular}{lccccc}
\toprule
 & $t^\ast{=}15$ & $t^\ast{=}30$ & $t^\ast{=}60$ & $t^\ast{=}120$ & MC-DDIM \\
\midrule
$\rho$, tissue        & $0.663$ & $0.671$ & $0.670$ & $0.669$ & $0.694$ \\
$\rho$, interior        & $0.619$ & $0.621$ & $0.612$ & $0.606$ & $0.633$ \\
$\rho$, edges           & $0.244$ & $0.296$ & $0.352$ & $0.406$ & $0.528$ \\
$\rho$, change region   & $0.042$ & $0.065$ & $0.100$ & $0.145$ & $0.265$ \\
$\rho$, partial (grad.\ rm.) & $0.317$ & $0.334$ & $0.339$ & $0.346$ & $0.409$ \\
AUROC, top-5\% error  & $0.854$ & $0.870$ & $0.884$ & $0.897$ & $0.922$ \\
\midrule
$\Delta\rho$, tissue & $-0.027$ & $-0.020$ & $-0.022$ & $-0.026$ & --- \\
$\Delta\rho$, partial (grad.\ rm.) & $-0.093$ & $-0.072$ & $-0.060$ & $-0.048$ & --- \\
\bottomrule
\end{tabular}
\caption{Sensitivity to the probe timestep on lung CT ($N{=}130$,
$G{=}47$), self-chain probe, $K = 30$, scored against the common
error map; absolutes are per-volume means, the paired contrasts are
participant-weighted, and every contrast interval excludes zero.
Tissue $\rho$ moves by less than $0.01$ between $t^\ast = 15$ and
$120$, the gradient-controlled partial by $0.03$.  The regions where the probe trails MC-DDIM, edges and true
change, improve monotonically with $t^\ast$ and at $120$ remain well
below it.  $t^\ast = 60$ was fixed before any of these runs and is
kept.}
\label{tab:tstar}
\label{tab:cost}
\end{table}

\section{How many probes?}
\label{app:kablation}

The reverse chain dominates the per-volume cost, so $K$ can be swept
from the same chain at the price of the extra probe forwards alone.
Table~\ref{tab:kablation} does this over the complete scored set of
lung CT; the OASIS value of $K$ was fixed by the same reasoning and
not swept.

\begin{table}[h]
\centering
\small
\begin{tabular}{rrcc}
\toprule
 & \multicolumn{3}{c}{PNG ($N{=}130$)} \\
\cmidrule(lr){2-4}
$K$ & Fwd. & $\bar\rho(\sigma, e)$ & \% of $K{=}30$ \\
\midrule
$5$  & $66$ & $0.567\;[0.530, 0.603]$ & $92.4\%$ \\
$10$ & $71$ & $0.595\;[0.557, 0.631]$ & $96.9\%$ \\
$15$ & $76$ & $0.604\;[0.566, 0.639]$ & $98.4\%$ \\
$30$ & $91$ & $0.613\;[0.575, 0.649]$ & $100\%$ \\
\bottomrule
\end{tabular}
\caption{Probe size, self-chain probe, $t^\ast = 60$, lung CT,
participant-weighted with participant-level intervals; percentages
relative to $K = 30$.  Returns
flatten quickly: $K = 10$ recovers $96.9\%$ and $K=15$ recovers
$98.4\%$.  The reverse chain dominates the per-volume cost, so the
sweep is obtained from the same chain and covers the whole test set
rather than a subset.  We report $K = 30$ on PNG and $K = 15$ on
OASIS, both fixed before any 63-participant score was read; nothing
in the method depends on those choices.  The maps are scored against
the self-chain (deployment) error map, as in
Table~\ref{tab:aggregate}.}
\label{tab:kablation}
\end{table}

\end{document}